\documentclass{article}

\PassOptionsToPackage{numbers, compress}{natbib}

\usepackage[final]{neurips_2022}

\usepackage{etoolbox}
\newbool{excludefood}
\setbool{excludefood}{true}

\usepackage[utf8]{inputenc} 
\usepackage[T1]{fontenc}    
\usepackage[pdftex]{graphicx}
\usepackage[font=small]{caption}  
\usepackage{hyperref}       
\usepackage{url}            
\usepackage{booktabs}       
\usepackage{amsfonts}       
\usepackage{nicefrac}       
\usepackage{microtype}      
\usepackage{xcolor}         
\usepackage{wrapfig}
\usepackage{subcaption}
\usepackage{tabularx}
\usepackage{amsmath}
\usepackage{amssymb}
\usepackage{algorithm}
\usepackage{algpseudocode}
\usepackage{mathtools}
\usepackage{copyrightbox}
\usepackage{bbm}
\usepackage{amsthm}

\usepackage{xr}

\makeatletter
\newcommand*{\addFileDependency}[1]{
  \typeout{(#1)}
  \@addtofilelist{#1}
  \IfFileExists{#1}{}{\typeout{No file #1.}}
}
\makeatother

\newcommand*{\myexternaldocument}[1]{%
    \externaldocument{#1}%
    \addFileDependency{#1.tex}%
    \addFileDependency{#1.aux}%
}
\myexternaldocument{supp}

\makeatletter
\renewcommand{\CRB@setcopyrightfont}{%
\usefont{T1}{phv}{m}{n}\fontsize{4}{4}\selectfont
}
\makeatother

\usepackage[normalem]{ulem}

\title{VICE: Variational Interpretable Concept Embeddings}

\author{Lukas Muttenthaler\thanks{Also affiliated with the Max Planck Institute for Human Cognitive and Brain Sciences, Leipzig, Germany.}  \\
Machine Learning Group\\
Technische Universit\"at Berlin\\
BIFOLD\thanks{Berlin Institute for the Foundations of Learning and Data, Germany.} \\ 
Berlin, Germany \\
\And
Charles Y. Zheng \\
Machine Learning Team, FMRI Facility \\
National Institute of Mental Health \\
Bethesda, MD, USA\\
\And 
Patrick McClure\thanks{Work was partially done while affiliated with the National Institute of Mental Health, Bethesda, MD, USA.} \\
Department of Computer Science\\
Naval Postgraduate School \\
Monterey, CA, USA\\
\And
Robert A. Vandermeulen \\
Machine Learning Group\\
Technische Universit\"at Berlin\\
BIFOLD\textsuperscript{$\dagger$} \\
Berlin, Germany \\
\And
Martin N. Hebart \\
Vision and Computational Cognition Group \\
MPI for Human Cognitive and Brain Sciences \\
Leipzig, Germany\\
\And
Francisco Pereira \\
Machine Learning Team, FMRI Facility\\
National Institute of Mental Health \\
Bethesda, MD, USA\\
}

\DeclareMathOperator*{\argmax}{arg\,max}
\DeclareMathOperator*{\argmin}{arg\,min}

\DeclareMathOperator{\quantize}{quantize}
\DeclareMathOperator{\loss}{loss}

\newtheorem{prop}{Proposition}[section]

\begin{document}

\maketitle

\begin{abstract}
A central goal in the cognitive sciences is the development of numerical models for mental representations of object concepts. This paper introduces Variational Interpretable Concept Embeddings (VICE), an approximate Bayesian method for embedding object concepts in a vector space using data collected from humans in a triplet odd-one-out task. VICE uses variational inference to obtain sparse, non-negative representations of object concepts with uncertainty estimates for the embedding values. These estimates are used to automatically select the dimensions that best explain the data. We derive a PAC learning bound for VICE that can be used to estimate generalization performance or determine a sufficient sample size for experimental design. VICE rivals or outperforms its predecessor, SPoSE, at predicting human behavior in the triplet odd-one-out task. Furthermore, VICE's object representations are more reproducible and consistent across random initializations, highlighting the unique advantage of using VICE for deriving interpretable embeddings from human behavior.
\end{abstract}

\section{Introduction}
\label{sec:intro}
Human knowledge about object concepts encompasses many types of information, including function or purpose, visual appearance, encyclopedic facts, or taxonomic characteristics. A central question in cognitive science concerns the representation of this knowledge and its use across different tasks. One approach to this question is inductive and lets subjects list properties for hundreds to thousands of objects \citep{mcrae2005semantic,devereux2014centre,buchanan2019english,hovhannisyan2020visual}, yielding large lists of responses about different types of properties. Specifically, objects are represented as vectors of binary properties. While this approach is agnostic to downstream prediction tasks, it may be biased since subjects may leave out important features and/or mention unimportant ones. For example, one may forego a general property (e.g., ``is an animal'') while providing a highly specific fact (e.g., ``is found in Florida''). In an alternative, deductive approach, researchers postulate dimensions of interest and subsequently let subjects rate objects in each dimension. \citet{binder2016toward} employed such an approach and collected ratings for hundreds of objects, verbs, and adjectives. These ratings were gathered over 65 dimensions, reflecting sensory, motor, spatial, temporal, affective, social, and cognitive experiences. Nonetheless, it is desirable to discover object representations that are not biased by the conducted behavioral task and whose dimensions are interpretable without necessitating \textit{a priori} assumptions about their semantic content.

Recently, \citet{ZhengPBH19} and \citet{hebart2020revealing} introduced SPoSE, a model of the mental representations of 1,854 objects in a 49-dimensional space. The model was learned from human judgments about object similarity, where subjects were asked to determine an odd-one-out object in random triplets of objects. SPoSE embedded each object in a vector space so that each dimension is non-negative and sparse (most objects have a close-to-zero entry for a given dimension). The authors showed that the embedding dimensions of objects were interpretable and that subjects could coherently label what the dimensions were ``about,'' ranging from categorical (e.g., animate, food) to functional (e.g., tool),  structural (e.g., made of metal or wood), or visual (e.g., coarse pattern). The authors hypothesized that interpretability arose from combining positivity and sparsity constraints so that no object was represented by every dimension, and most dimensions were present for only a few objects. In addition, SPoSE could predict human judgements close to the estimated best attainable performance \citep{ZhengPBH19,hebart2020revealing}.

Despite its notable performance, SPoSE has several limitations. The first stems from the use of an $\ell_{1}$ sparsity penalty to promote interpretability.  In SPoSE, 6 to 11 dominant dimensions for an object account for most of the prediction performance. These dimensions are different between objects. A potential issue with enforcing SPoSE to have even fewer dimensions is that it may cause excessive shrinkage of the dominant values \citep{belloni2013least}.
Second, when inspecting the distributions of values across objects, most SPoSE dimensions do not reflect the exponential prior induced by the $\ell_{1}$ penalty. Overcoming this prior may lead to suboptimal performance and inconsistent solutions, specifically in low data regimes. Third, SPoSE uses an ad-hoc criterion for determining the dimensionality of the solution via an arbitrary threshold on the density of each dimension. Finally, SPoSE has no criterion for determining convergence of its optimization process, nor does it provide any formal guarantees on the sample size needed to learn a model of desired complexity.

To overcome these limitations we introduce VICE, a variational inference (VI) method for embedding object concepts with interpretable, sparse, and non-negative dimensions. We start by discussing related work and a description of the triplet task and SPoSE, followed by our presentation of theory and experimental results.

\noindent {\bf Contribution 1: VICE solves major limitations of SPoSE} First, VICE encourages shrinkage while allowing for small entry values by using a \emph{spike-and-slab} prior \citep{blundell15,fahrmeir2010bayesian,george1993variable,mitchell1988bayesian,ray2020spike,Rockova18,titsias2011spike}. We deem this more appropriate than an exponential prior, because \textit{importance} -- the value an object takes in a dimension -- is different from \textit{relevance} -- whether the dimension is applicable to that object -- and both can be controlled separately.
Second, we use VI with a unimodal posterior for representing each object in a dimension which yields a mean value and an uncertainty estimate. While unimodality makes it possible to use the mean values as representative object embeddings, the uncertainty estimates allow us to use a statistical procedure to automatically select the dimensions that best explain the data.
Third, we use this procedure to introduce a convergence criterion that reliably identifies \emph{representational stability}, i.e. the consistency in the number of selected dimensions.

\noindent {\bf Contribution 2: A PAC bound on the generalization of SPoSE and VICE models}
This bound can be used \emph{retrospectively} to provide guarantees about the generalization performance of a converged model. Furthermore, it can be used \emph{prospectively} to determine the sample size required to identify a representation given the number of objects and a maximum possible number of dimensions.

\noindent {\bf Contribution 3: Extensive evaluation of model performance across multiple datasets} We compare VICE with SPoSE over three different datasets. One of these datasets contains concrete objects, another consists of adjectives, and the third is composed of food items. Experimentally, we find that VICE rivals or outperforms the performance of SPoSE in modeling human behavior. Moreover, we find that VICE yields dimensions that are much more reproducible, with a lower variance for the number of (selected) dimensions. Both of these measures are particularly important in the cognitive sciences. Lastly, we compare VICE with SPoSE for reduced amounts of data and show that VICE has significantly better performance on all measures.

A \texttt{PyTorch} implementation of VICE featuring Continuous Integration is publicly available at \url{https://github.com/LukasMut/VICE}. The GitHub repository additionally includes code to reproduce all of the experiments that are presented in this paper.

\section{Related work}
\label{sec:related_work}
\citet{navarro2008latent} introduced a method for learning semantic concept embeddings from item similarity data. The method infers the number of embedding dimensions using the Indian Buffet Process (IBP) \citep{griffiths2011indian}.
Their approach relies on continuous-valued similarity ratings rather than discrete forced-choice behavior and is not directly applicable to our setting. It is also challenging to scale the IBP to the number of dimensions and samples in our work \citep{zoubin2013scaling}. \citet{roads2021enriching} introduced a method for learning object embeddings from behavior in an 8-rank-2 task. Their method predicts behavior from the embeddings by using active sampling \citep{Gottlieb2018} to query subjects with the most informative stimuli, yielding an object similarity matrix. The interpretability of embedding dimensions was not considered in this work.

A different approach has been to develop interpretable concept representations from text corpora. Early methods used word embeddings with positivity and sparsity constraints \citep{murphy2012learning}. Later work in this direction used topic model representations of Wikipedia articles about objects \citep{pereira2013using}, transformations of word embeddings into sparse non-negative representations \citep{subramanian2018spine,panigrahi2019word2sense}, or predictions of properties \citep{devereux2014centre} or dimensions \citep{utsumi2020exploring}. Others have considered using text corpora in conjunction with imaging data \citep{fyshe2014interpretable,derby2018using}. Finally, \citet{derby2019feature2vec} introduced a neural network function that maps the sparse feature space of a semantic property norm to the dense space of a word embedding, identifying informative combinations of properties and ranking candidate properties for new words.

\section{Triplet task}
\label{method:triplet_task}
The {\em triplet task}, also known as the \emph{triplet odd-one-out task}, is used for discovering object concept embeddings from similarity judgments over a set of $m$ different objects. These judgments are collected from human participants who are given queries that consist of a \emph{triplet} of objects (e.g., $\{\text{``suit''}, \text{``flamingo''}, \text{``car''}\}$). Participants are asked to consider the three pairs in a triplet $\{\{\text{``suit'',``flamingo''}\}, \{\text{``suit'',``car''}\}, \{\text{``flamingo'',``car''})\}$, and to decide which pair is the most similar, leaving the third as the odd-one-out. We assign each object a numerical index, e.g., $1 \gets\text{``aardvark''}, \ldots, 1854 \gets\text{``zucchini''}$. Let $\{y, z\}$ denote the indices in this pair, e.g.,  $\{y,z\} = \{268,609\}$ for ``suit'' and ``flamingo.'' A dataset $\mathcal{D}$ is a set of $n$ ordered pairs of presented triplets and selected pairs. That is,  $\mathcal{D} \coloneqq \big(\{i_s, j_s, k_s\}, \{y_s, z_s\}\big)_{s=1}^{n}$, where $\left\{y_s,z_s\right\} \subset \left\{i_s,j_s,k_s\right\}$. In Appendix~\ref{app:triplet_task} we discuss in detail why the triplet task appears to be a sensible choice for modeling object similarities in humans.

\begin{figure}[h!]
\centering
    \begin{subfigure}{.3\textwidth}
        \centering
        \captionsetup{justification=centering}
        \copyrightbox
        {\includegraphics[height=1.1in,width=1.1in]{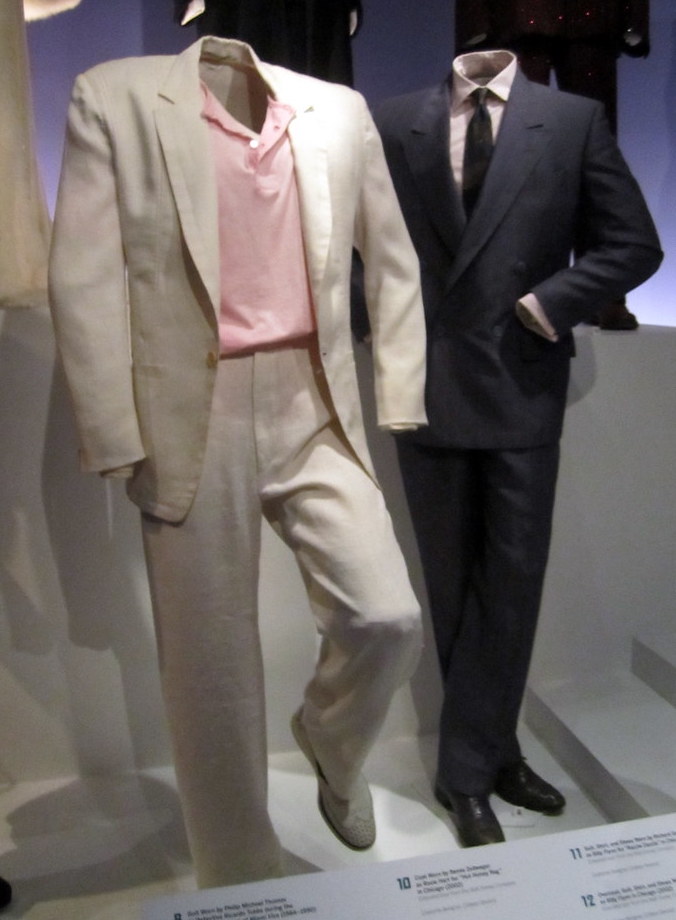}}
        {(c) Janderk1968}
    \end{subfigure}
    \begin{subfigure}{.3\textwidth}
        \centering
        \captionsetup{justification=centering}
        \copyrightbox
         {\includegraphics[height=1.1in,width=1.1in]{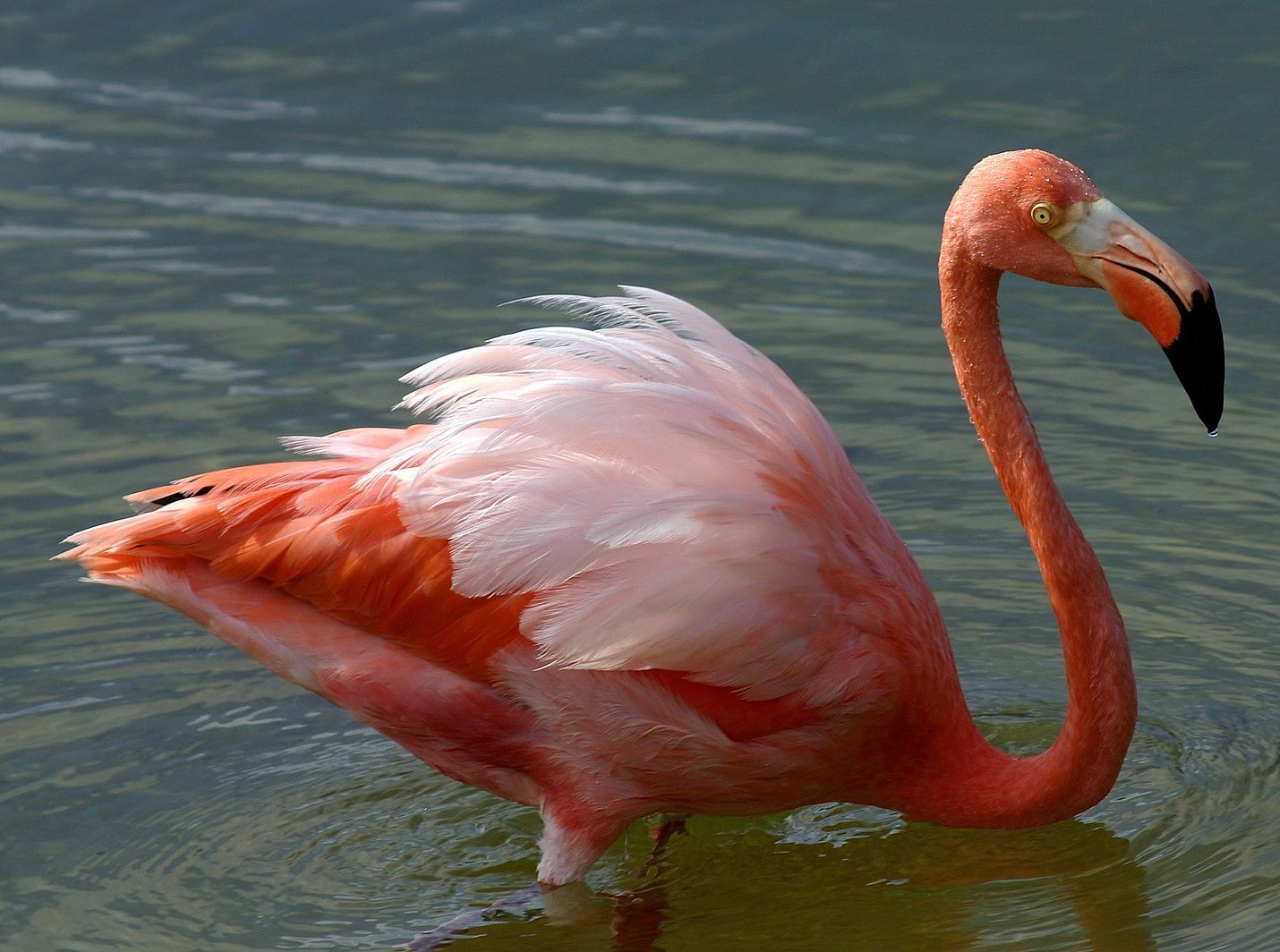}}
         {(c) Charles J. Sharp}
    \end{subfigure}
    \begin{subfigure}{.3\textwidth}
        \centering
        \captionsetup{justification=centering}
         \copyrightbox
        {\includegraphics[height=1.1in,width=1.1in]{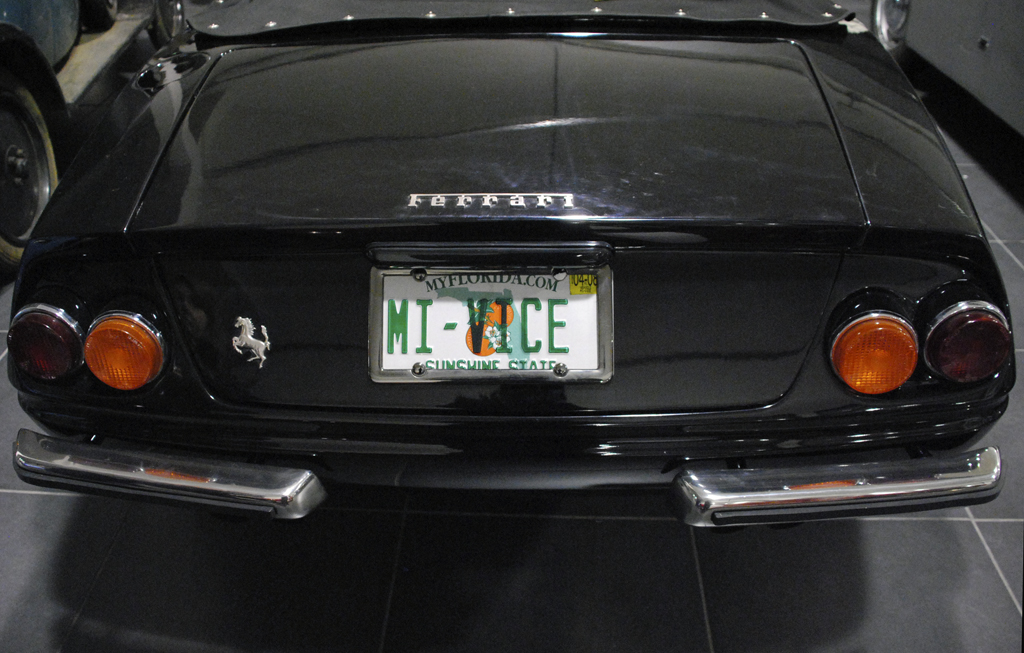}}
        {(c) Tim Evanson}
    \end{subfigure}%
\caption{Example triplet containing the objects ``suit'',  ``flamingo'', and ``car'' (creative commons images).}
\label{fig:example_triplet}
\end{figure}

\section{Formal setting}

Sparse Positive object Similarity Embedding (SPoSE) \citep{ZhengPBH19} is an approach for finding interpretable embedding dimensions from the triplet task. It does so by finding an embedding vector ${\bf x}_i = \left[x_{i1},\ldots,x_{id}\right]$ for every object $i$. Let $X$ denote the $m \times d$ matrix $({\bf x}_1,\hdots, {\bf x}_m)$ and $S \coloneqq XX^T$ be the similarity matrix, where $S_{ij}$ denote its entry at $i,j$. The probability of choosing $\{y_s, z_s\}$ as the most similar pair of objects, given an object triplet $\{i_s, j_s, k_s\}$ and the embedding matrix $X$, is modeled as
\begin{align}
p(\{y_s,z_s\}|\{i_s, j_s, k_s\},X)~\coloneqq&~\frac{\exp(S_{y_s,z_s})}{\exp(S_{i_s,j_s}) + \exp(S_{i_s,k_s}) + \exp(S_{j_s,k_s})}.
\label{eq:triplet_probs}
\end{align}
At this point we remark that each triplet $\{i_{s}, j_{s}, k_{s}\}$ is chosen uniformly at random from the collection of all possible sets of object triplets $\mathcal{T}$. That is, $\{i_{1}, j_{1}, k_{1}\},\ldots, \{i_{n}, j_{n}, k_{n}\} \overset{\text{i.i.d.}}{\sim} \mathcal{U}(\mathcal{T})$. This is precisely stated in Appendix~\ref{app:data_loglikelihood}. SPoSE uses maximum a posteriori (MAP) estimation with a Laplace prior under a non-negativity constraint (equivalent to an exponential prior) to find the most likely embedding $X$ for the training data $\mathcal{D}$. This leads to the training objective
\begin{equation*}
\label{eq:sparse_triplet_objective}
\argmin_{X \geq 0}~-\log p(\mathcal{D}|X) + \lambda \sum_{i=1}^{m} \sum_{j=1}^{d} |X_{ij}|,
\end{equation*}
where $\log p(\mathcal{D}|X) = \text{constant} +\sum_{s=1}^{n} \log p(\{y_s,z_s\}|\{i_s, j_s, k_s\},X)$ (see Appendix~\ref{app:data_loglikelihood} for a full derivation of the log-likelihood function) and $\lambda$ is determined using cross-validation.

\section{VICE}
\label{sec:vice}

In contrast to SPoSE, we use mean-field VI \citep{vi_review} instead of a MAP estimate for approximating the posterior probability $p(X|\mathcal{D})$. In VICE we impose additional constraints on the embedding matrix $X$ by using a prior that encourages shrinkage while allowing entries in $X$ to be close to zero.

\begin{algorithm}
    \caption{VICE optimization for individual triplets for a single training epoch}
    \label{alg:vice_optimization}
    \begin{algorithmic}
    \Require $\mathcal{D}, \theta, \alpha$ \Comment{Recall that $\mathcal{D} \coloneqq (\{i_s, j_s, k_s\}, \{y_s, z_s\})_{s=1}^{n}$, and $\alpha$ is a learning rate}
    \State $\theta \coloneqq \{\mu, \sigma\}$ \Comment{$\mu, \sigma \in \mathbb{R}^{m \times d}$}
    \For{$s \in \{1, \hdots, n\}$}
        \State $\epsilon \sim \mathcal{N}(0, I)$ \Comment{Draw i.i.d Gaussian noise where $\epsilon \in \mathbb{R}^{m \times d}$ and $\epsilon_{ij} \overset{\text{i.i.d.}}{\sim} \mathcal{N}(0,1)$}
        \State $X_{\theta, \epsilon} \coloneqq \mu + \sigma \odot \epsilon$ \Comment{Apply reparameterization trick; $\odot$ is the Hadamard product}
        \State $S \coloneqq [X_{\theta, \epsilon}]_{+}[X_{\theta, \epsilon}]_{+}^{T}$ \Comment{Compute similarity matrix using the non-negative parts of $X_{\theta, \epsilon}$}
        \State $\mathcal{L}_{\text{data}}(X_{\theta, \epsilon}) \triangleq \log\left[\frac{\exp(S_{y_s,z_s})}{\exp(S_{i_s,j_s}) + \exp(S_{i_s,k_s}) + \exp(S_{j_s,k_s})}\right]$ \Comment{Data log-likelihood term}
        \State $\mathcal{L}_{\text{complexity}}(X_{\theta, \epsilon}) \triangleq \frac{1}{n}\big[\log q_{\theta}(X_{\theta, \epsilon}) - \log p(X_{\theta, \epsilon})\big]$ \Comment{KL divergence/Regularization term}
        \State $\mathcal{L}_{\text{total}}(X_{\theta, \epsilon}) \triangleq \mathcal{L}_{\text{complexity}}(X_{\theta, \epsilon}) - \mathcal{L}_{\text{data}}(X_{\theta, \epsilon})$
        \State $\mu \gets \mu - \alpha \nabla_{\mu}\mathcal{L}_{\text{total}}(X_{\theta, \epsilon})$ \Comment{Update the embedding means}
        \State $\sigma \gets \sigma - \alpha \nabla_{\sigma}\mathcal{L}_{\text{total}}(X_{\theta, \epsilon})$ \Comment{Update the embedding standard deviations}
    \EndFor
    \Ensure $\theta$ \Comment{Return the optimized set of model parameters $\theta$}
    \end{algorithmic}
\end{algorithm}

\subsection{Variational Inference}
\label{sec:vbi}
For VICE we consider approximating $p(X|\mathcal{D})$ with a variational distribution, $q_{\theta}(X)$, where $q_{\theta} \in \mathcal{Q}$, and $\theta$ is optimized to minimize the KL divergence to the true posterior, $p(X|\mathcal{D})$. In VICE the KL divergence objective function (derived in Appendix~\ref{app:vbi}) is
\begin{equation}
     \argmin_\theta~\mathbb{E}_{q_\theta(X)}\left[\frac{1}{n}\left(\log q_\theta(X) - \log p(X)\right) - \frac{1}{n}\sum_{s=1}^{n} \log p\left(\{y_s,z_s\}|\{i_s, j_s, k_s\}, X\right)\right].
     \label{eq:vi_kld}
\end{equation}
\noindent {\bf Variational distribution} VI requires a choice of a parametric variational distribution $q \in \mathcal{Q}$. For VICE we use a Gaussian distribution with a diagonal covariance matrix $q_\theta(X) = \mathcal{N}(\mu,\text{diag}(\sigma^2))$ where $\theta = \{\mu, \sigma\}$. Therefore, each embedding dimension has a \textit{mean} and a \textit{standard deviation}. We deem a Gaussian variational distribution appropriate for a variety of reasons. First, under certain conditions, the posterior is Gaussian in the infinite-data limit \citep{kleijn2012bernstein}. Second, a unimodal variational distribution makes it possible to use $\mu_1,
\ldots, 
\mu_m$ as representative object embeddings. A fixed representation is useful for downstream use cases that rely on a single embedding vector for each object. Use cases range from using embeddings as targets in a regression task over unsupervised clustering of embeddings to interpreting embeddings. More complex variational families are clearly not as practical for this. Third, a Gaussian posterior is a computationally convenient choice.

Similarly to \citet{titsias14}, we use a Monte Carlo approximation of Equation~\ref{eq:vi_kld} by sampling a limited number of $X$s from $q_{\theta}(X)$ during training. We generate $X$ with the reparameterization trick \citep{KingmaVAE,rezende14}, $X_{\theta, \epsilon} = \mu + \sigma \odot \epsilon$, where
$\epsilon \in \mathbb{R}^{m\times d}$ is entrywise $\mathcal{N}(0,1)$, and $\odot$ denotes the Hadamard (element-wise) product. This leads to the objective
\begin{equation}
     \argmin_\theta~\frac{1}{nR} \sum_{r=1}^R 
     \left(\log q_\theta(X_{\theta, \epsilon^{(r)}}) - \log p(X_{\theta,\epsilon^{(r)}}) - \sum_{s=1}^{n} \log p(\{y_s,z_s\}|\{i_s, j_s, k_s\},[X_{\theta,\epsilon^{(r)}}]_+)
     \right).
\label{eq:vi_objective}
\end{equation}
We apply a ReLU function, denoted by $[\cdot]_+$, to the sampled $X_{\theta, \epsilon}$ values to guarantee that $X_{\theta, \epsilon} \in \mathbb{R}_{+}^{m\times d}$. As commonly done in the Dropout and Bayesian Neural Network literature \citep{srivastava2014dropout,blundell15,gal2016dropout,McClureK16}, we set $R = 1$ for computational efficiency during the optimization process. The optimization is outlined for individual triplets (i.e., where $B=1$) for a single training epoch in Algorithm~\ref{alg:vice_optimization}.

\noindent {\bf Posterior probability estimation} Computational efficiency at inference time is not as critical as it is during training. Therefore, we can get a better estimate of the posterior probability distribution over the three possible odd one-one-out choices by letting $R \gg 1$. Using the optimized variational posterior, $q_{\hat{\theta}(X)}$, we approximate the probability distribution with a Monte Carlo estimate \citep{Graves11,blundell15,McClureK16,vi_review} from $R$ samples $X^{(r)} = X_{\hat{\theta},\epsilon^{(r)}}$ for $r=1,\hdots,R$, yielding
\begin{equation}
    \hat{p}\left(\{y,z\}|\{i,j,k\}\right) \coloneqq \frac{1}{R} \sum_{r=1}^R p(\{y,z\}|\{i,j,k\},X^{(r)}).
\label{eq:choice_distribution}
\end{equation}
\noindent {\bf Spike-and-slab prior} As discussed above, SPoSE induces sparsity through an $\ell_{1}$ penalty which, along with the non-negativity constraint, is equivalent to using an exponential prior. Through examination of the publicly available histograms of weight values in the two most important SPoSE dimensions (see Figure~\ref{fig:spose_dimension} in Appendix~\ref{app:histograms_of_spose_dimensions}), we observed that the dimensions did not resemble an exponential distribution. Instead, they contained a \emph{spike} of probability at zero and a wide \emph{slab} of probability for the non-zero values. To model this, we use a spike-and-slab Gaussian mixture prior \citep{blundell15,fahrmeir2010bayesian,george1993variable,ishwaran2005spike,malsiner2018comparing},
\begin{equation}
     p(X)= \prod_{i=1}^{m} \prod_{j=1}^{d} (\pi_{\text{spike}} \mathcal{N}(X_{ij};0,\sigma_{\text{spike}}^2) + (1-\pi_{\text{spike}}) \mathcal{N}(X_{ij};0,\sigma_{\text{slab}}^2)),
\label{eq:spike_slab}
\end{equation}
which encourages shrinkage. This prior has three parameters, $\sigma_{\text{spike}}$, $\sigma_{\text{slab}}$, and $\pi_{\text{spike}}$. $\pi_{\text{spike}}$ is the probability that an embedding dimension is drawn from the \textit{spike} Gaussian. Since spike and slab distributions are mathematically interchangeable, by convention we require that $\sigma_{\text{spike}} \ll \sigma_{\text{slab}}$.

\subsection{Dimensionality reduction and convergence}
\label{method:pruning}

For interpretability purposes it is desirable for the object embedding dimensionality, $d$, to be small. In contrast to SPoSE, which employs a user-defined threshold to prune dimensions, VICE exploits the uncertainty estimates for embedding values to select a subset of informative dimensions.

The pruning procedure works by assigning importance scores to each of the $d$ dimensions, which reflect the number of objects that we can confidently say have a non-zero weight in a dimension. To compute the score, we use the variational embedding for each object $i$ and dimension $j$ -- location $\mu_{ij}$ and scale $\sigma_{ij}$ parameters -- to compute the posterior probability that the weight is truncated to zero according to the left tail of a Gaussian distribution with that location and scale (see \S\ref{sec:vbi}). This gives us a posterior probability of the weight being zero for each object within a dimension \citep{Graves11}. To calculate the overall importance of a dimension, we estimate the number of objects that have non-zero weights, while controlling the False Discovery Rate \citep{Benjamini1995} with $\alpha = 0.05$. We define the importance of each dimension $j$ to be the number of objects for which $P(X_{ij} > 0) \geq .95$ holds. After convergence, we prune the model by removing dimensions with $5$ or fewer statistically significant objects. This is a commonly used reliability threshold in semantic property norms (e.g., \citep{mcrae2005semantic,devereux2014centre}).

In gradient-based optimization, the gradient of an objective function with respect to the parameters of a model, $\nabla{\mathcal{L(\theta)}}$, is used to iteratively find parameters $\hat{\theta}$ that minimize that function.
We use a {\em representational stability} criterion to determine convergence. That is, the optimization process halts when the number of identified dimensions - as described above - has not changed by a \emph{single} dimension over the past $L$ epochs (e.g., $L = 500$).  Given that our goal is to find stable estimates of the number of dimensions, we considered this to be more appropriate than other convergence criteria such as evaluating the cross-entropy error on a validation set or evidence-based criteria \citep{evidence_based_gradients,evidence_based_hessian}. For further details on convergence and the optimization process see Appendix~\ref{app:gradient_based_optimization}.

\section{Sample complexity bound}
\label{sec:pac_bound}

We use statistical learning theory to obtain estimates of the sample size needed to appropriately constrain VICE (and SPoSE) models. These estimates can be used \emph{retrospectively} to obtain probabilistic guarantees on the generalization of a trained model to unseen data, or \emph{prospectively} to decide how much data to collect for a study. The bound presented here applies to any embedding matrix, $X$, using the model in Equation~\ref{eq:triplet_probs}. When using this bound to analyze VICE, we extract $\mu$ and use it as a fixed embedding to predict the most likely odd-one-out for a query triplet, rather than sampling from the variational distribution.

To obtain a useful bound we make two assumptions. First, we assume that there exists an upper bound $M$ on the largest value in the embedding. Second, we assume that the embeddings obtained by either SPoSE or VICE can be quantized coarsely with marginal losses in predictive performance. We can choose a discretization scale $\Delta$, e.g., $\Delta = 0.5$, and round embedding values to a non-negative integer multiple of $\Delta$. While we employ this quantization primarily to use learning theory bounds for finite hypothesis classes, it could have benefits for interpretation as well (e.g., a dimension could consist of labels such as zero (0), very low (0.5), low (1), medium (1.5), high (2)). In all datasets we have used, SPoSE and VICE embeddings with reasonable priors had values below $2.7$ (see Appendix~\ref{app:quantization_results}). 
Given an upper bound $M$ and a discretization scale $\Delta$, all entries of $X$ are limited to the set $\mathbb{A}:=\{0, \Delta, \hdots, (k-1)\Delta, k\Delta\}$. The following bound tells us that, given enough samples, the true error rate is not much worse than the estimated error rate for such discretized $X$ matrices, with high probability.
\begin{prop} 
Given $\delta >0 $, $\epsilon > 0$, and $n\geq \left(m d \log (k+1) + \log \left(1/\delta\right)\right)/\left(2\epsilon^2\right)$ training samples $ \left(\{y_s,z_s\}, \{i_{s}, j_{s}, k_{s}\}\right)_{s=1}^n$, then
\begin{equation*}
P\left(\sup_{X \in \mathbb{A}^{ m\times d}} 
\hat{R}(X)
-  R(X) < \epsilon\right) \ge 1-\delta,
\end{equation*}
where
\begin{align*}
    \hat{R}(X) &:= \frac{1}{n}\sum_{s=1}^n \mathbbm{1}\left( \left\{y_s,z_s\right\} \neq \argmax_{\{y, z\}} p\left(\{y, z\}|\{i_s, j_s, k_s\},X\right) \right)
\end{align*}
and
\begin{align*}
    R(X) &:= P_{\left(\{y', z'\},\{i', j', k'\}\right)}\left( \left\{  y',z' \right\} \neq \argmax_{\{y, z\}} p\left(\{y, z\}|\{i', j', k'\},X\right)\right).
\end{align*}
\end{prop}
\begin{proof}[Proof sketch.]
This bound follows from applying Hoeffding's inequality to the $(k+1)^{m\times d}$ elements of $\mathbb{A}^{ m\times d}$ combined with a union bound. This proof is virtually identical to that of Theorem 2.13 in \citet{mohri2018foundations}. A complete proof can be found in Appendix~\ref{app:proof}.
\end{proof}
To use the bound we first decide on the number of quantization steps $k =\left\lceil M/\Delta\right\rceil$. The probability of violating the bound, $\delta$, is analogous to the Type I error control in hypothesis testing and is often set to  $0.05$. Since $n$ in our proposition depends very weakly on $\delta$, it is convenient to use $\delta = 0.001$.  The error tolerance $\epsilon$ \emph{does} have a major effect on the sample size estimate provided by the bound. For the bound to be practically useful, $\epsilon$ has to be smaller than the difference between the training error -- average zero-one loss over the training set examples -- and random guessing error.
For the datasets we use in this paper we observed a difference of $0.2-0.3$. A conservative choice would halve the accuracy gap of 0.2, giving $\epsilon = 0.1$. Together, with $k=4$ (from above), these values result in a prospective sample size of $n \ge 50 \cdot (m d\log(5) + \log(1000))$. To use the bound retrospectively, we fix $n$, $\delta$ while varying $\Delta$ (and consequently $k$) to get a guarantee on $\epsilon$, which in turn yields a probabilistic upper bound on the generalization error for an embedding, $R(X) \leq \hat{R}(X) + \epsilon$. For more details, see Algorithm~\ref{alg:adaptive_quantization} in Appendix~\ref{app:retrospective_algorithm}. 

\section{Experiments}
\label{sec:experiments}

\noindent {\bf Data} We performed experiments for three triplet datasets: \textsc{Things} (used to develop SPoSE \cite{ZhengPBH19,hebart2020revealing}), \textsc{Adjectives}\footnote{\textsc{Adjectives} is not yet published, but was shared with us by Shruti Japee.}, and \textsc{Food} \citep{carrington2021} \footnote{\textsc{Food} was shared with us by Jason Avery.}. \textsc{Things} and \textsc{Adjectives} contain random samples from all possible triplet combinations for 1,854 objects and 2,372 adjectives, respectively. \textsc{Things} and \textsc{Adjectives} are each comprised of a large training dataset which contains no ``repeats,'' i.e., there is only one human response for each triplet ($\sim 1.5$M triplets for \textsc{Things} and $\sim 800$K triplets for \textsc{Adjectives}). For each of these training datasets, 10\% of the triplets are assigned to a predefined validation set. The test sets for both contain the results for 1,000 random triplets that are not contained in the training data, with 25 repeats for each triplet. \textsc{Food} contains data for every possible triplet combination for 36 objects and repeats for some triplets. We partitioned this dataset into train (45\%), validation (5\%), and test (50\%) sets with disjoint triplets. Appendix~\ref{app:dataset_details} provides details about stimuli, data collection, and quality control.\label{sec:data}

\noindent {\bf Experimental setup} We implemented both SPoSE and VICE in \texttt{PyTorch} \citep{DBLP:conf/nips/PaszkeGMLBCKLGA19} using Adam \citep{DBLP:journals/corr/KingmaB14} with $\eta = 0.001$.  To find the best VICE hyperparameter combination, we performed a grid search over $\sigma_{\text{spike}}, \sigma_{\text{slab}}, \pi_{\text{spike}}$, optimizing Equation~\ref{eq:vi_objective} and evaluating each model on the validation set. All experiments were performed over $20$ different random initializations. For VICE we observed that the hyperparameter values with the lowest average cross-entropy error on the validation set were highly similar across the different datasets (see Table~\ref{tab:optimal_hyperparams} in Appendix~\ref{app:experimental_details}). For more information about the experimental setup, including details about training, weight initialization, hyperparameter grid, and optimal combination, see Appendix~\ref{app:experimental_details}. \label{sec:experimental_setup}

\noindent {\bf Evaluation measures} VICE estimates $p(\{y, z\}|\{i, j, k\})$, the model's softmax probability distribution over a given triplet (see Equation~\ref{eq:choice_distribution}). The first evaluation measure we consider is the prediction \emph{accuracy} with respect to the correct human choice $\{y, z\}$,  where the predicted pair is $\argmax_{\{y, z\}} \hat{p}\left(\{y, z\}|\{i, j, k\}\right)$. We estimate an upper bound on the prediction accuracy by using the repeats in the test set (see Appendix~\ref{app:experimental_details}). A Bayes-optimal classifier would predict the majority outcome for any triplet, whereas random guessing gives an accuracy of $1/3$. Note that the triplet task is subjective and thus there is not necessarily a definitive correct answer for a given triplet. If a test set contains multiple responses for a triplet, it provides information about the relative similarities of the three object pairs. This lets us calculate probability distributions over answers for each triplet. Approximating these distributions closely is important in cognitive science applications. Therefore, we additionally evaluate the models by computing the KL divergences between the models' posterior probability estimates $\hat{p}(\{y, z\}|\{i, j, k\})$ (see Equation~\ref{eq:choice_distribution}) and the human probability distributions inferred from the test set across all test triplets and report the mean. \label{sec:eval_measures}

\subsection{Results on THINGS}
\label{sec:results_things}
\begin{figure}[t]
\centering
\begin{subfigure}{.326\textwidth}
    \centering
    \captionsetup{justification=centering}
    \includegraphics[height=1.1in, width=1.0\textwidth]{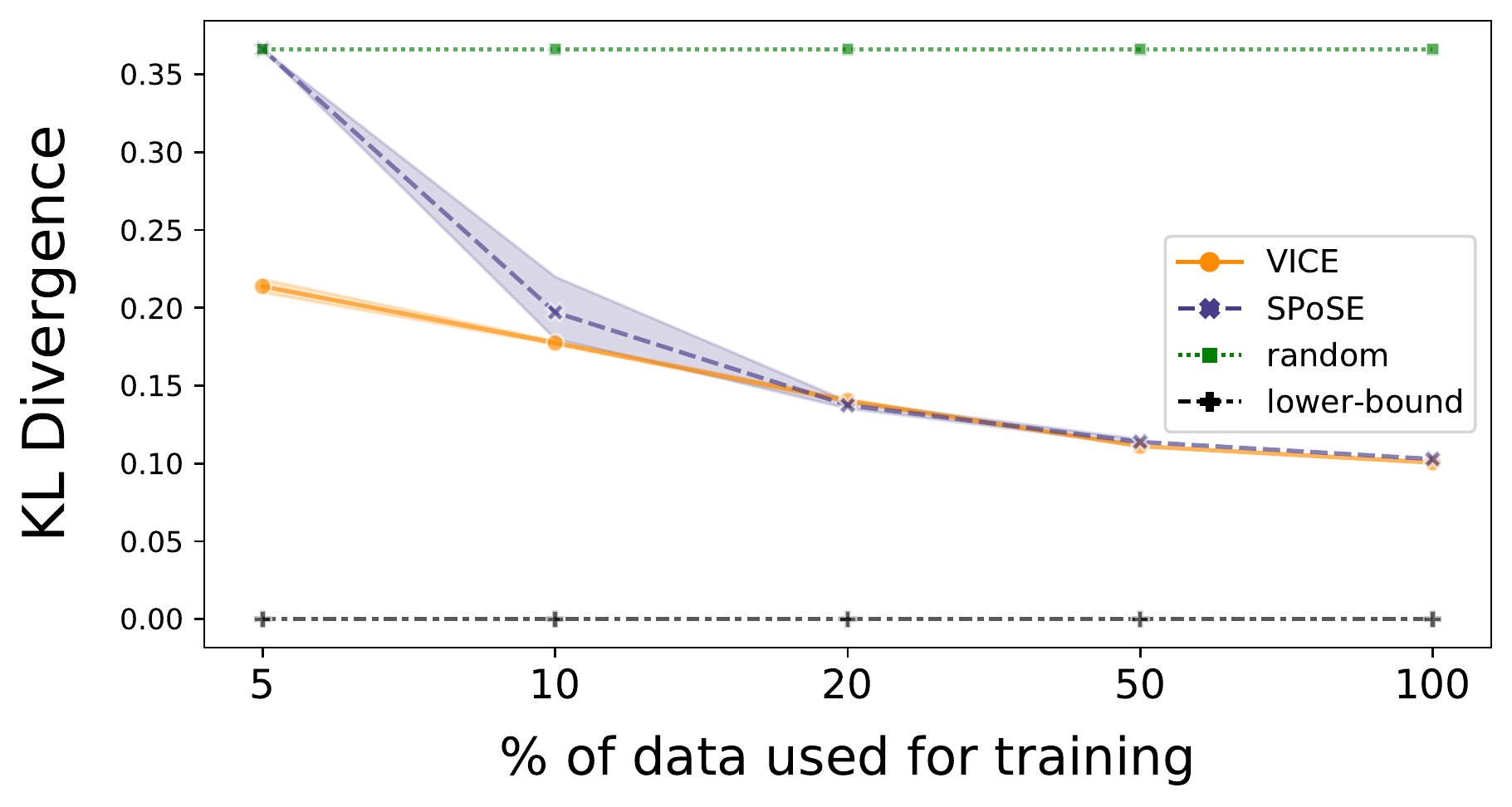}
\end{subfigure}%
\begin{subfigure}{.326\textwidth}
    \centering
    \captionsetup{justification=centering}
    \includegraphics[height=1.1in, width=1.0\textwidth]{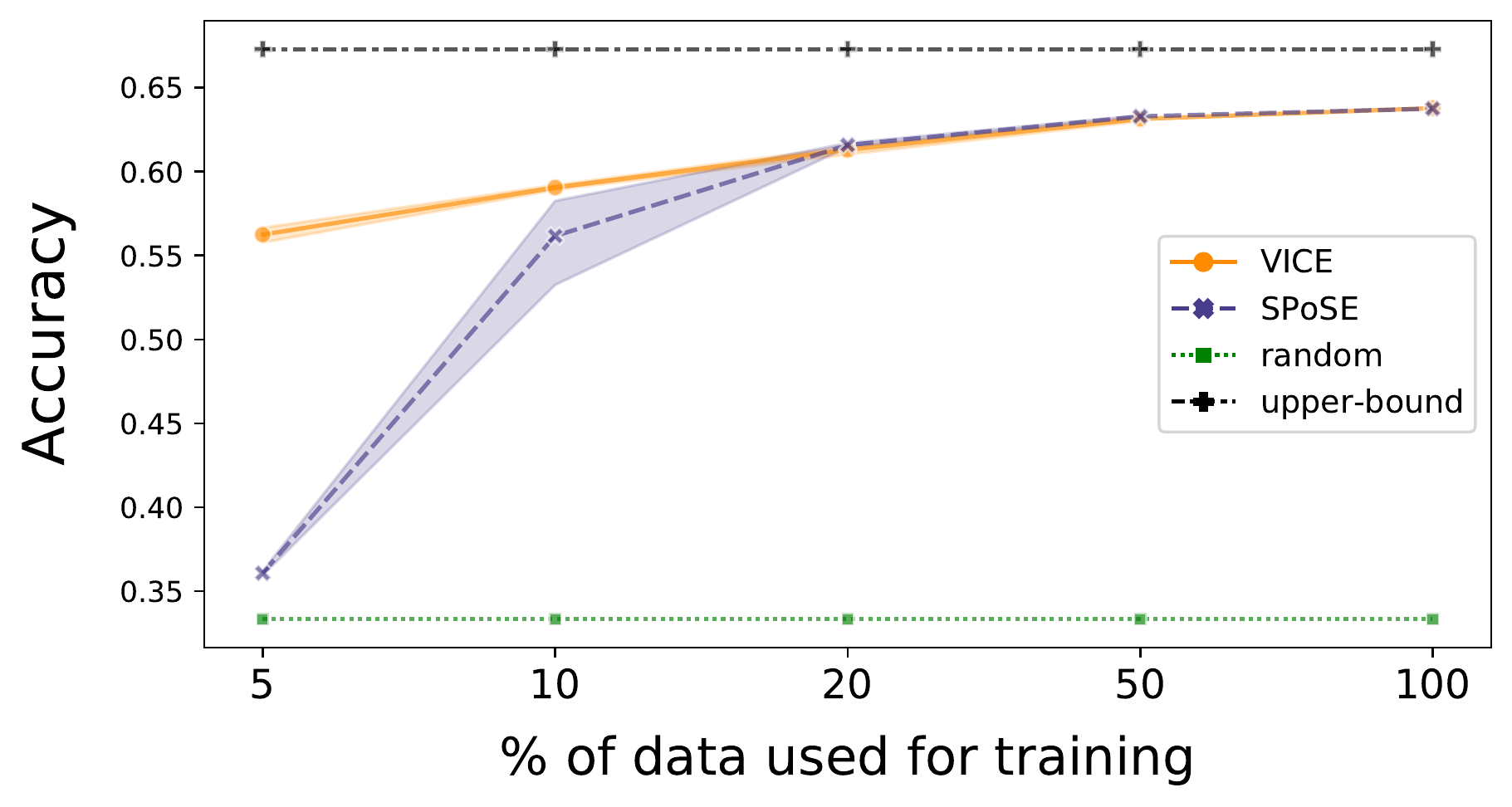}
\end{subfigure}
\begin{subfigure}{.326\textwidth}
    \centering
    \captionsetup{justification=centering}
    \includegraphics[height=1.1in, width=1.0\textwidth]{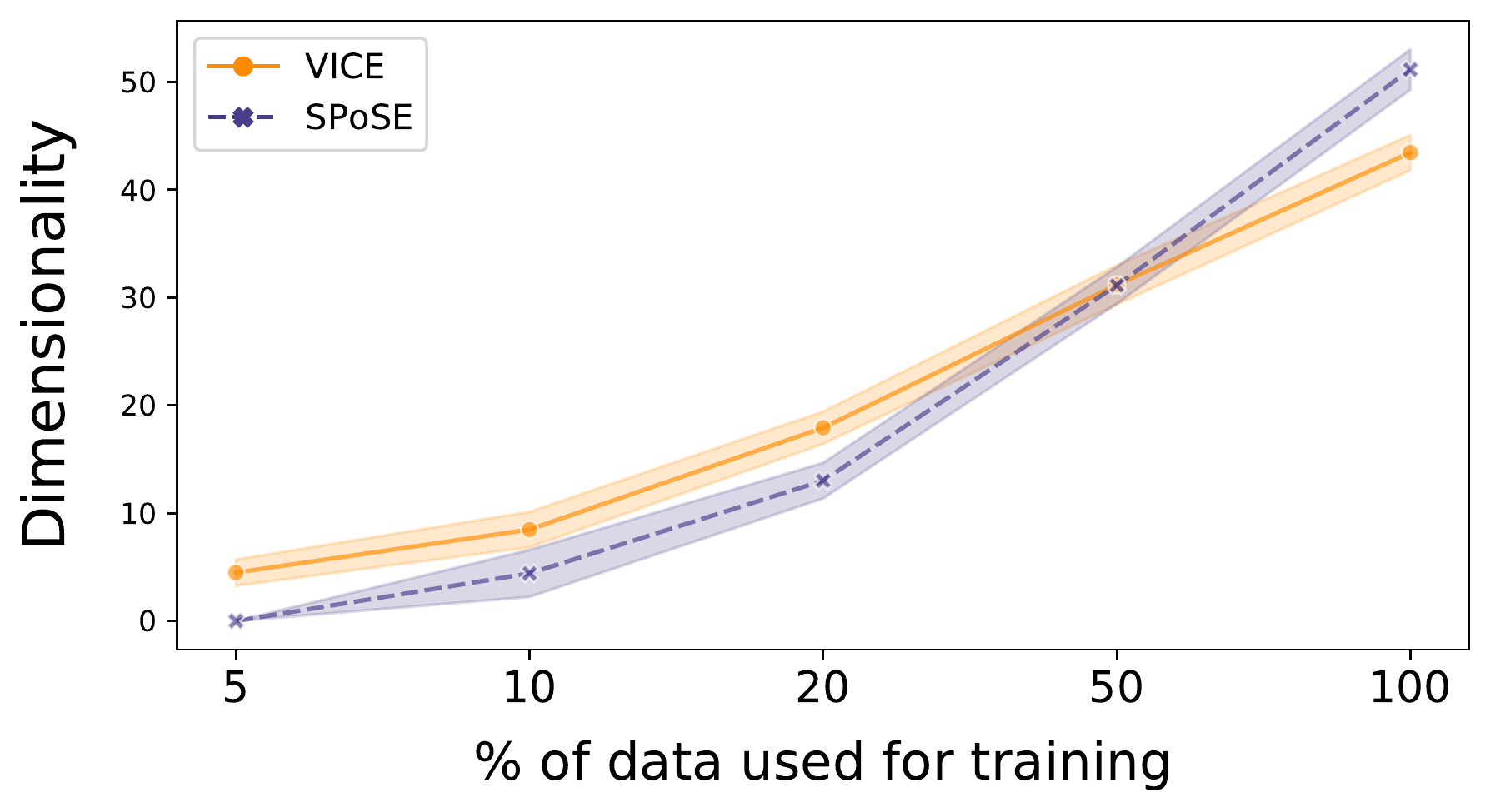}
\end{subfigure}
\caption{VICE vs. SPoSE. \textbf{Left}: KL divergences between model output and human probability distributions. \textbf{Center}: Triplet task prediction accuracies. \textbf{Right}: Number of identified embedding dimensions.
VICE and SPoSE were each trained on differently sized subsets of the \textsc{Things} training data. Error bands visualize 95\% CIs across the different random initializations and partitions of the data.}
\label{fig:data_effiency}
\end{figure}
\noindent{\bf Full dataset}
\label{results:full}
For this experiment we compared a representative model from VICE and SPoSE. The representative models were chosen to be those with the median cross-entropy error on the validation set over the 20 seeds. For VICE, we set the number of Monte Carlo samples to $R = 50$. On the test set, VICE and SPoSE achieved similar prediction accuracies of $0.638$ and $0.637$ respectively (estimated upper bound for accuracy was $0.673$). Likewise, VICE and SPoSE achieved similar average KL divergences of $0.100$ and $0.103$ respectively (random guessing KL divergence was $0.366$). The differences between the median model accuracies and KL divergences were not statistically significant according to a two-sided paired \textit{t}-test over individual test triplets. Hence, VICE and SPoSE predicted triplets approximately equally well when trained on the full dataset.
This is not surprising since Bayesian methods based on Monte Carlo sampling become more like deterministic maximum likelihood estimation in the infinite data limit. If $n \to \infty$, then the left-hand side in the expectation of Equation~\ref{eq:vi_kld}, which accounts for the contribution of the prior, goes to zero. Conversely, the effects of the prior are more prominent when $n$ is small, as we show in the following section.

\noindent{\bf Data efficiency experiments}
\label{results:subsamples}
Performance on small datasets is important in cognitive science, since behavioral experiments often have small sample sizes (e.g., tens to hundreds of volunteer in-lab subjects) compared to \textsc{THINGS}, which is particularly large, or they can be costly to collect. To test whether VICE can model the data better than SPoSE in small sample regimes, we performed experiments training on smaller subsets with sizes equal to 5\%, 10\%, 20\%, and 50\% of the training dataset. For each size, we divided the training set into equal partitions and performed an experiment for each partition, thereby giving multiple results for each subset size. Validation and held-out test sets remained unchanged. In Figure~\ref{fig:data_effiency} we show the average KL divergence (left) and prediction accuracy (middle) across data partitions and 20 random seeds (used to initialize model parameters) for various dataset sizes. The average performance across partitions was used to compute the confidence intervals (CIs). The differences in prediction accuracies and KL divergences between VICE and SPoSE were notable for the 5\% and 10\% data subsets.
In the former, SPoSE predicted only slightly better than random guessing; in the latter, SPoSE showed a large variability between random seeds and data partitions, as can be seen in the 95\% CIs. In comparison, VICE showed small variation in the two performance metrics across random seeds and performed much better than random guessing. The differences between VICE and SPoSE for the $5\%$ and $10\%$ subset scenarios were statistically significant according to a two-sided paired \textit{t}-test  comparing individual triplet predictions between median models ($p < 0.001$). For the full training dataset, VICE used significantly fewer dimensions ($p < 0.001$; unpaired sign test) than SPoSE to achieve comparable performance.

\subsection{Other results}
\noindent {\bf Other datasets} On the \textsc{Adjectives} test set, VICE and SPoSE had accuracies of $0.559$ and $0.562$, respectively (estimated upper bound was $0.607$), and KL divergences of $0.083$ and $0.088$, respectively (random guessing was $0.366$). On the \textsc{Food} test set, VICE and SPoSE median models had accuracies of $0.693$ and $0.698$, respectively. The differences in accuracy (and KL divergence for \textsc{Adjectives}) between VICE and SPoSE for both datasets were not statistically significant according to two-sided paired \textit{t}-tests on the median model predictions across triplets on the test set. However, for \textsc{Food}, VICE used significantly fewer dimensions than SPoSE to achieve similar performance (see Table~\ref{tab:reliability}).

\noindent {\bf Hyperparameters} We observed that VICE's performance is fairly insensitive to hyperparameter selection. Using default hyperparameters, $\sigma_{\text{spike}} = 0.25$, $\sigma_{\text{slab}} = 1.0$, $\pi_{\text{sigma}} = 0.5$, yields an accuracy score within $0.015$ of the best cross-validated model in the full dataset setting, on all three datasets.

\subsection{Reproducibility and Representational stability}
\label{sec:reproducibility}
\noindent {\bf Reproducibility} As mentioned in the introduction, a key criterion for learning concept representations beyond predictive performance is {\em reproducibility}, i.e., learning similar representations for different random initializations on the same training data. To evaluate this, we compared the learned embeddings from 20 differently initialized VICE and SPoSE models. 
The first aspect of reproducibility we investigate is whether the models yield a consistent number of embedding dimensions across random seeds. 

As reported in Table~\ref{tab:reliability}, VICE yielded fewer dimensions than SPoSE with less variance in the number of dimensions for all three datasets. The embedding dimensionality was significantly smaller according to an independent $t$-test for \textsc{Things} ($p < 0.001$) and \textsc{Food} ($p < 0.001$), but not for \textsc{Adjectives} ($p = 0.108$). The difference in the standard deviations for \textsc{Adjectives} was statistically significant according to a two-sided \textit{F}-test ($ p = 0.002$), but was not statistically significant for \textsc{Things} ($ p = 0.283$) or \textsc{Food} ($ p = 0.378$). In addition, we observed that the identified dimensionality is consistent as long as $d$ is chosen to be sufficiently large (see Appendix~\ref{app:experimental_details}).

\begin{table}[h!]

\caption{Reproducibility of VICE and SPoSE. Reported are the means and standard deviations with respect to selected dimensions and the average reproducibility score of dimensions (in \%) across random seeds. Bold means VICE performed statistically significantly better with $\alpha = 0.05$.}
\centering
    \begin{footnotesize}
    \begin{tabularx}{\linewidth} {@{}lXXXXr@{}}
    \toprule 
    &\multicolumn{2}{c}{\textsc{VICE}}&\multicolumn{2}{c}{\textsc{SPoSE}}\\
    \textsc{Data$\setminus$ Metric} & Selected Dims. & Reproducibility & Selected Dims. & Reproducibility \\
    \midrule
    \textsc{Things} & \centering  $\mathbf{44}~(1.59)$ & \centering $\mathbf{87.01\%}$ & \centering $52~(1.82)$ & \centering $81.30$\%  & \\
    \textsc{Adjectives} & \centering  $21~(\mathbf{0.77})$ & \centering $76.76$\% & \centering $22~(1.53)$  & \centering  $71.64$\% & \\
    \textsc{Food} & \centering  $\mathbf{5}~(0.95)$ &  \centering $\mathbf{87.38\%}$ & \centering $16~(1.02)$  & \centering $62.88$\% & \\
    \bottomrule
    \end{tabularx}
    \end{footnotesize}
\label{tab:reliability}
\end{table}
The second aspect to reproducibility we examine is the degree to which the identified dimensions are similar across random initializations up to a permutation. To calculate the number of reproducible dimensions, we associated each embedding dimension of a model with the most similar embedding dimension across the other models. We quantify reproducibility of a dimension as the average Pearson correlation between one dimension and its best match across the 19 remaining models. In Table~\ref{tab:reliability} we report the average relative number of dimensions with a correlation $> 0.8$ across models. The embedding dimensions in VICE were more reproducible than those in SPoSE. The difference in average reproducibility scores was statistically significant according to a two-sided, independent $t$-test for \textsc{Food} ($p = 0.030$) and \textsc{Things} ($p = 0.008$), but not for \textsc{Adjectives} ($p = 0.466$).

\noindent {\bf Representational stability} For all three datasets, VICE found reproducible and stable dimensions across all $20$ random initializations and converged when the embedding dimensionality had not changed over the past $L = 500$ epochs. The median number of epochs to achieve representational stability was comparable for the similarly-sized \textsc{Things} and \textsc{Adjectives} datasets, but occurred later for \textsc{Food} (likely due to the smaller number of gradient updates per epoch). In Appendix~\ref{app:gradient_based_optimization} we show plots that demonstrate that the convergence criterion as defined in \S\ref{method:pruning} worked reliably for all datasets.

\subsection{Interpretability}
\label{sec:interpretability}

\begin{figure}[h!]
\centering
\begin{subfigure}{.245\textwidth}
    \centering
    \includegraphics[width=0.95\textwidth]{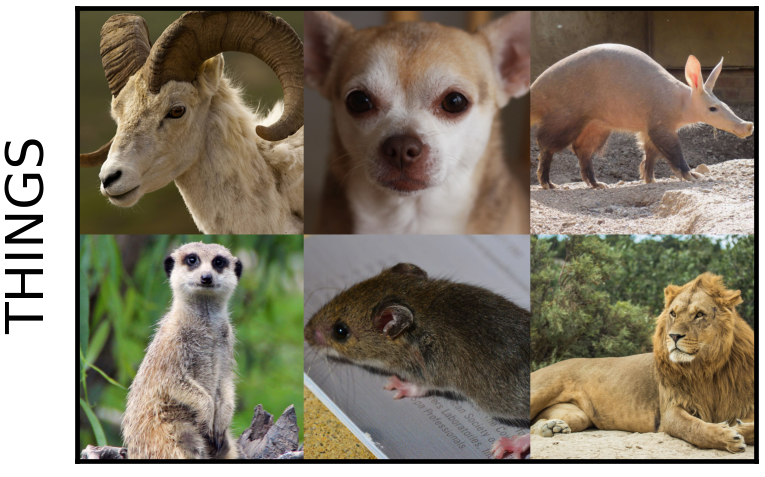}
\end{subfigure}
\begin{subfigure}{.245\textwidth}
    \centering
    \includegraphics[width=0.90\textwidth]{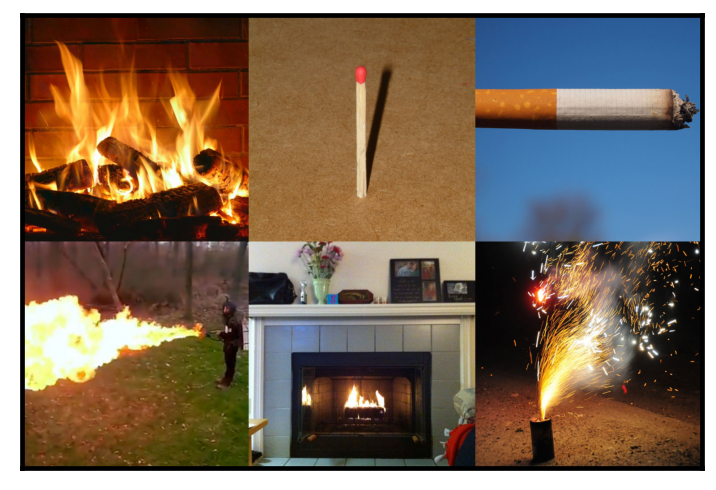}
\end{subfigure}
\begin{subfigure}{.245\textwidth}
    \centering
    \includegraphics[width=0.90\textwidth]{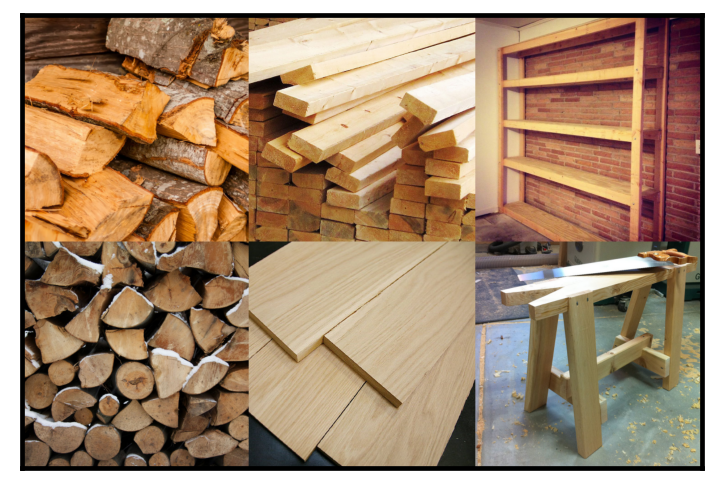}
\end{subfigure}
\begin{subfigure}{.245\textwidth}
    \centering
    \captionsetup{justification=centering}
    \includegraphics[width=0.90\textwidth]{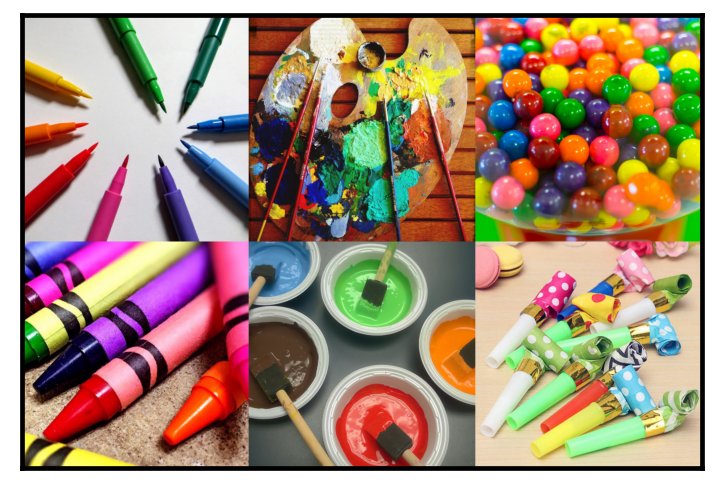}
\end{subfigure}
\begin{subfigure}{.244\textwidth}
    \centering
    \includegraphics[width=0.95\textwidth]{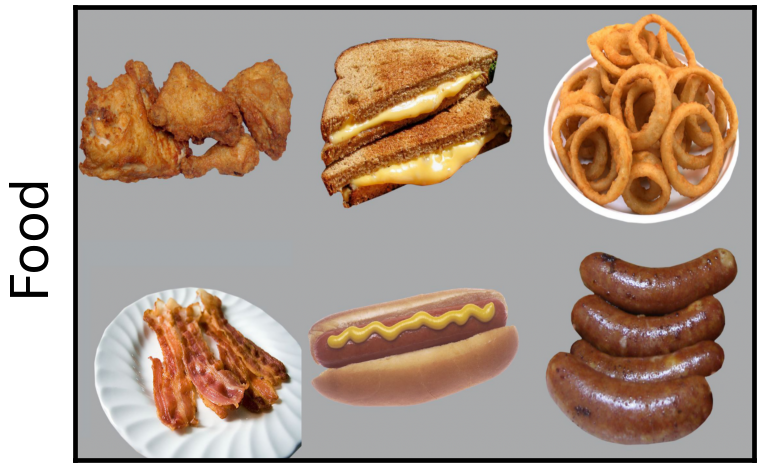}
\end{subfigure}%
\begin{subfigure}{.245\textwidth}
    \centering
    \includegraphics[width=0.90\textwidth]{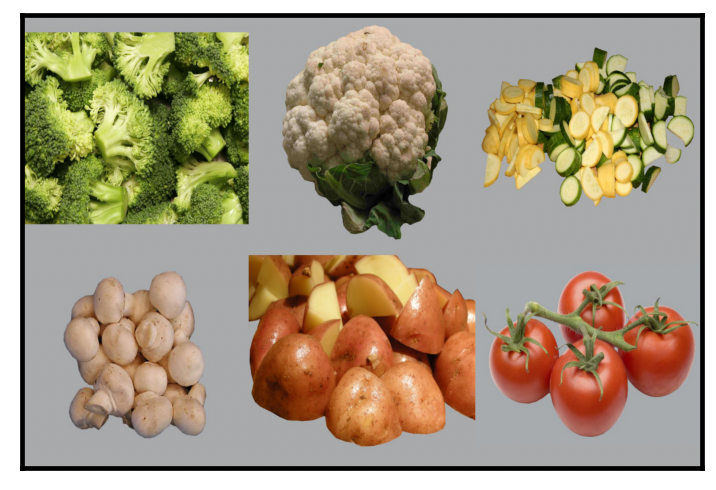}
\end{subfigure}
\begin{subfigure}{.245\textwidth}
    \centering
    \includegraphics[width=0.90\textwidth]{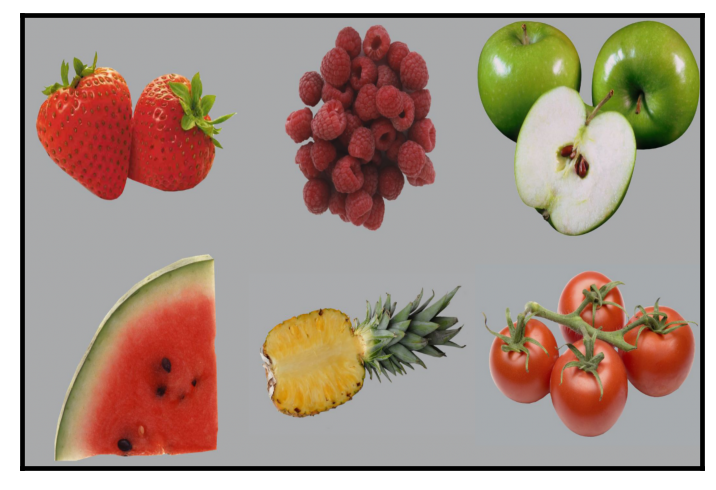}
\end{subfigure}
\begin{subfigure}{.245\textwidth}
    \centering
    \includegraphics[width=0.90\textwidth]{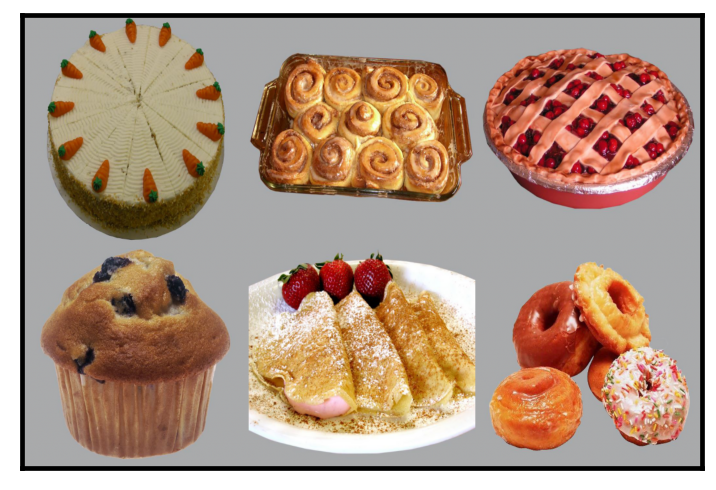}
\end{subfigure}
\begin{subfigure}{.244\textwidth}
    \centering
    \includegraphics[height=0.7in, width=0.95\textwidth]{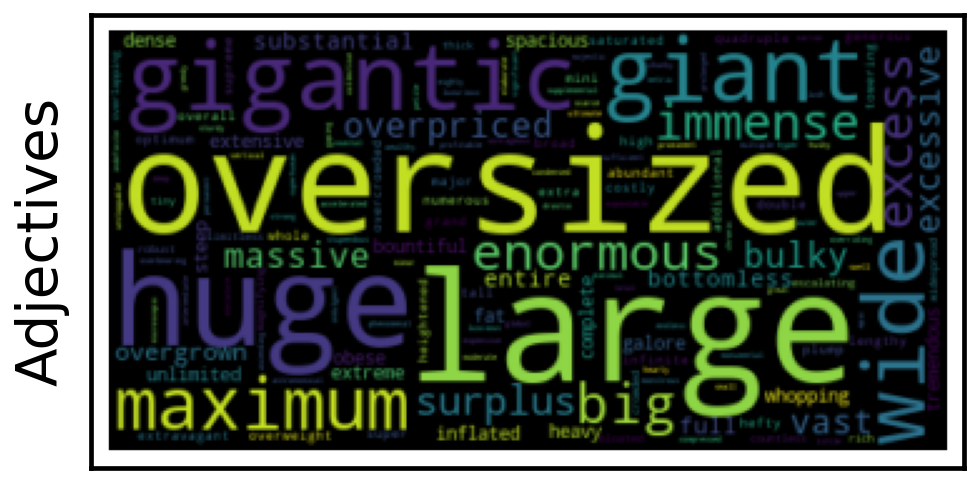}
\end{subfigure}%
\begin{subfigure}{.245\textwidth}
    \centering
    \includegraphics[height=0.7in, width=0.90\textwidth]{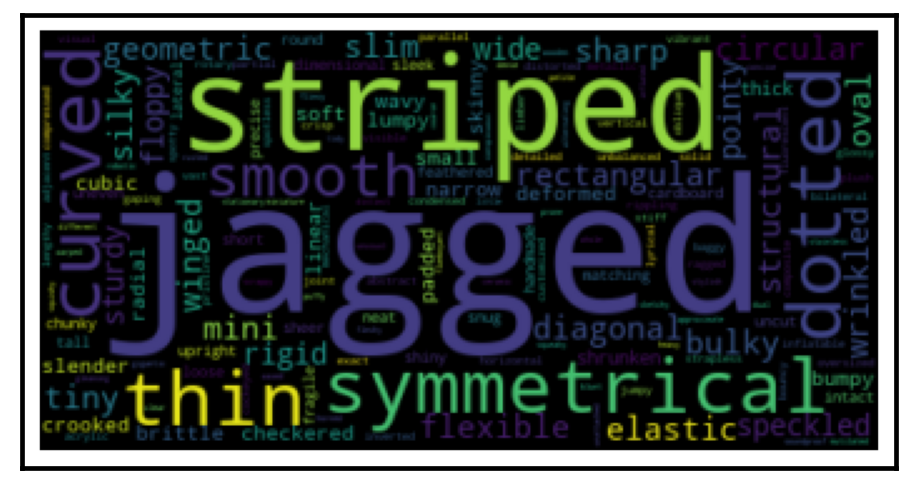}
\end{subfigure}
\begin{subfigure}{.245\textwidth}
    \centering
    \includegraphics[height=0.7in, width=0.90\textwidth]{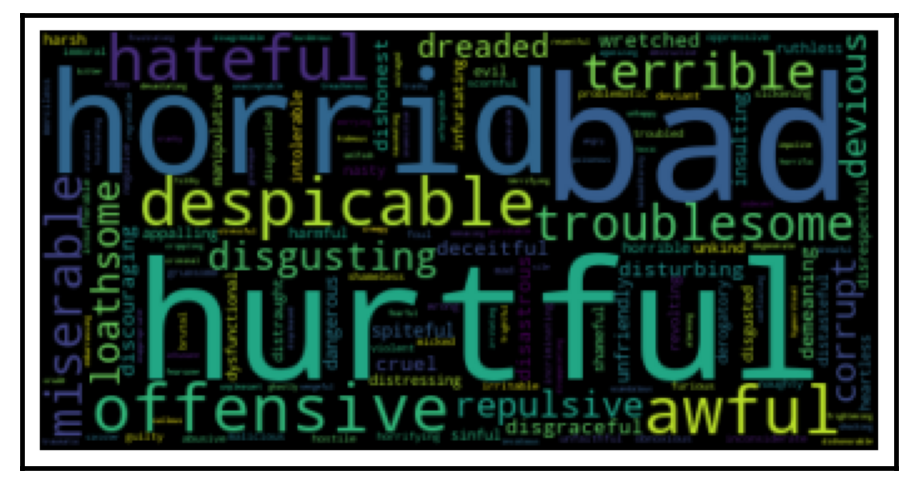}
\end{subfigure}
\begin{subfigure}{.245\textwidth}
    \centering
    \includegraphics[height=0.7in, width=0.90\textwidth]{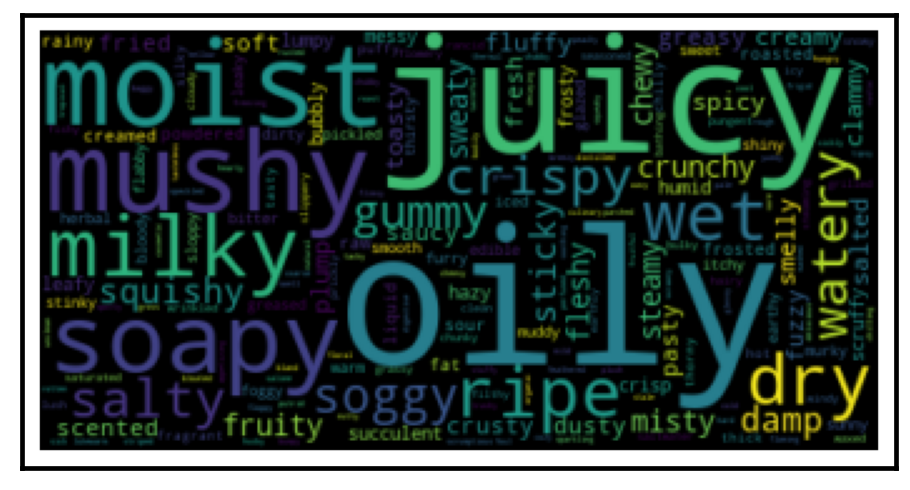}
\end{subfigure}
\caption{Four example VICE dimensions showing the top six objects for \textsc{Things} (\textit{top})
and \textsc{Food} (\textit{middle}), and wordclouds for \textsc{Adjectives} (\textit{bottom}).}
\label{fig:interpretability}
\end{figure}
One of the benefits of SPoSE is the interpretability of the dimensions of its concept embeddings which was evidenced through experiments with humans \citep{hebart2020revealing}. VICE dimensions are equally interpretable due to similar constraints on its embedding space. Therefore, it is easy to sort objects within a dimension of the VICE mean embedding matrix $\mu$ in descending order, to obtain human judgments of what an embedding dimension represents. To illustrate the interpretability of VICE we show in Figure~\ref{fig:interpretability} four example dimensions of a representative model for \textsc{Things}, \textsc{Food}, and \textsc{Adjectives}. For \textsc{Things} the dimensions appear to represent \textit{categorical}, \textit{functional}, \textit{structural}, and \textit{color-related} information. For \textsc{Food} the dimensions appear to be a combination of \textit{fried}, \textit{vegetable}, \textit{fruit}, and \textit{sweet} food items.
For \textsc{Adjectives} the dimensions seem to reflect \textit{size/magnitude}, \textit{visual appearance}, \textit{negative valence}, and \textit{sensory} adjectives. In Appendix~\ref{app:object_dimensions} we show VICE dimensions for \textsc{Things} only, since manuscripts about the other datasets are still in preparation, and, hence, cannot be shown publicly.

\section{Limitations}
\label{sec:limitations}
The goal of our method is to identify general mental representations of objects. The dimensions identified by a model reflect semantic characteristics that explain task performance for many subjects in the experimental subject population (here, Amazon Mechanical Turk subjects in the United States). As such, it is possible that they reflect biases widespread in that population. Furthermore, the choice of population may affect the identified dimensions. That is, a chéf may classify food items differently from a lay subject, and a linguist would likely have a more complex representation of an adjective. The effects of expertise or developmental stage in mental representations are of obvious interest to cognitive scientists. Therefore, we envision further research in those areas, which may additionally provide some indication of how widely representations can vary.

\section{Conclusion}
\label{sec:conclusion}
One of the central goals in the cognitive sciences is the development of computational models of mental representations of object concepts. Such models may serve as a component for other behavioral prediction models or as a basis for identifying the nature of concept representations in the human brain. In this paper we introduced VICE, a novel VI approach for learning interpretable object concept embeddings by modeling human behavior in an triplet odd-one-out task. We showed that VICE predicts human behavior close to the estimated best attainable performance across three datasets and that VICE outperforms a competing method, SPoSE, in low data regimes. In addition, VICE has several characteristics that are desirable for scientific use. It has an automated procedure for determining the number of dimensions sufficient to explain the data, which further enables the detection of convergence during training. This leads to better model reproducibility across different random initializations and hyperparameter settings. As a result, VICE can be used out-of-the-box without requiring to perform an extensive search over random seeds or tuning model hyperparameters. Finally, we introduced a PAC learning bound on the generalization performance for a VICE model. Although VICE assumes a shared mental representation across participants - akin to much of the literature -, we believe that the VI framework can be leveraged to model inter-individual differences, which we plan to do in future work.

\subsubsection*{Acknowledgments}
LM and RV acknowledge support by the Federal Ministry of Education and Research (BMBF) for the Berlin Institute for the Foundations of Learning and Data (BIFOLD) (01IS18037A). LM and MNH acknowledge support by a Max Planck Research Group grant awarded to MNH by the Max Planck Society. FP, CZ, and PM acknowledge the support of the National Institute of Mental Health Intramural Research Program (ZIC-MH002968). PM acknowledges the support of the Naval Postgraduate School's Research Initiation Program. This study utilized the high-performance computational capabilities of the Biowulf Linux cluster at the National Institutes of Health, Bethesda, MD (http://biowulf.nih.gov) and the Raven and Cobra Linux clusters at the Max Planck Computing \& Data Facility, Garching, Germany (https://www.mpcdf.mpg.de/services/supercomputing/). The authors would like to thank Chris Baker for useful initial discussions, Shruti Japee, Jason Avery, and Alex Martin for sharing data, and Erik Daxberger, Lorenz Linhardt, Adrian Hill, Niklas Schmitz, and Marco Morik for valuable feedback on earlier versions of the paper.

\bibliography{neurips}
\bibliographystyle{iclr2022_conference}
\clearpage

\section*{Checklist}
\begin{enumerate}

\item For all authors...
\begin{enumerate}
  \item Do the main claims made in the abstract and introduction accurately reflect the paper's contributions and scope?
    \answerYes{}
  \item Did you describe the limitations of your work?
    \answerYes{See \S~\ref{sec:conclusion}.}
  \item Did you discuss any potential negative societal impacts of your work?
    \answerNA{}
  \item Have you read the ethics review guidelines and ensured that your paper conforms to them?
    \answerYes{}{}
\end{enumerate}

\item If you are including theoretical results...
\begin{enumerate}
  \item Did you state the full set of assumptions of all theoretical results?
    \answerYes{} 
        \item Did you include complete proofs of all theoretical results?
    \answerYes{}
\end{enumerate}

\item If you ran experiments...
\begin{enumerate}
  \item Did you include the code, data, and instructions needed to reproduce the main experimental results (either in the supplemental material or as a URL)? \answerYes{See \url{https://github.com/LukasMut/VICE} for our codebase. See \url{https://osf.io/jum2f/} and \url{https://things-initiative.org/} for data.}
    
  \item Did you specify all the training details (e.g., data splits, hyperparameters, how they were chosen)?
    \answerYes{See \S~\ref{sec:data}, \S~\ref{sec:experimental_setup} and Appendix~\ref{app:experimental_details}.}
    \item Did you report error bars (e.g., with respect to the random seed after running experiments multiple times)?
    \answerYes{See Figure~\ref{fig:data_effiency} and Table~\ref{tab:reliability}.}
    \item Did you include the total amount of compute and the type of resources used (e.g., type of GPUs, internal cluster, or cloud provider)?
    \answerYes{See Appendix~\ref{app:experimental_details}}
\end{enumerate}

\item If you are using existing assets (e.g., code, data, models) or curating/releasing new assets...
\begin{enumerate}
  \item If your work uses existing assets, did you cite the creators?
    \answerYes{}
  \item Did you mention the license of the assets?
    \answerYes{}
  \item Did you include any new assets either in the supplemental material or as a URL?
    \answerYes{See \url{https://github.com/LukasMut/VICE}.}
  \item Did you discuss whether and how consent was obtained from people whose data you're using/curating?
    \answerYes{}
  \item Did you discuss whether the data you are using/curating contains personally identifiable information or offensive content?
    \answerYes{See Appendix~\ref{app:dataset_details}.}
\end{enumerate}

\item If you used crowdsourcing or conducted research with human subjects...
\begin{enumerate}
  \item Did you include the full text of instructions given to participants and screenshots, if applicable?
    \answerNA{}
  \item Did you describe any potential participant risks, with links to Institutional Review Board (IRB) approvals, if applicable?
    \answerNA{}
  \item Did you include the estimated hourly wage paid to participants and the total amount spent on participant compensation?
    \answerNA{}
    
\end{enumerate}

We re-used crowdsourced data collected by various collaborators, and attribute it to them in the main paper. Even though we did not collect these data ourselves, we provide a summary of the information asked for in 5. in Appendix~\ref{app:dataset_details}, as well as a short description of the task and stimuli in Section~\ref{method:triplet_task}.

\end{enumerate}

\appendix
\clearpage

\section{Triplet task}
\label{app:triplet_task}

The triplet task compares favorably to other possible alternatives for eliciting human pairwise object similarity judgments.  Asking for similarity ratings directly, on a numeric or qualitative scale, introduces the difficulty of differing inter-individual calibration of ratings.  Having a choice task reduces this inter-individual variability by reducing the number of possible actions the user can make, e.g., three in the triplet task.  Within choice tasks, there are also $k$-way forced choice tasks which specify a reference object and query the user for the most similar object to the reference among a list of $k$ candidate objects.  It seems likely that the choice in the forced choice task is not strictly driven by pairwise item similarities between the choices and the reference, but that the reference may influence the selection of features used to assess item similarities.  For instance, the subject may choose the most prominent feature associated with the reference, and then make a uni-dimensional assessment of which candidate scores highest on that feature.  This is less likely to be an issue in the triplet task because with no reference object, the subject has to consider all three pairs of items to evaluate which pair has greater similarity than the other two pairs.  In particular, it would be less feasible to try to find a single feature that would be relevant for all three pairs of objects.  Hence, we speculate that the triplet task may promote more multidimensional similarity assessments than the $k$-way forced-choice task.

Many studies find it convenient to pool triplet tasks from multiple participants.  This could be justified under an assumption where subjects within a population all share the same common mental representations of items, which in turn determine the pairwise similarities that govern the choice of the odd-one-out selection.  Alternatively, it could be the case that the representation of each item has random variation, both within a subject and across subjects.  In such a scenario, it could be useful to use a model which assumes that a \emph{probabilistic} (rather than a \emph{deterministic}) embedding governs the triplet choices. As we discussed theoretically in \S\ref{sec:vice} and showed empirically in \S\ref{sec:experiments}, variational bayesian inference is one way that appears appropriate to model such a probabilistic embedding.

\section{Objective function}

\subsection{Probability model and the log-likelihood function}
\label{app:data_loglikelihood}

In the following we define the log-likelihood function of the data given the embedding matrix, $\log p(\mathcal{D}|X)$. Recall that $\mathcal{D} \coloneqq \big(\{i_s, j_s, k_s\}, \{y_s, z_s\}\big)_{s=1}^{n}$. For simplicity, we assume that each triplet $\{i_{s}, j_{s}, k_{s}\}$ is chosen uniformly at random from the collection of all possible sets of object triplets $\mathcal{T}$. That is, $\{i_{1}, j_{1}, k_{1}\},\ldots, \{i_{n}, j_{n}, k_{n}\} \overset{\text{i.i.d.}}{\sim} \mathcal{U}(\mathcal{T})$. Hence, the probability for choosing a triplet, $p(\{i_{s}, j_{s}, k_{s}\})$, is the same for all $\mathcal{D}_{s} \in \{\mathcal{D}_{1},\ldots,\mathcal{D}_{n}\}$ and can therefore be treated as a coefficient for computing $p(\mathcal{D}|X)$.

For our probability model the log-likelihood of the data given the embedding matrix, $\log p(\mathcal{D}|X)$, can then be defined as
\begin{align}
    \log p(\mathcal{D}|X)~=&~\log \prod_{s=1}^{n} p(\mathcal{D}_{s}|X)\notag\\
    =&~\log \prod_{s=1}^{n} p(\{y_s,z_s\},\{i_s, j_s, k_s\}|X)\notag\\
    =&~\log \prod_{s=1}^{n} p(\{y_s,z_s\}|\{i_s, j_s, k_s\}, X)p(\{i_s, j_s, k_s\}|X)\\
    =&~\log \prod_{s=1}^{n} p(\{y_s,z_s\}|\{i_s, j_s, k_s\},X)p(\{i_s, j_s, k_s\}).\notag
\end{align}
Recall that $p(\{i_s, j_s, k_s\})$ is a coefficient that is the same for all $s \in \{1,\ldots,n\}$. We let $C \coloneqq p(\{i_{s}, j_{s}, k_{s}\})$ and thus we can further rewrite, 
\begin{align*}
    \log \prod_{s=1}^{n} p(\{y_s,z_s\}|\{i_s, j_s, k_s\},X)p(\{i_s, j_s, k_s\})=&~\log \left[C^{n}\prod_{s=1}^{n} p(\{y_s,z_s\}|\{i_s, j_s, k_s\},X)\right]\\
    =&~\log C^{n} + \log \prod_{s=1}^{n} p(\{y_s,z_s\}|\{i_s, j_s, k_s\},X)\\
    =&~n\log C +\sum_{s=1}^{n} \log p(\{y_s,z_s\}|\{i_s, j_s, k_s\},X),
\end{align*}
where $(6)$ follows from the chain rule of probability. Note that the constant $n\log C$ can be ignored in the minimization of the VICE objective function, whose derivation is outlined in Appendix~\ref{app:vbi}.

\subsection{KL divergence}
\label{app:vbi}
In VI, one minimizes the KL divergence between $q_\theta(X)$, the variational posterior, and, $p(X|\mathcal{D})$, the true posterior,
\begin{equation*}
    \argmin_\theta D_{\text{KL}} ( q_{\theta}(X) \Vert p(X|\mathcal{D}) ),
\end{equation*}
where
\begin{align*}
    D_{\text{KL}}( q_{\theta}(X) \Vert p(X|\mathcal{D}) )~=&~\mathbb{E}_{q_\theta(X)} \left[ \log \frac{q_{\theta}(X)}{p(X|\mathcal{D}) } \right]\\
    =&~\mathbb{E}_{q_\theta(X)} \big[\log q_{\theta}(X) -  \log p(X|\mathcal{D}) \big]\\ 
    =&~\mathbb{E}_{q_\theta(X)} \left[ \log q_{\theta}(X) - \log \frac{p(\mathcal{D}|X)p(X)}{p(\mathcal{D})} \right]\\ 
    =&~\mathbb{E}_{q_\theta(X)} \big[\log q_{\theta}(X) - \log p(X) - \log p(\mathcal{D}|X) \big] + \log p(\mathcal{D})\\
    =&~\mathbb{E}_{q_\theta(X)} \left[ \log q_{\theta}(X) - \log p(X) - n\log C -\sum_{s=1}^{n} \log p(\{y_s,z_s\}|\{i_s, j_s, k_s\},X) \right]\\ &+ \log p(\mathcal{D})\\
    =&~\mathbb{E}_{q_\theta(X)} \left[ \log q_{\theta}(X) - \log p(X) - \sum_{s=1}^{n} \log p(\{y_s,z_s\}|\{i_s, j_s, k_s\},X) \right]\\ &- n\log C + \log p(\mathcal{D}).
\end{align*}

\subsection{VICE objective function}
Because $n\log C$ and $\log p(\mathcal{D})$ are constants and not functions of the variational parameters, we can ignore both terms in the minimization and get the following VICE objective
\begin{equation*}
     \argmin_\theta~\mathbb{E}_{q_\theta(X)}\left[\log q_\theta(X) - \log p(X) - \sum_{s=1}^{n} \log p(\{y_{s}, z_{s}\}|\{i_s, j_s, k_s\},X)\right].
\end{equation*}
Multiplying this by $(1 / n)$, where $n$ is the number of training examples, does not change the minimum of the objective function and results in
\begin{equation*}
     \argmin_\theta~\mathbb{E}_{q_\theta(X)}\left[\frac{1}{n}(\log q_\theta(X) - \log p(X)) - \frac{1}{n}\sum_{s=1}^{n} \log p(\{y_{s}, z_{s}\}|\{i_s, j_s, k_s\},X)\right].
\end{equation*}

\section{Optimization, convergence and prior}

\subsection{Gradient-based optimization}
\label{app:gradient_based_optimization}

In gradient-based optimization, the gradient of an objective function with respect to the parameters of a model, $\nabla{\mathcal{L(\theta)}}$, is used to iteratively find parameters $\hat{\theta}$ that minimize that function. Equation~\ref{eq:vi_objective} computes the expected log-likelihood of the entire training data. However, using every training data point to compute a gradient update is computationally expensive for large datasets and often generalizes poorly for non-convex objective functions \citep{smith2020generalization}. In VICE, we stochastically approximate the training log-likelihood using random subsets (i.e., mini-batches) of the training data, where each mini-batch consists of $B$ triplets \citep{robbins1951stochastic}. This leads to an objective function that is a doubly stochastic approximation of Equation~\ref{eq:vi_objective} \citep{titsias14}, due to sampling weights from the variational distribution, $q_{\theta} \in \mathcal{Q}$, and sampling a random mini-batch of training examples during each gradient step (i.e., performing mini-batch gradient descent),
\begin{equation*}
     \mathcal{L}_{batch} = \frac{1}{n}\left[\log q_\theta(X_{\theta,\epsilon}) - \log p(X_{\theta,\epsilon})\right] - \frac{1}{B}\sum_{b=1}^{B} \log p(\{y_{b}, z_{b}\}|\{i_b, j_b, k_b\},[X_{\theta,\epsilon}]_+),
\end{equation*}
where $p(\{y, z\}|\{i, j, k\}, X)$ for a single sample is defined in Equation~\ref{eq:triplet_probs}. To find parameters, $\theta$, that optimize Equation~\ref{eq:vi_objective}, we iteratively update both the means, $\mu$, and the standard deviations, $\sigma$, of each VICE dimension, by
\begin{equation*}
    \mu_{t+1} \coloneqq \mu_{t} - \alpha  \nabla_{\mu_{t}} \mathcal{L}_{batch}
\end{equation*}
and
\begin{equation*}
    \sigma_{t+1} \coloneqq \sigma_{t} - \alpha  \nabla_{\sigma_{t}} \mathcal{L}_{batch},
\end{equation*}
where $\alpha$ is the learning rate for $\theta$.

\subsection{Convergence}

We define \emph{convergence} as the point in time, $t^{*}$, where the number of embedding dimensions has not changed by a \textit{single} dimension over the past $L$ epochs. We denote this as the point of \emph{representational stability}. To ensure convergence, we recommend letting $L$ be relatively large (e.g., $L \gg 100$). We found $L = 500$ to work well for our experiments. Figure~\ref{fig:convergence} shows that the convergence criterion defined in \S\ref{method:pruning} worked reliably for all datasets.


\begin{figure}[h!]
\centering
\begin{subfigure}{.329\textwidth}
    \centering
    \includegraphics[height=1.1in, width=1.\textwidth]{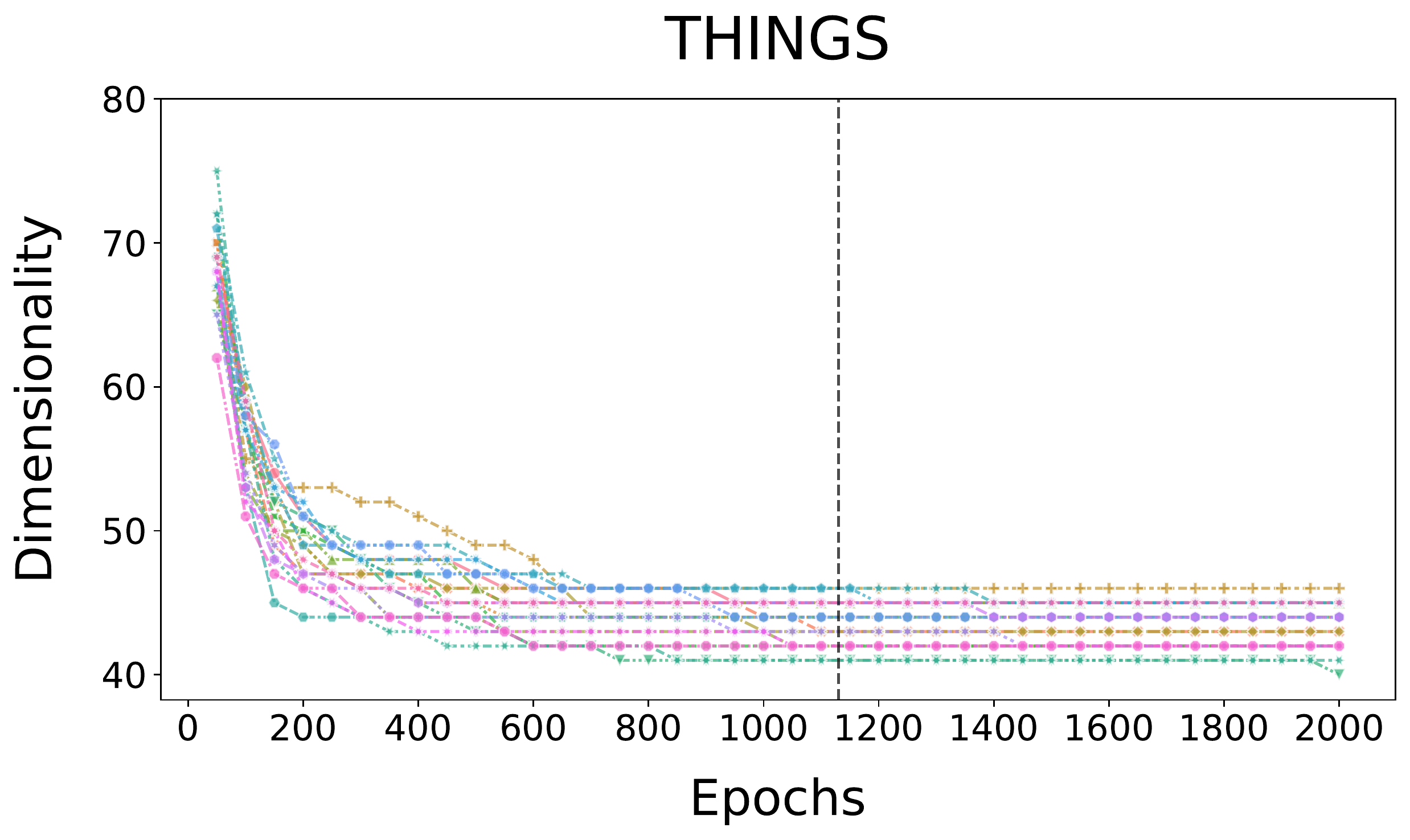}
\end{subfigure}
\begin{subfigure}{.329\textwidth}
    \centering
    \includegraphics[height=1.1in, width=1.\textwidth]{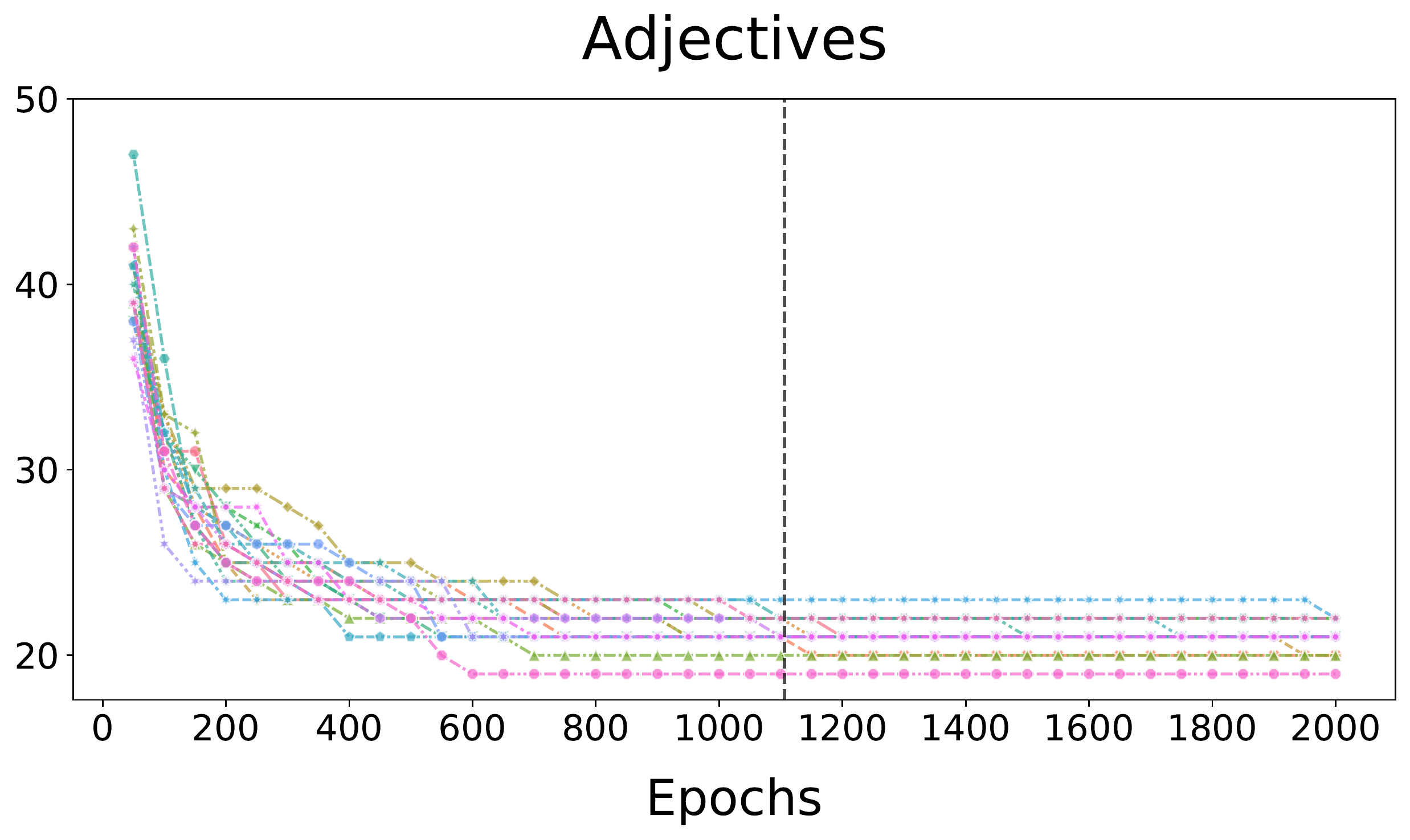}
\end{subfigure}
\begin{subfigure}{.329\textwidth}
    \centering
    \includegraphics[height=1.1in, width=1.\textwidth]{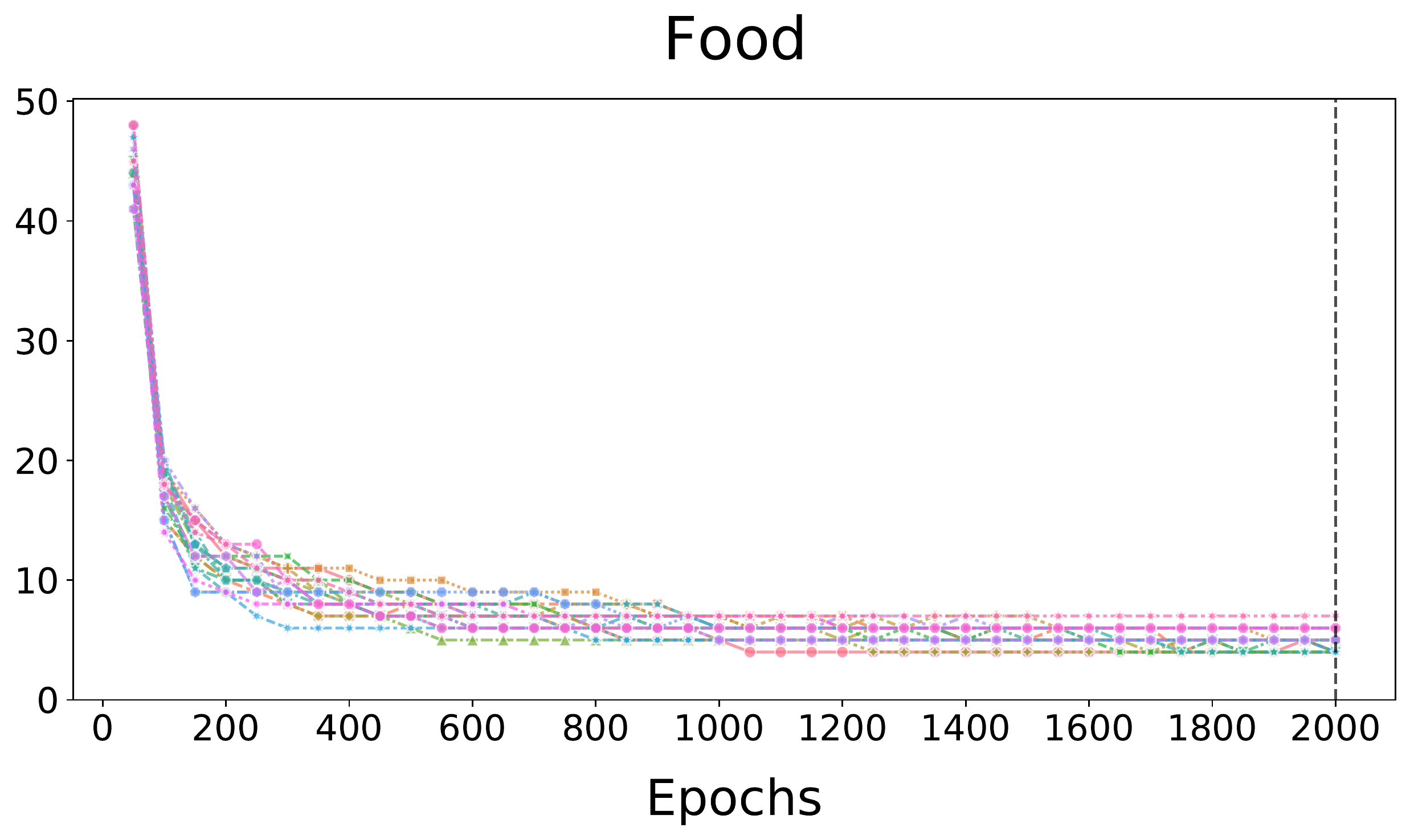}
\end{subfigure}
\caption{These plots show the number of embedding dimensions over time for \textsc{Things}, \textsc{Adjectives}, and \textsc{Food} respectively. Each line in a plot corresponds to a single random seed. Vertical dashed lines indicate the median number of epochs (across random seeds) until the convergence criterion with $L = 500$ was met.}
\label{fig:convergence}
\end{figure}

\subsection{(In-)efficient prior choice}
\label{app:histograms_of_spose_dimensions}

\begin{figure}[h!]
\centering
\begin{subfigure}{.45\textwidth}
    \centering
    \captionsetup{justification=centering}
    \includegraphics[height=1.1in]{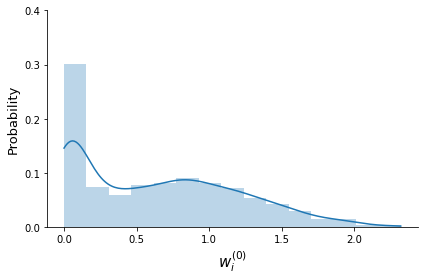}
\end{subfigure}%
\begin{subfigure}{.45\textwidth}
    \centering
    \captionsetup{justification=centering}
    \includegraphics[height=1.1in]{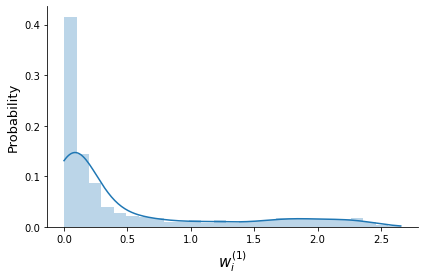}
\end{subfigure}
\caption{Histograms and PDFs of the first two SPoSE dimensions after training.}
\label{fig:spose_dimension}
\end{figure}
As discussed in \S\ref{sec:vbi}, SPoSE imposes a combination of an $\ell_{1}$ penalty and a non-negativity constraint on its embedding values. This is analogous to having an exponential prior on those values. If we consider the distribution of values across objects for the two most important SPoSE dimensions in Figure~\ref{fig:spose_dimension}, we can see that they show a distribution with a spike around $0$, and a much smaller, wide slab of probability mass for non-zero values. Overcoming the exponential prior is data inefficient. SPoSE was developed using a dataset that was orders of magnitude larger than what a typical psychological experiment might collect, a setting in which this issue would likely not manifest itself. However, SPoSE has not been tested on the more common, smaller datasets, and our results suggest that the implicit SPoSE prior leads to suboptimal results in those regimes, when compared to the spike-and-slab prior used by VICE, as shown in Figure~\ref{fig:data_effiency}.

\section{Proof of Proposition 3.1}
\label{app:proof}
\begin{proof}
Recall that
\begin{align*}
    \hat{R}(X) &:= \sum_{s=1}^n \frac{1}{n}\mathbbm{1}\left( \left\{y_s,z_s\right\} \neq \argmax_{\{y, z\}} p\left(\{y, z\}|\{i_s, j_s, k_s\},X\right) \right)
\end{align*}
and
\begin{align*}
    R(X) &:= P_{\{y', z'\}|\{i', j', k'\}}\left( \left\{  y',z' \right\} \neq \argmax_{\{y, z\}} p\left(\{y, z\}|\{i', j', k'\},X\right)\right).
\end{align*}
Using the union bound we have that
\begin{align}
P\left(\sup_{X \in \mathbb{A}^{ m\times d}} \hat{R}(X)-  R(X) \ge \epsilon\right) 
  &\le\sum_{X \in \mathbb{A}^{ m\times d}} P\left( \hat{R}(X)-  R(X) \ge \epsilon\right). \label{eqn:union-bound}
\end{align}
Since $\mathbb{E}\left[\hat{R}(X)\right] = R(X)$ and the summands of $\hat{R}(X)$ are independent and satisfy
\begin{align*}
  0 \le\frac{1}{n}\mathbbm{1}\left( \left\{y_s,z_s\right\} \neq \argmax_{\{y, z\}} p\left(\{y, z\}|\{i_s, j_s, k_s\},X\right) \right)\le \frac{1}{n}
\end{align*}
for all $X\in \mathbb{A}^{ m\times d}$ we can apply Hoeffding's Inequality to yield 
\begin{equation}
P\left( \hat{R}(X)-  R(X) \ge \epsilon\right)
   \le \exp\left(-\frac{2\epsilon^2}{\sum_{i=1}^n \left(\frac{1}{n}-0\right)^2} \right)\\
   = \exp\left(-2\epsilon^2n \right) \label{eqn:hoeffdings-bound}
\end{equation}
for all $X\in \mathbb{A}^{ m\times d}$. The set $\mathbb{A}^{ m\times d}$ contains $(k+1)^{md}$ elements so combining Equation~\ref{eqn:union-bound} and Equation~\ref{eqn:hoeffdings-bound} we get
\begin{align*}
P\left(\sup_{X \in \mathbb{A}^{ m\times d}} \hat{R}(X)-  R(X) \ge \epsilon\right) 
  &\le (k+1)^{md}\exp\left(-2\epsilon^2n \right)\\
  &= \exp\left(md\log(k+1) - 2\epsilon^2n \right)\\
  \iff P\left(\sup_{X \in \mathbb{A}^{ m\times d}} \hat{R}(X)-  R(X) < \epsilon\right)&\ge 1- \exp\left(md\log(k+1) - 2\epsilon^2n \right). 
\end{align*}

To arrive at the proposition statement we need that 

\begin{equation*}
1-\delta \ge 1- \exp\left(md\log(k+1) - 2\epsilon^2n \right).
\end{equation*}
Solving for $n$ we get
\begin{align*}
    1- \exp\left(md\log(k+1) - 2\epsilon^2n \right)     &\ge1-\delta\\
    \iff \exp\left(md\log(k+1) - 2\epsilon^2n \right)   &\le\delta\\
    \iff md\log(k+1) - 2\epsilon^2n    &\le\log(\delta)\\
    \iff \left( md\log(k+1) +\log(1/\delta)\right)/(2\epsilon^2) &\le n.
\end{align*}
\end{proof}

\section{Experimental details}
\label{app:experimental_details}

\paragraph{Training} Although we have developed a reliable convergence criterion for VICE (see \S\ref{method:pruning}), to guarantee a fair comparison between VICE and SPoSE, each model configuration was trained, using 20 different random seeds for $2000$ epochs. Each VICE model was initialized with two weight matrices, $\mu \in \mathbb{R}^{m \times d}$ and $\log{(\sigma)} \in \mathbb{R}^{m \times d}$, where $m$ refers to the number of unique objects in the dataset (\textsc{Things}: $m = 1854$; \textsc{Adjectives}: $m = 2372$; \textsc{Food}: $m = 36$) and $d$, the initial dimensionality of the embedding, was set to $100$ (the $\log$ ensures that $\sigma$ is positive). In preliminary experiments, we observed that for sufficiently large $d$ our dimensionality reduction method (see \S\ref{method:pruning}) prunes to similar representations, regardless of the choice of $d$. This is why we did not consider models with larger initial embedding dimensionality.

\paragraph{Weight initialization} We initialized the means of the variational distributions, $\mu$, following a Kaiming He initialization \citep{DBLP:conf/iccv/HeZRS15}. The logarithms of the scales of the variational distributions, $\log{(\sigma)}$, were initialized with $\epsilon = -1/s_{\mu}$, where $s_{\mu}$ is the standard deviation over the entires of $\mu$, so $\log(\sigma) = \epsilon  \mathbf{1}$. 

\paragraph{Hyperparameter grid} The final grid was the Cartesian product of the following hyperparameter sets: $\pi_{\text{spike}} = \{0.1, 0.2, 0.3, 0.4, 0.5, 0.6, 0.7, 0.8, 0.9\}$, $\sigma_{\text{spike}} = \{0.125, 0.25, 0.5, 1.0, 2.0 \}$, $\sigma_{\text{slab}} = \{0.25, 0.5, 1.0, 2.0, 4.0, 8.0\}$, subject to the constraint $\sigma_{\text{spike}} \ll \sigma_{\text{slab}}$, where combinations that did not satisfy the constraint were discarded. We observed that setting $\sigma_{\text{slab}} > 8.0$ led to numerical overflow issues during optimization.
For SPoSE, we used the same range as was done in \citet{ZhengPBH19}, with a finer grid of 64 values.

\paragraph{Optimal hyperparameters} We found the optimal VICE hyperparameter combination through a two step procedure. First, among the final $180$ combinations (see Cartesian product above), we applied our pruning method (see \S\ref{method:pruning}) to each model and kept the subsets of dimensions where more than 5 objects had non-zero weight. For SPoSE we used the pruning heuristic proposed in \citet{ZhengPBH19}.  We defined the optimal hyperparameter combination as that with the lowest average cross-entropy error on the validation set across twenty different random initializations. The optimal hyperparameter combinations for VICE and SPoSE on the full datasets are reported in Table~\ref{tab:optimal_hyperparams}.

\begin{table}[ht!]
\caption{Optimal hyperparameter combinations for VICE and SPoSE according to the average cross-entropy error on the validation set for the three datasets \textsc{Things}, \textsc{Adjectives}, and \textsc{Food}.}
\centering
    \begin{footnotesize}
    \begin{tabularx}{\linewidth} {@{}lXXXXr@{}}
    \toprule 
    &\multicolumn{3}{c}{\textsc{VICE}}&\multicolumn{1}{c}{\textsc{SPoSE}}\\
    \textsc{Data$\setminus$ Hyperparam.} & \centering $\sigma_{\text{spike}}$ & \centering $\sigma_{\text{slab}}$ & \centering $\pi$ & \centering $\lambda$ & \\
    \midrule
    \textsc{Things} & \centering  $0.125$  & \centering $0.5$ & \centering $0.6$ & \centering $5.75$  & \\
    \textsc{Adjectives} & \centering  $0.25$  & \centering $0.5$ & \centering $0.6$ & \centering $4.96$  & \\
    \textsc{Food} & \centering  $0.25$  & \centering $1.0$ & \centering $0.8$ & \centering $2.90$  & \\
    \bottomrule
    \end{tabularx}
    \end{footnotesize}
\label{tab:optimal_hyperparams}
\end{table}

\paragraph{Computational resources} This study utilized the high-performance computational capabilities of the Biowulf Linux cluster at the National Institutes of Health, Bethesda, MD (\url{http://biowulf.nih.gov}) and the Raven and Cobra Linux clusters at the Max Planck Computing \& Data Facility (MPCDF), Garching, Germany (\url{https://www.mpcdf.mpg.de/services/supercomputing/}). The total number of CPU hours used were approximately $40,000,000$ (Biowulf) and $5,000$ (MPCDF).

\paragraph{Accuracy upper bound}

For \textsc{Things} \citep{ZhengPBH19,hebart2020revealing} and \textsc{Adjectives}, a random subset of triplets was chosen to be presented multiple times to different participants. For a given triplet - repeated over many participants - this provides a way to estimate the distribution of responses over all participants. If for a given triplet the response distribution is $(0.2, 0.3, 0.5)$, then the best predictor for the participants' responses is the third object. This results in an accuracy score of 50\%, averaged across repetitions. Alternatively, one may observe a distribution of $(0.1,0.8,0.1)$ for a different triplet. The best one could do is to identify the second object as the odd-one-out, and get 80\% accuracy. From this, we can see that no classifier can do worse than 33\%. Taking the average best prediction accuracy over all of the repeated triplets gives us an estimate for the best possible average prediction score. This is defined to be the upper bound for the prediction performance.

\section{Generalization error bound}

\subsection{Algorithm for generalization error upper bound}
\label{app:retrospective_algorithm}

\begin{algorithm}
    \caption{Algorithm for generalization error upper bound via adaptive quantization}
    \label{alg:adaptive_quantization}
    \begin{algorithmic}
    \Require $\mu$
    \State $M \gets \max(\mu)$
    \State $\{\Delta_1,\hdots, \Delta_m\} \gets \{0.05, ..., 1.0\}$ \Comment{pre-determined set of quantization scales}
    \State $\alpha \gets 0.05$ \Comment{desired Type I error control rate}
    \State $\delta \gets \frac{\alpha}{m}$
    \For{$i \in \{1, \hdots, m\} $}
        \State $\mu^{\dagger}_{i} \coloneqq \quantize{(\mu, \Delta_{i})}$ \Comment{quantization with $\Delta_{i}$}
        \State $\hat{R}_i = \loss{(\mu^{\dagger}_{i}, \mathcal{D})}$ \Comment{(training) error}
        \State $\bar{R}_i \coloneqq \hat{R}_i + \sqrt{\frac{m d \log(\lceil M/\Delta_i \rceil + 1) + \log(1/\delta)}{2N_{\text{train}}}}$ \Comment{generalization upper bound}
    \EndFor
    \State $i^{*} \coloneqq \argmin \{\bar{R}_{i}, ..., \bar{R}_{m}\}$
    \Ensure $(\mu^{\dagger}_{i^{*}}~, \bar{R}_{i^{*}})$ \Comment{$\bar{R}_{i^{*}}$ holds with probability $1 - \alpha$}
    \end{algorithmic}
\end{algorithm}

We have the flexibility of choosing the quantization scale post-hoc.  As long as we search over a pre-specified set of $m$ quantization scales $\{\Delta_1,\hdots, \Delta_m\}$, using a union bound, the PAC bound holds simultaneously for all quantized embeddings with probability at least $1-m\delta$.  Therefore, we can find the quantization scale that gives us the best probabilistic upper bound on generalization error. 

\subsection{Quantization}
\label{app:quantization_results}

This section describes a number of empirical findings which support the feasibility of obtaining useful bounds by using our proposed quantization-based PAC bound (see \S\ref{sec:pac_bound}) and the algorithm for obtaining retrospective generalization bounds (see Appendix~\ref{app:retrospective_algorithm}).

Recall that we require two assumptions for our bounding approach to be effective. Specifically, we assume,

 \begin{enumerate}
     \item Sparse embeddings obtained by either SPoSE or VICE can be quantized in a relatively coarse fashion, with only marginal losses in predictive performance.
     \item There exists an upper bound $M$ on the largest value in an embedding.
 \end{enumerate}
 
In the following we present empirical results which demonstrate that those assumptions are satisfied for the three datasets, \textsc{Things}, \textsc{Adjectives}, and \textsc{Food}, which we have used to evaluate VICE.

\begin{figure*}[ht!]
\centering
\begin{subfigure}{.33\textwidth}
    \centering
    \includegraphics[height=1.2in, width=1.0\textwidth]{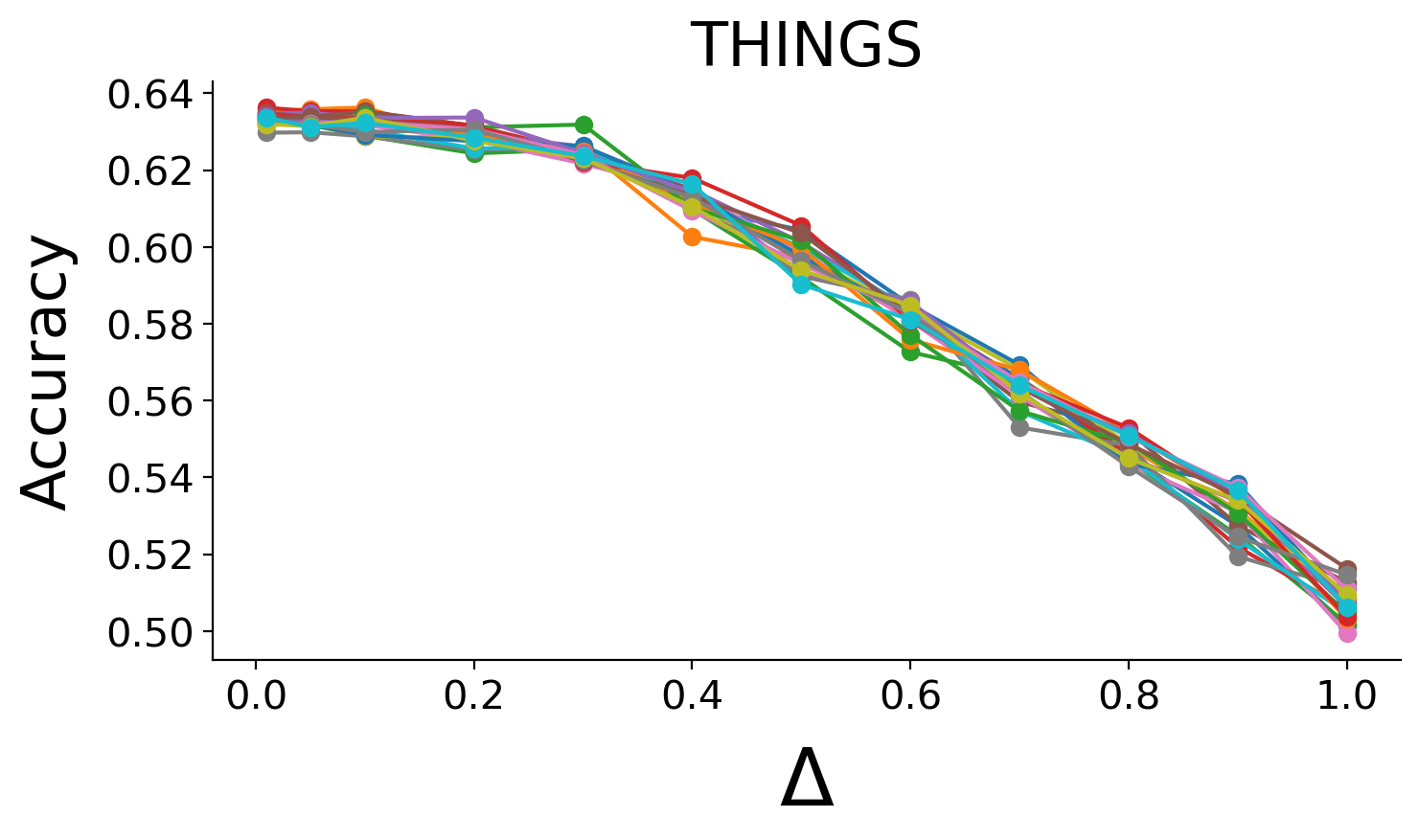}
\end{subfigure}%
\begin{subfigure}{.33\textwidth}
    \centering
    \includegraphics[height=1.2in, width=1.0\textwidth]{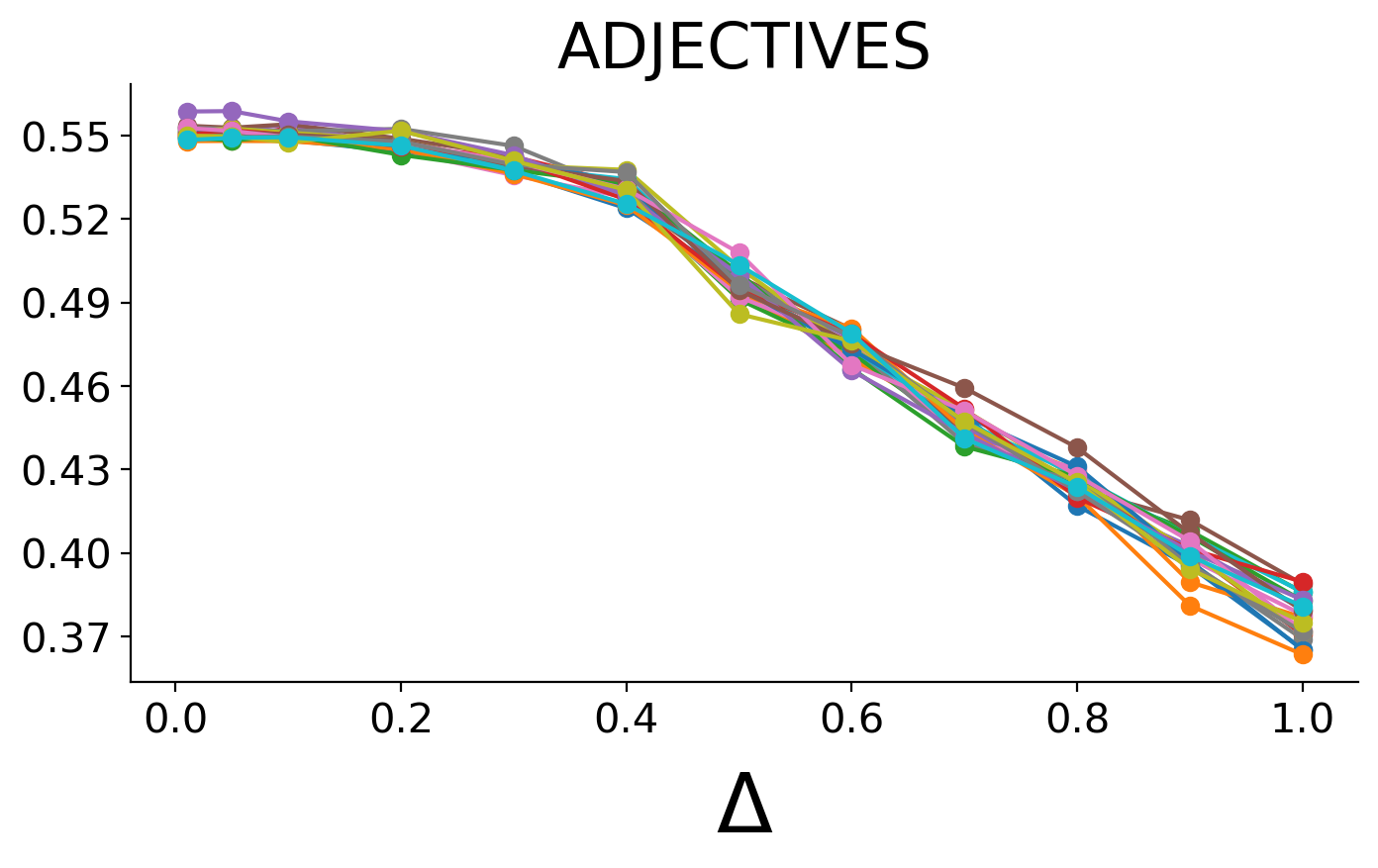}
\end{subfigure}
\begin{subfigure}{.33\textwidth}
    \centering
    \includegraphics[height=1.2in, width=1.0\textwidth]{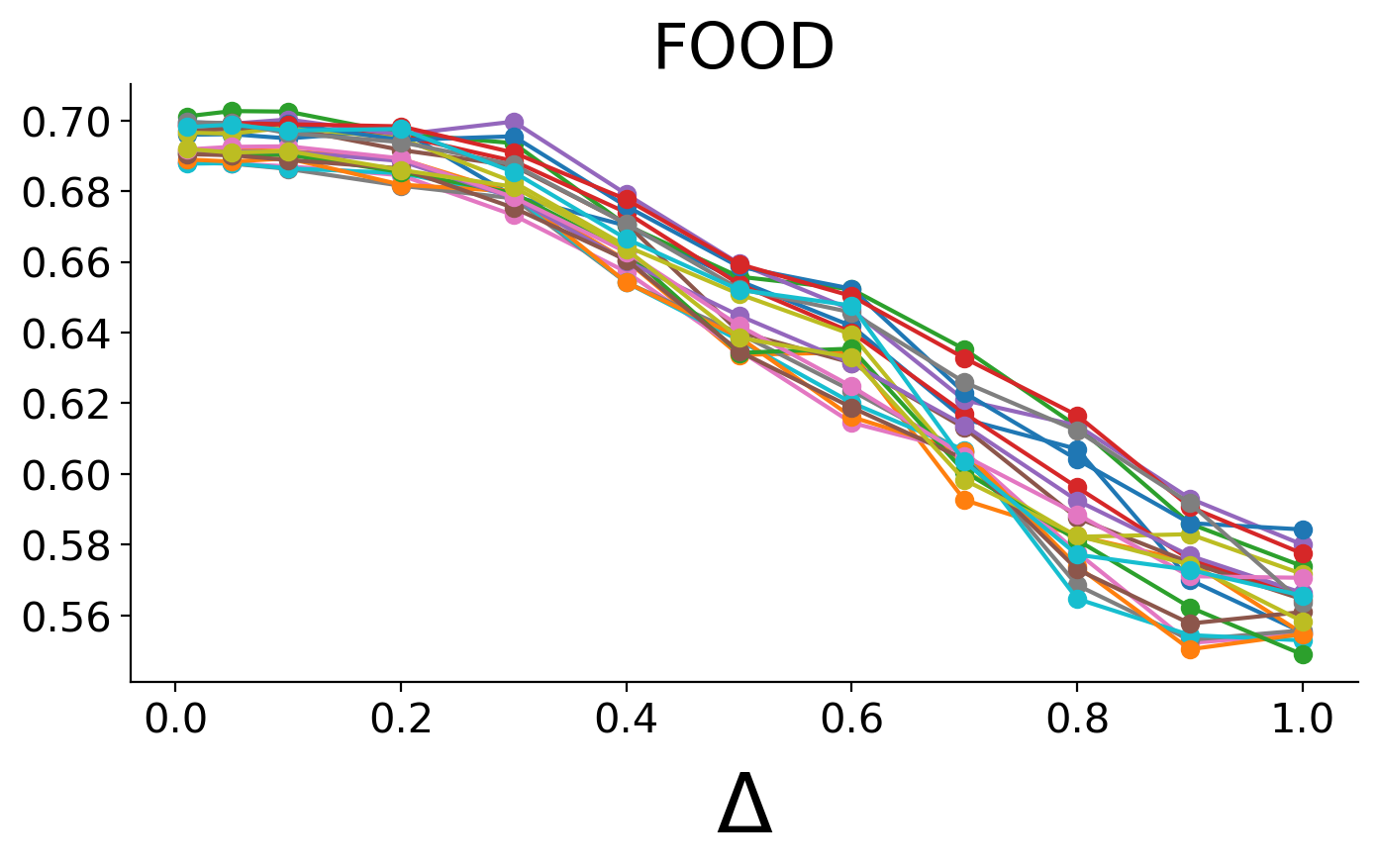}
\end{subfigure}%
\caption{Test performance as a function of the quantization scale, $\Delta$, for twenty random initializations for \textsc{Things}, \textsc{Adjectives}, and \textsc{Food}.}
\label{fig:quantization_test_errors}
\end{figure*}

Recall that it is possible to choose a certain quantization scale, $\Delta$, and round an embedding value to a non-negative integer multiple of $\Delta$.
While we employ quantization mainly to use STL bounds for finite hypothesis classes, it could have benefits for interpretation as well (e.g., a dimension would consist of labels such as zero (0), very low (0.5), low (1), medium (1.5), high (2)).

In Figure~\ref{fig:quantization_test_errors} we show generalization performances of VICE models as a function of the quantization scale, $\Delta$, for \textsc{Things}, \textsc{Adjectives}, and \textsc{Food}. To quantize the VICE embeddings, we used the same set of quantization scales for every dataset, $\Delta \in \{0.05, \hdots, 1.0\}$, as defined in Algorithm~\ref{alg:adaptive_quantization}. As $\Delta$ increases, the number of bins, to which the embedding values can possibly be assigned, decreases. That is, the space of possible embedding values gets smaller as a function of increasing $\Delta$, and, therefore, maintaining generalization performance becomes more difficult. Note that quantization with $\Delta = 0$ is equivalent to performing inference with the original embeddings.

As a result, the optimal generalization bounds that we get by using Algorithm~\ref{alg:adaptive_quantization} are typically obtained with $\Delta \leq 0.2$.  However, we explore a much larger range since \emph{a priori} we do not know the level of granularity needed to preserve most of the information in an embedding. Theoretically, there exist embeddings, such as highly sparse embeddings with nearly binary elements, which could maintain high accuracy at large quantization scales.  Hence, the optimal level of quantization for the bound is an empirical question that may vary across datasets.  However, we limit the upper range to 1, because if the largest weight is less than 2.7 (as we will examine below), then $\Delta=1$ already reduces the number of distinct weight values to 3, which is small enough that we are likely to see diminishing returns by considering even larger discretization scales.

\paragraph{Upper bound}

The upper bound, $M$, as defined in \S\ref{sec:pac_bound}, for the VICE mean embeddings was approximately $2.7$ for all datasets. Empirically, we found the maximum weights for \textsc{Things}, \textsc{Adjectives}, and \textsc{Food} to be $2.6285$, $2.2130$ and $2.4307$, respectively, across all random initializations.  Theoretically, this upper bound is not surprising due to a combination of two factors.  First, in most datasets, the cross-entropy error goes to infinity as weights become arbitrarily large, due to the increasing over-confidence of incorrect probability estimates.  Hence, the optimal cross-entropy is achieved with bounded weights. Second, almost any kind of regularization, including the $\ell_{1}$ regularization used in SPoSE and the spike-and-slab regularization used for VICE, further encourages the weights to shrink. The spike-and-slab prior, in particular, penalizes large weights much higher than small weights, which is a desirable property in gradient-based optimization.

\section{Dataset acquisition}
\label{app:dataset_details}
All datasets were collected by crowd-sourcing human responses to the triplet task described in \S\ref{method:triplet_task} on the Amazon Mechanical Turk platform, using workers located in the United States. For \textsc{Adjectives}, all words deemed offensive were removed from the list of adjectives being considered. There are no offensive images in either \textsc{Things} or \textsc{Food}. All workers provided informed consent, and were compensated financially for their time ($0.5~c$ per response, and additional $10~c$ per completed HIT).  Workers could participate in blocks of 20 triplet trials and could choose to work on as many such blocks as they liked. The online research was approved by the Office of Human Research Subject Protection at the National Institutes of Health and conducted in accordance with all relevant ethical regulations. Worker ages were not assessed, and no personally identifiable information was collected. 

\paragraph{\textsc{Things} and \textsc{Adjectives}}

A total of 5,526 workers participated in collecting the first dataset (5,301 after exclusion; 3,159 female, 2,092 male, 19 other, 31 no response). A total of 336 workers participated in collecting the second dataset (325 after exclusion; 156 female, 103 male, 66 not reported).
Workers were excluded if they exhibited overly fast responses in at least 5 sets of 20 trials (the speed cut-off was $25\%$ or more responses $< 800$ms and $50\%$ or more responses $< 1,100$ms) or if they carried out at least 200 trials and showed overly deterministic responses ($> 40\%$ of responses in one of the three odd-one-out positions; expected value, $33\%$).

\paragraph{\textsc{Food}}
A total of 554 subjects participated in collecting the dataset (487 after exclusion). Workers were excluded if they exhibited overly fast responses (reaction time of less than 500ms). They were also excluded if they failed on either of two catch trials in each HIT, where they saw images of ‘+’, ‘-‘, ’-’ or ‘=’, ‘+’, or ‘=’ instead of food pictures;  they were instructed on the slide to select the ‘+’. 

\section{Interpretability}
\label{app:object_dimensions}

Here, we display the top $6$ objects for each of the $45$ VICE dimensions \textsc{Things} \citep{hebart2019things}. Objects were sorted in descending order according to their absolute embedding value. As we have done for every other experiment, we used the \textit{pruned} median model to guarantee the extraction of a representative sample of embedding dimensions without being over-optimistic with respect to their interpretability (see \S\ref{sec:results_things} for how the median model was identified).

\begin{figure*}[ht!]
\centering
\begin{subfigure}{.45\textwidth}
    \centering
    \includegraphics[width=1.0\textwidth]{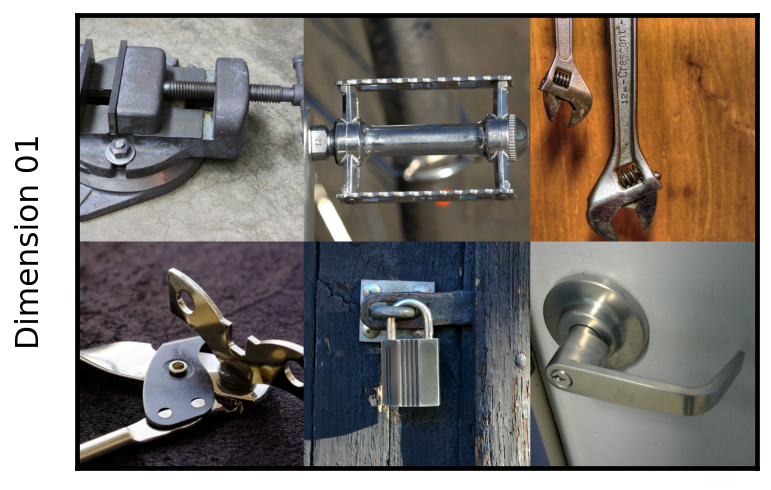}
\end{subfigure}%
\begin{subfigure}{.45\textwidth}
    \centering
    \includegraphics[width=1.0\textwidth]{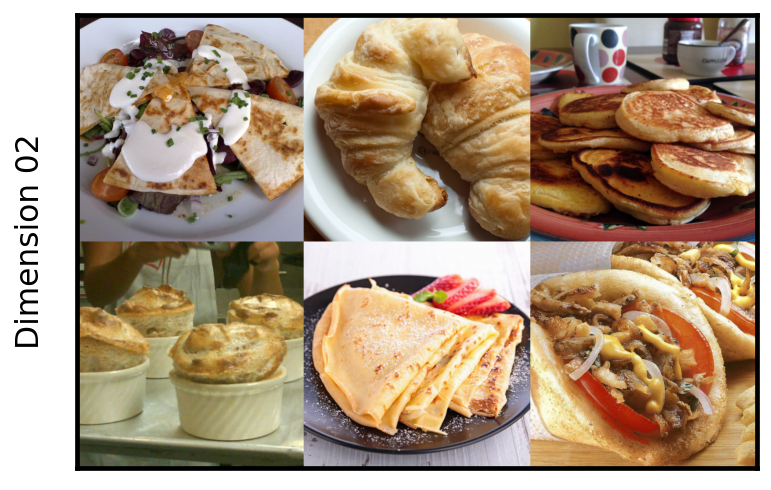}
\end{subfigure}
\begin{subfigure}{.45\textwidth}
    \centering
    \includegraphics[width=1.0\textwidth]{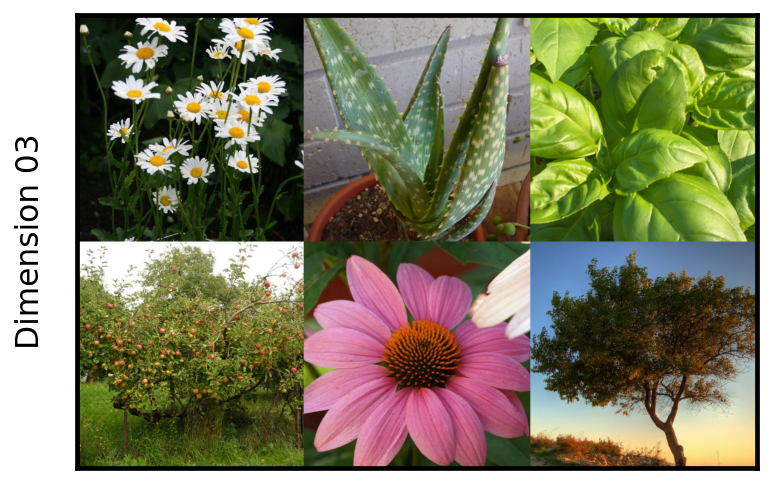}
\end{subfigure}%
\begin{subfigure}{.45\textwidth}
    \centering
    \includegraphics[width=1.0\textwidth]{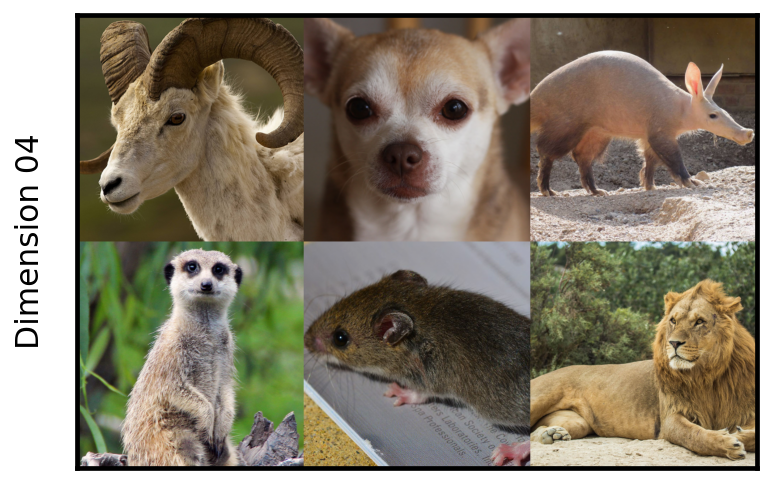}
\end{subfigure}
\begin{subfigure}{.45\textwidth}
    \centering
    \includegraphics[width=1.0\textwidth]{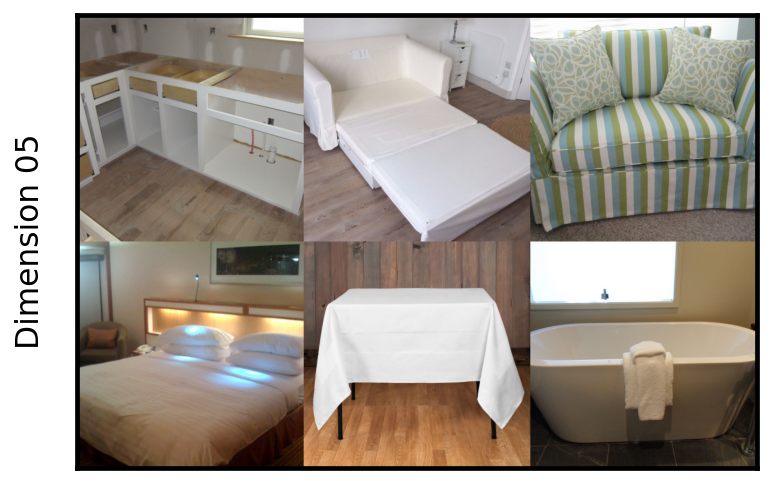}
\end{subfigure}%
\begin{subfigure}{.45\textwidth}
    \centering
    \includegraphics[width=1.0\textwidth]{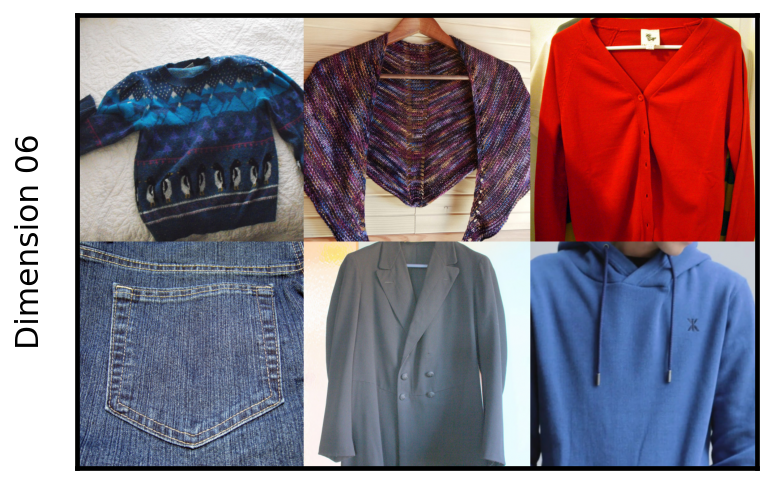}
\end{subfigure}
\begin{subfigure}{.45\textwidth}
    \centering
    \includegraphics[width=1.0\textwidth]{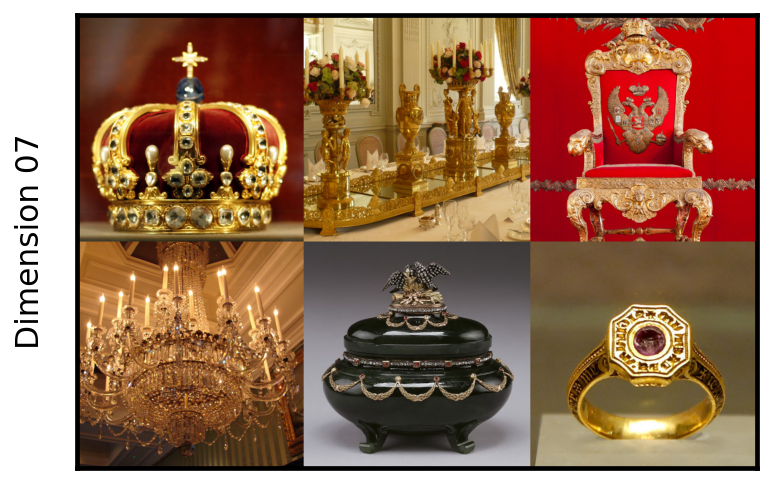}
\end{subfigure}%
\begin{subfigure}{.45\textwidth}
    \centering
    \includegraphics[width=1.0\textwidth]{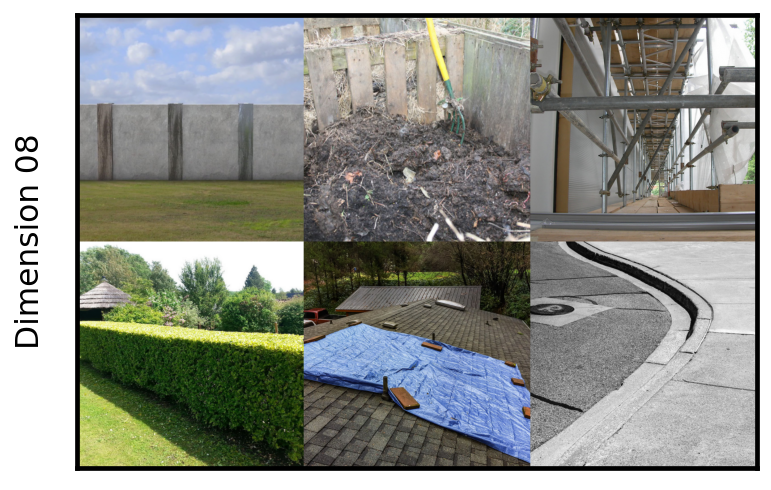}
\end{subfigure}
\begin{subfigure}{.45\textwidth}
    \centering
    \includegraphics[width=1.0\textwidth]{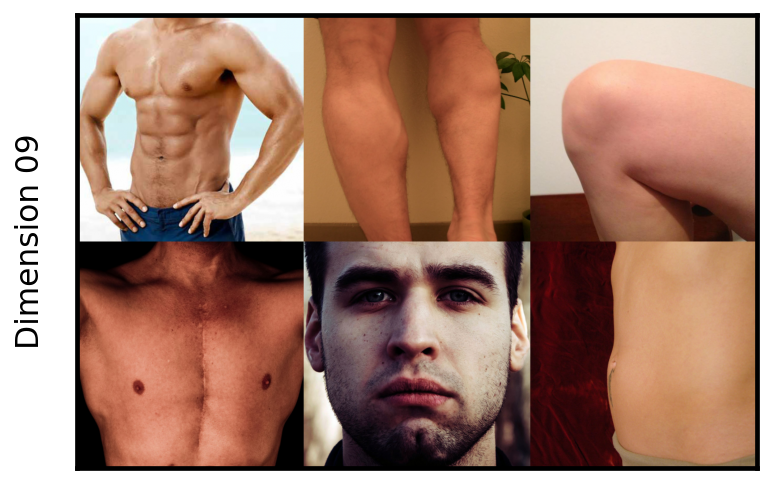}
\end{subfigure}%
\begin{subfigure}{.45\textwidth}
    \centering
    \includegraphics[width=1.0\textwidth]{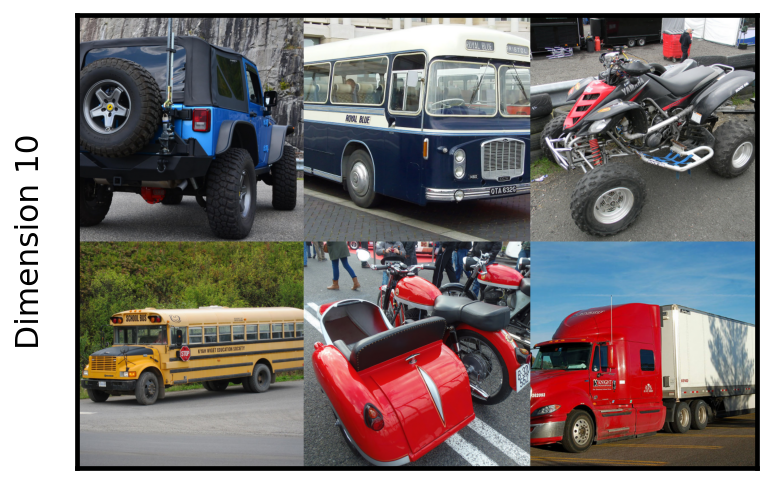}
\end{subfigure}
\caption{\textsc{Things} Dimensions 1-10.}
\end{figure*}

\begin{figure*}[t]
\centering
\begin{subfigure}{.45\textwidth}
    \centering
    \includegraphics[width=1.0\textwidth]{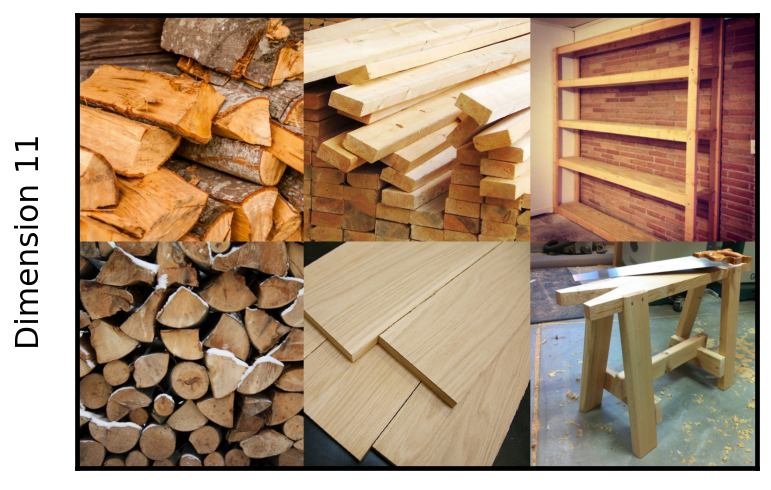}
\end{subfigure}%
\begin{subfigure}{.45\textwidth}
    \centering
    \includegraphics[width=1.0\textwidth]{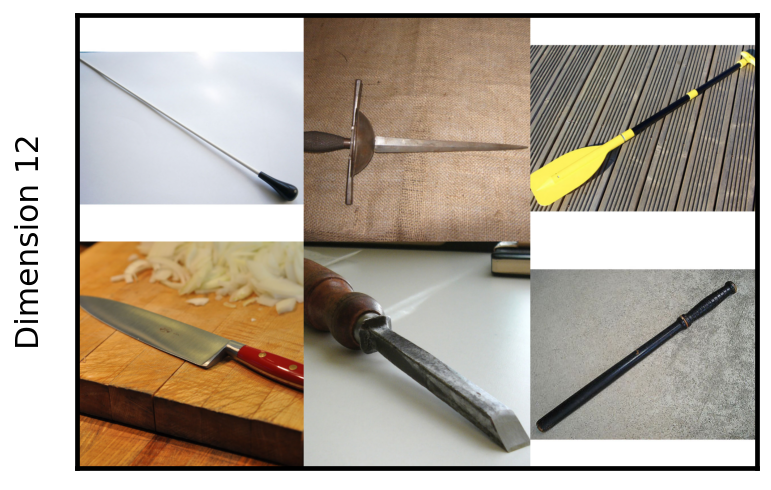}
\end{subfigure}
\begin{subfigure}{.45\textwidth}
    \centering
    \includegraphics[width=1.0\textwidth]{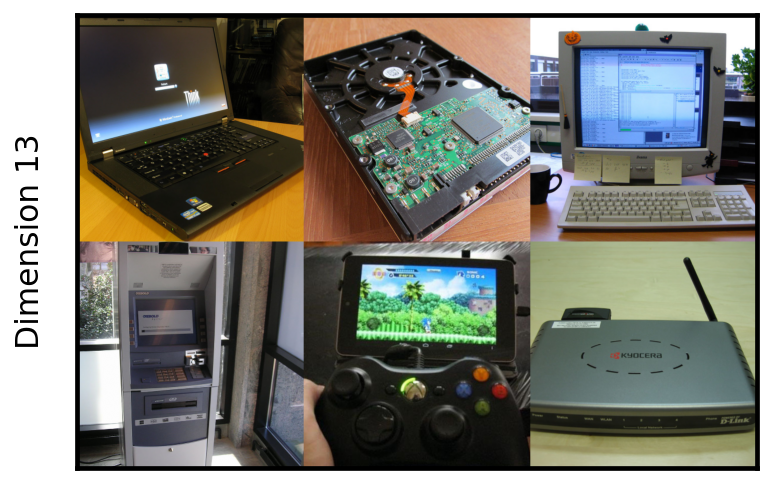}
\end{subfigure}%
\begin{subfigure}{.45\textwidth}
    \centering
    \includegraphics[width=1.0\textwidth]{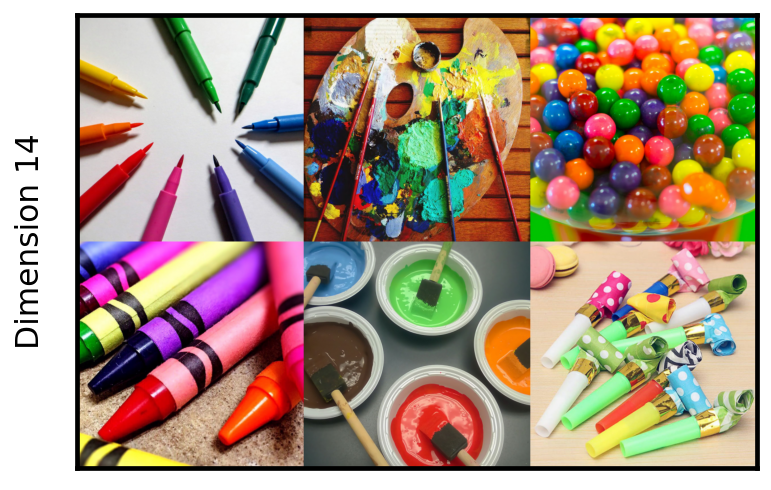}
\end{subfigure}
\begin{subfigure}{.45\textwidth}
    \centering
    \includegraphics[width=1.0\textwidth]{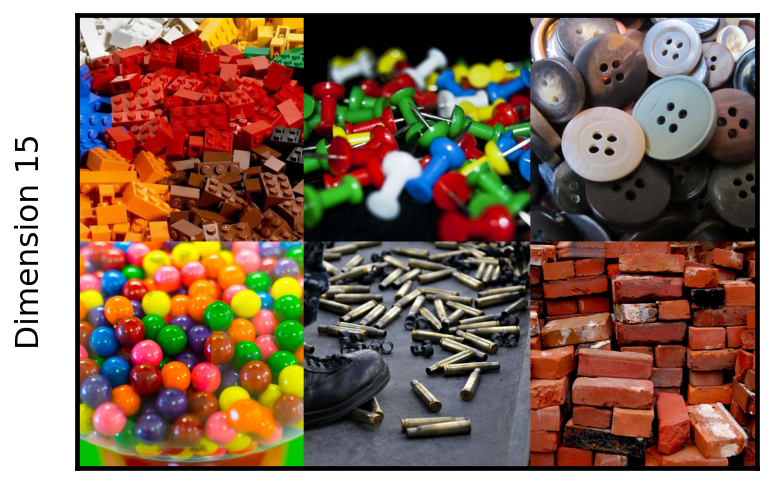}
\end{subfigure}%
\begin{subfigure}{.45\textwidth}
    \centering
    \includegraphics[width=1.0\textwidth]{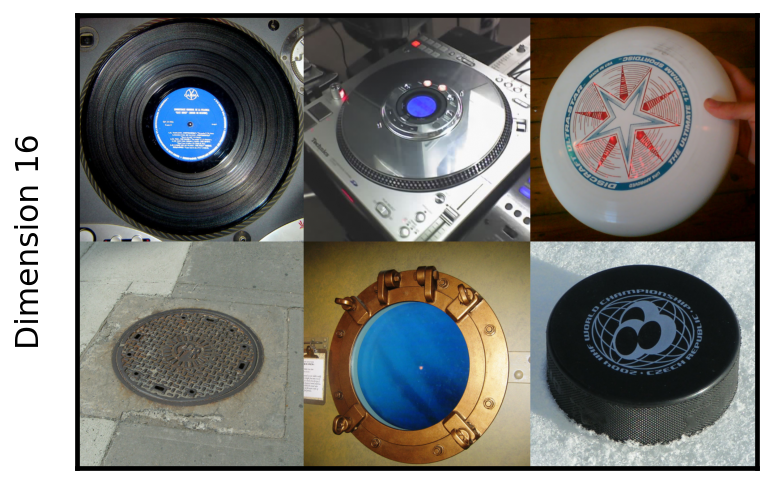}
\end{subfigure}
\begin{subfigure}{.45\textwidth}
    \centering
    \includegraphics[width=1.0\textwidth]{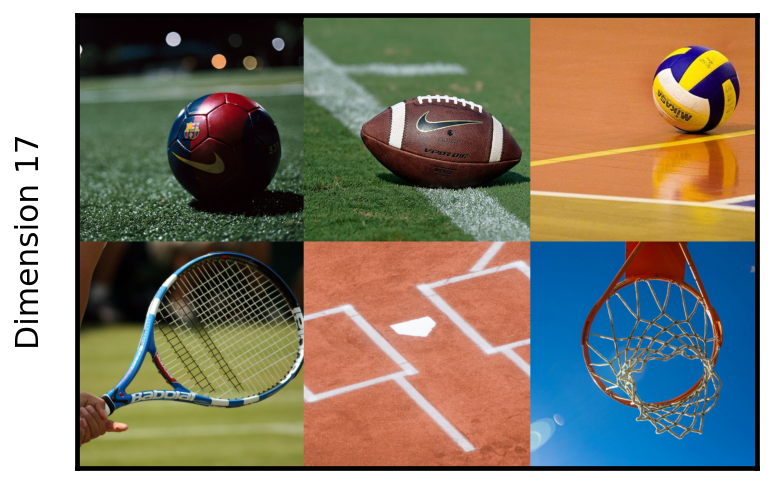}
\end{subfigure}%
\begin{subfigure}{.45\textwidth}
    \centering
    \includegraphics[width=1.0\textwidth]{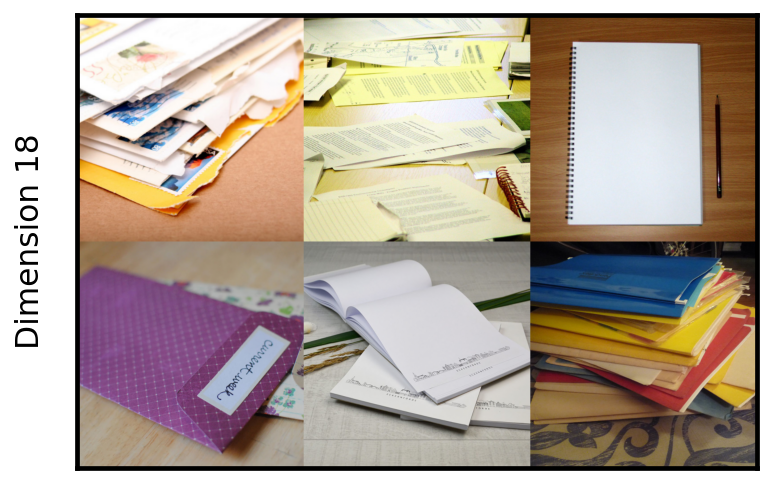}
\end{subfigure}
\begin{subfigure}{.45\textwidth}
    \centering
    \includegraphics[width=1.0\textwidth]{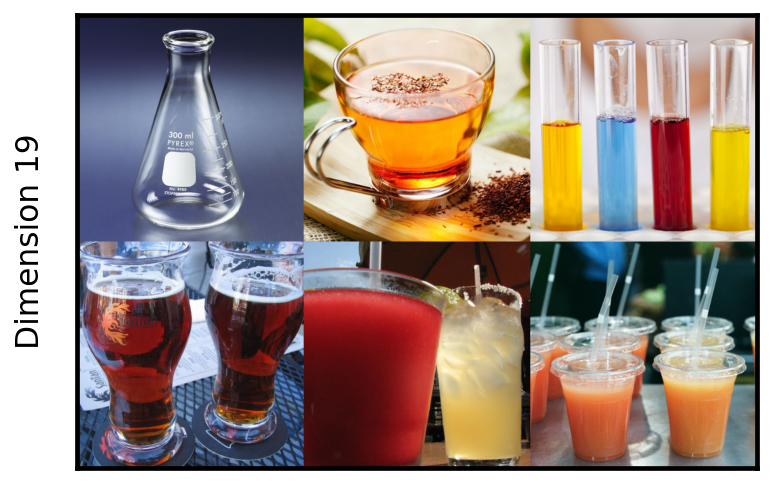}
\end{subfigure}%
\begin{subfigure}{.45\textwidth}
    \centering
    \includegraphics[width=1.0\textwidth]{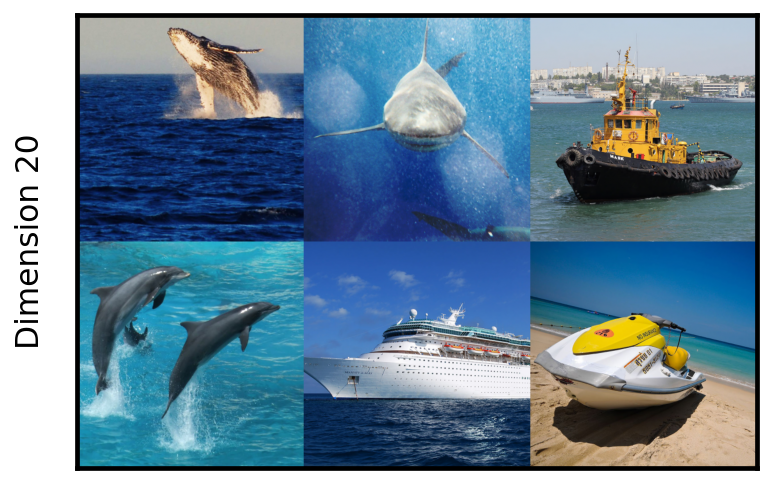}
\end{subfigure}
\caption{\textsc{Things} Dimensions 11-20.}
\end{figure*}

\begin{figure*}[t]
\centering
\begin{subfigure}{.45\textwidth}
    \centering
    \includegraphics[width=1.0\textwidth]{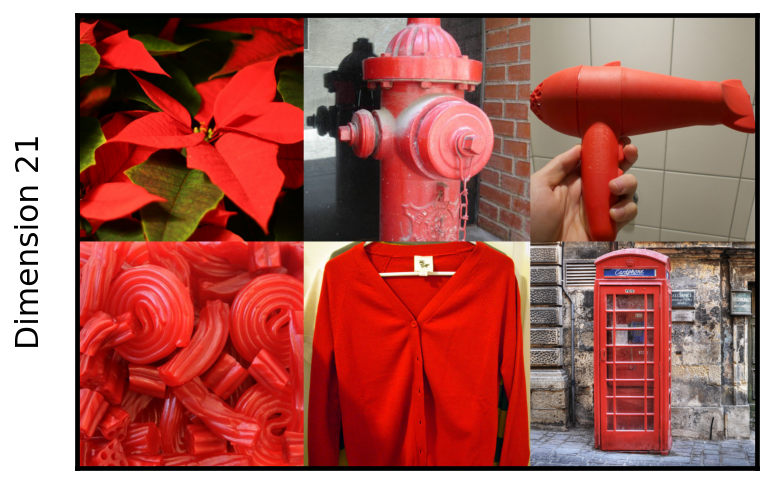}
\end{subfigure}%
\begin{subfigure}{.45\textwidth}
    \centering
    \includegraphics[width=1.0\textwidth]{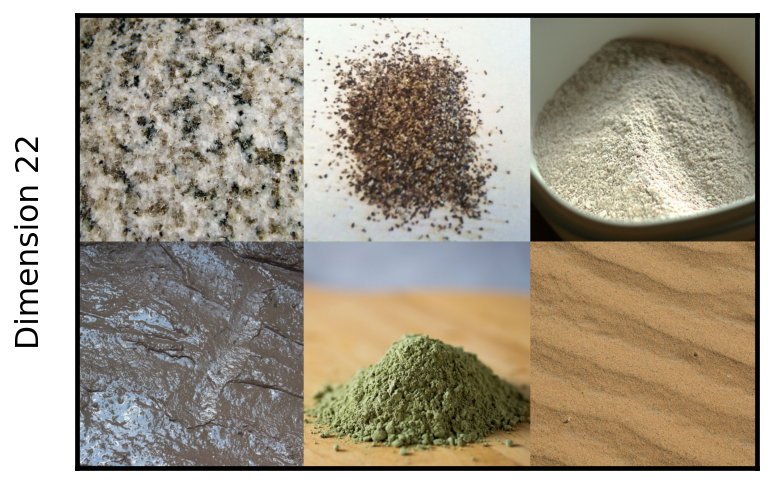}
\end{subfigure}
\begin{subfigure}{.45\textwidth}
    \centering
    \includegraphics[width=1.0\textwidth]{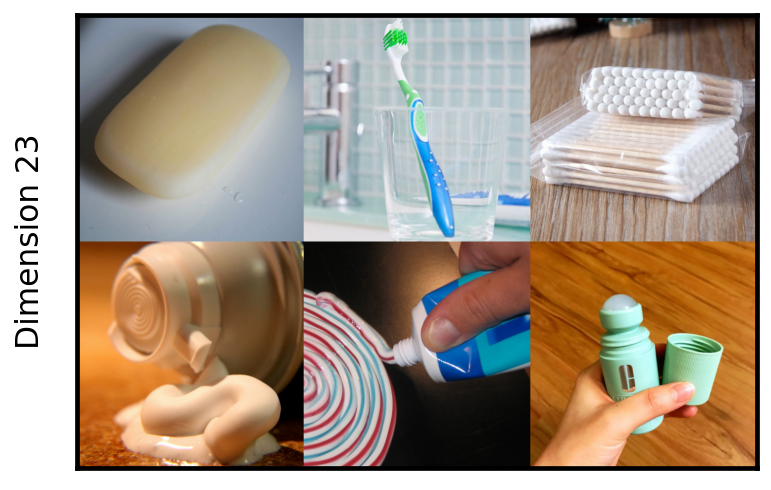}
\end{subfigure}%
\begin{subfigure}{.45\textwidth}
    \centering
    \includegraphics[width=1.0\textwidth]{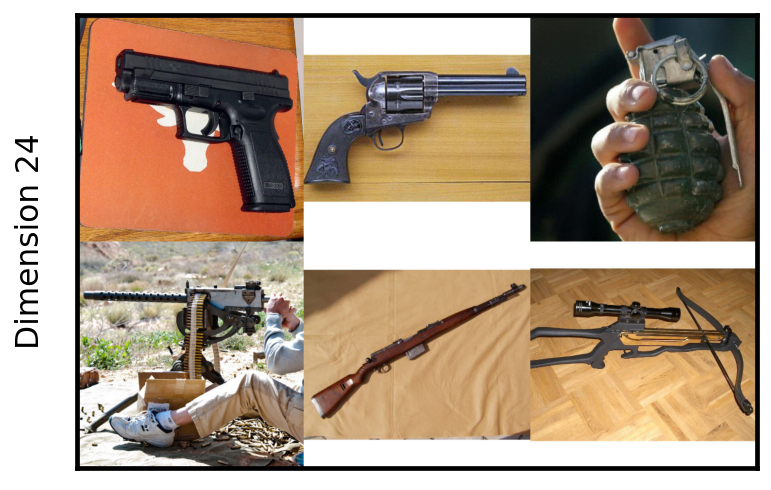}
\end{subfigure}
\begin{subfigure}{.45\textwidth}
    \centering
    \includegraphics[width=1.0\textwidth]{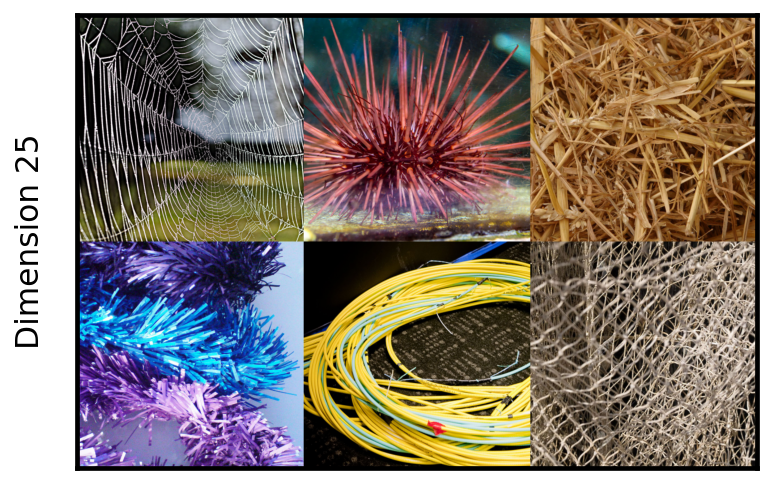}
\end{subfigure}%
\begin{subfigure}{.45\textwidth}
    \centering
    \includegraphics[width=1.0\textwidth]{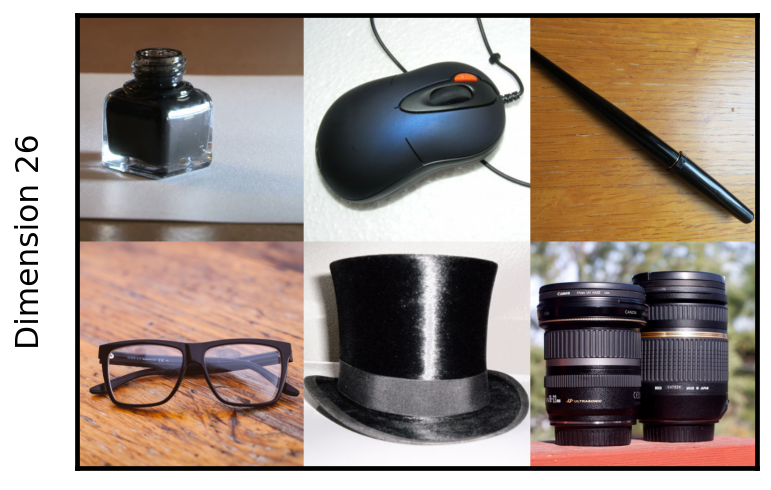}
\end{subfigure}
\begin{subfigure}{.45\textwidth}
    \centering
    \includegraphics[width=1.0\textwidth]{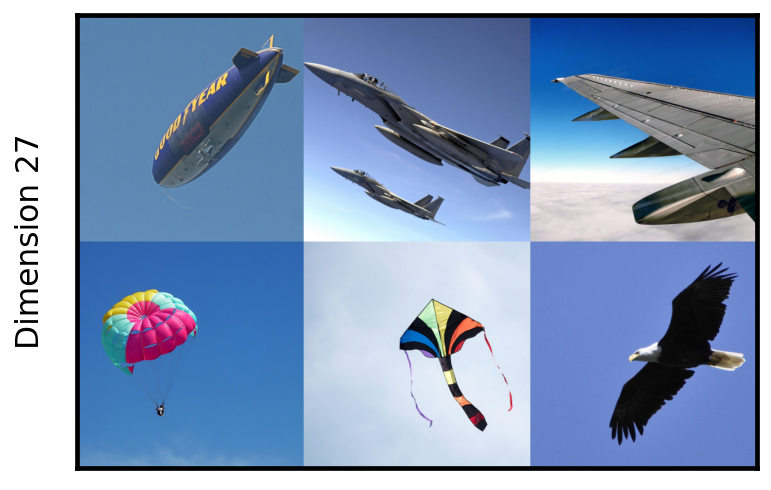}
\end{subfigure}%
\begin{subfigure}{.45\textwidth}
    \centering
    \includegraphics[width=1.0\textwidth]{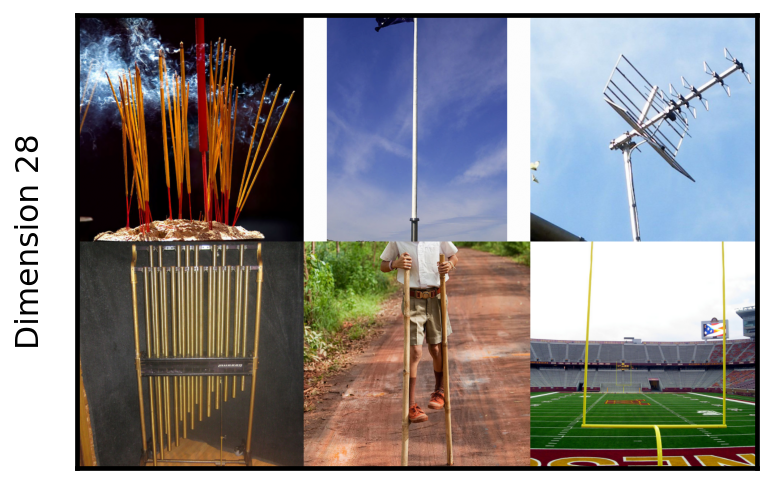}
\end{subfigure}
\begin{subfigure}{.45\textwidth}
    \centering
    \includegraphics[width=1.0\textwidth]{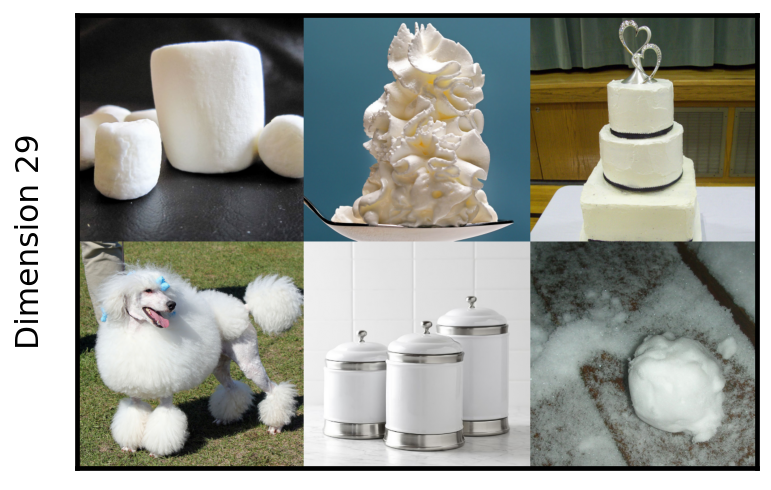}
\end{subfigure}%
\begin{subfigure}{.45\textwidth}
    \centering
    \includegraphics[width=1.0\textwidth]{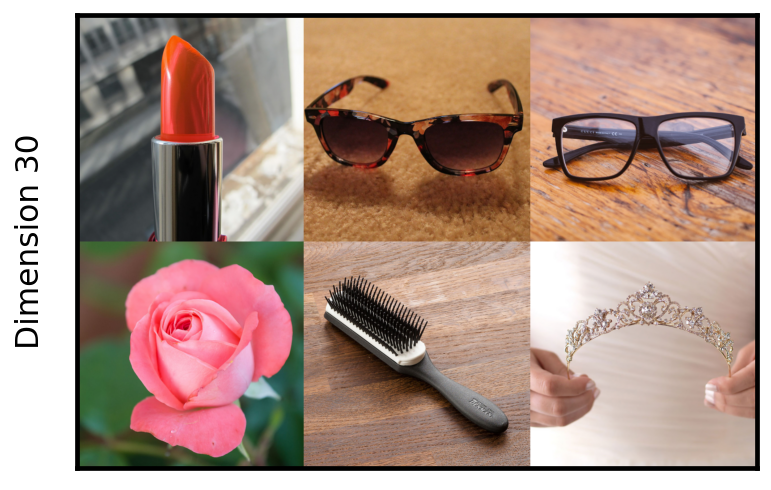}
\end{subfigure}
\caption{\textsc{Things} Dimensions 21-30.}
\end{figure*}

\begin{figure*}[t]
\centering
\begin{subfigure}{.45\textwidth}
    \centering
    \includegraphics[width=1.0\textwidth]{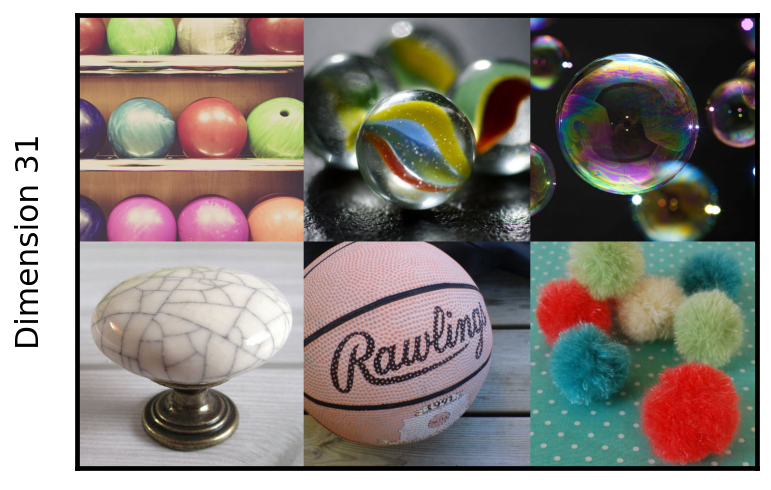}
\end{subfigure}%
\begin{subfigure}{.45\textwidth}
    \centering
    \includegraphics[width=1.0\textwidth]{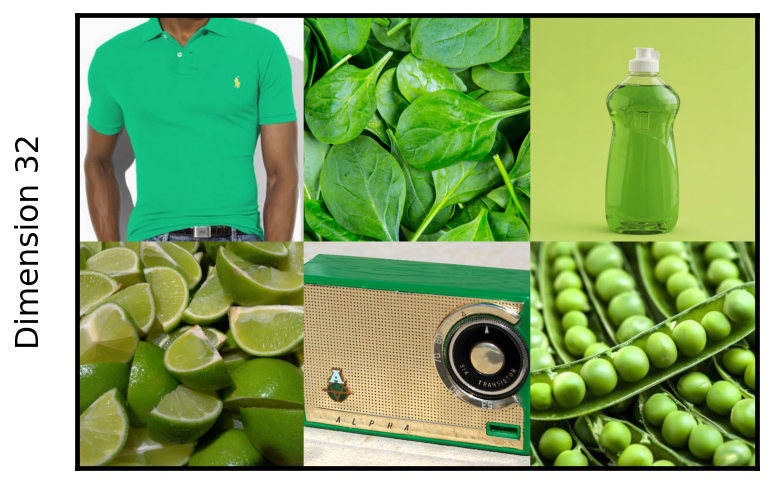}
\end{subfigure}
\begin{subfigure}{.45\textwidth}
    \centering
    \includegraphics[width=1.0\textwidth]{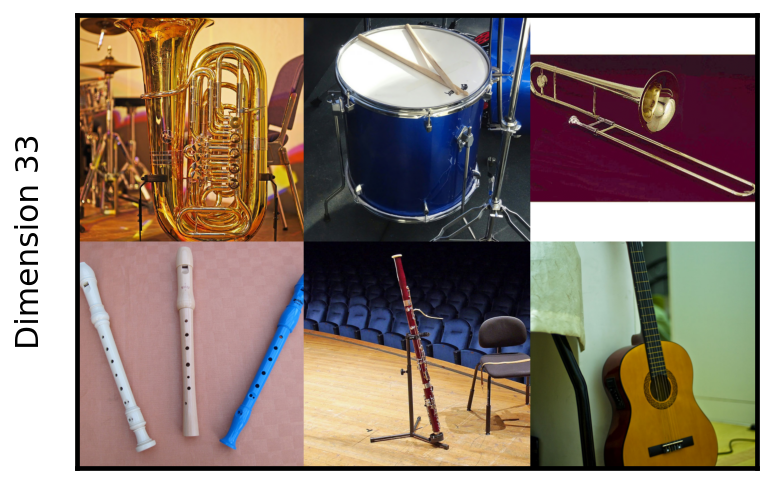}
\end{subfigure}%
\begin{subfigure}{.45\textwidth}
    \centering
    \includegraphics[width=1.0\textwidth]{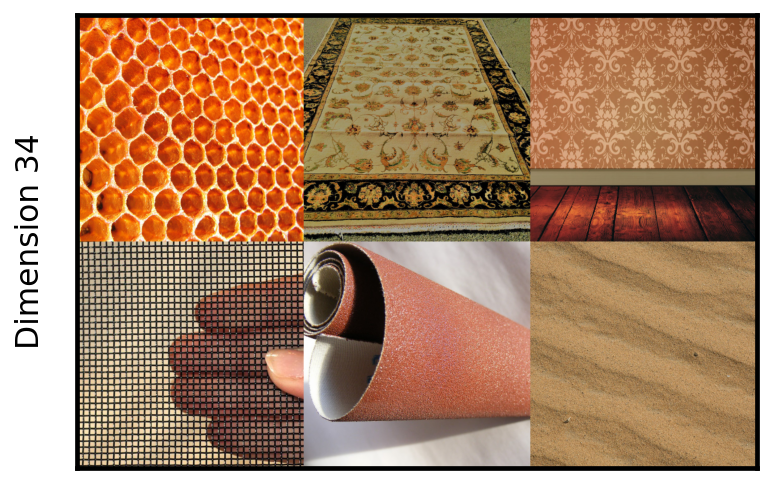}
\end{subfigure}
\begin{subfigure}{.45\textwidth}
    \centering
    \includegraphics[width=1.0\textwidth]{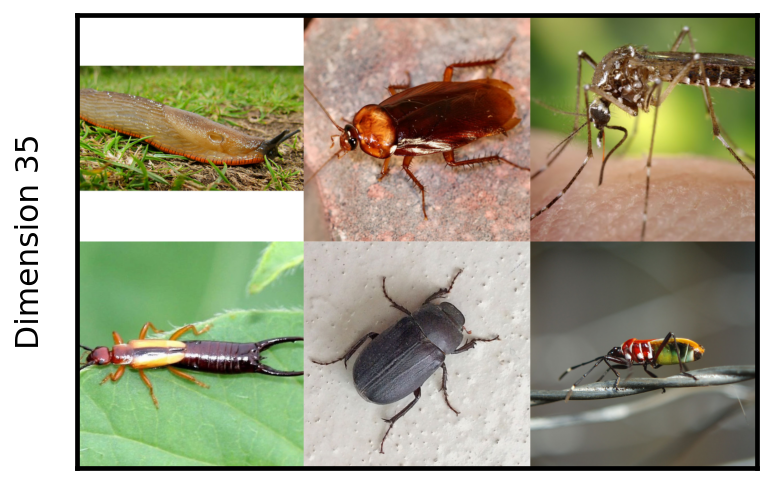}
\end{subfigure}%
\begin{subfigure}{.45\textwidth}
    \centering
    \includegraphics[width=1.0\textwidth]{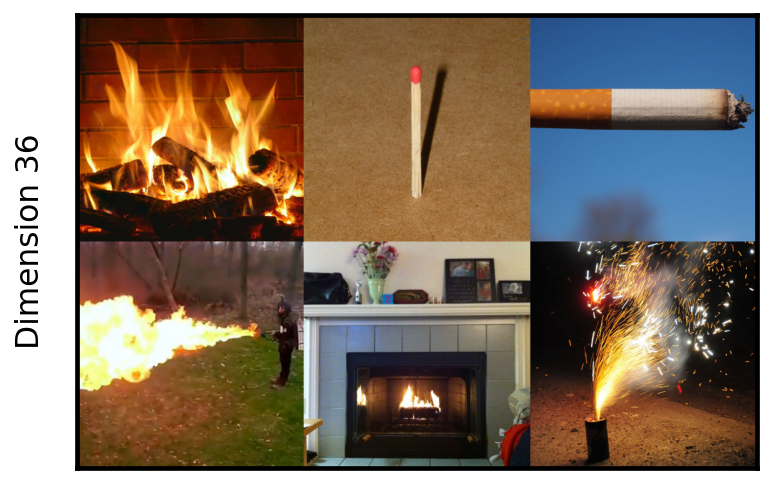}
\end{subfigure}
\begin{subfigure}{.45\textwidth}
    \centering
    \includegraphics[width=1.0\textwidth]{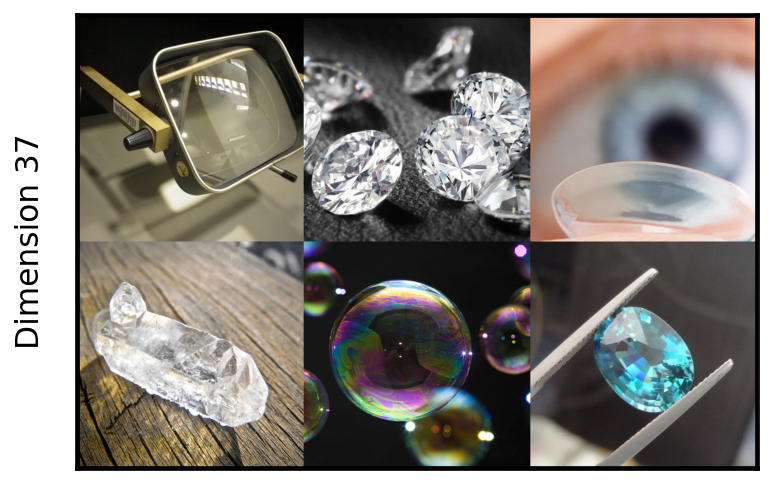}
\end{subfigure}%
\begin{subfigure}{.45\textwidth}
    \centering
    \includegraphics[width=1.0\textwidth]{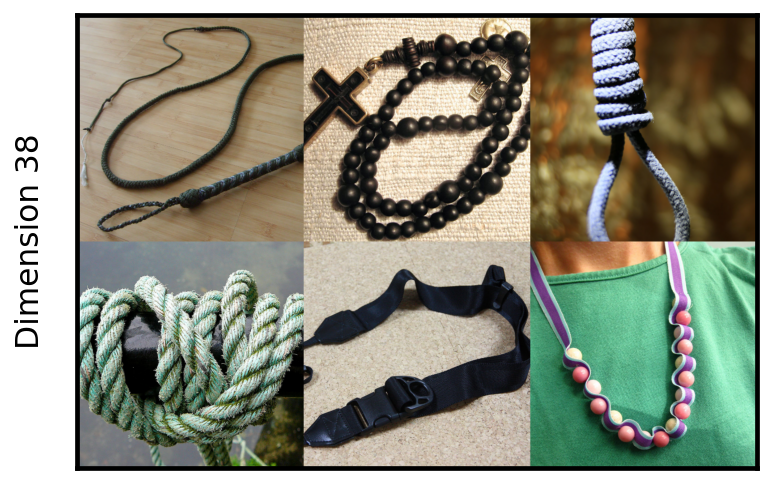}
\end{subfigure}
\begin{subfigure}{.45\textwidth}
    \centering
    \includegraphics[width=1.0\textwidth]{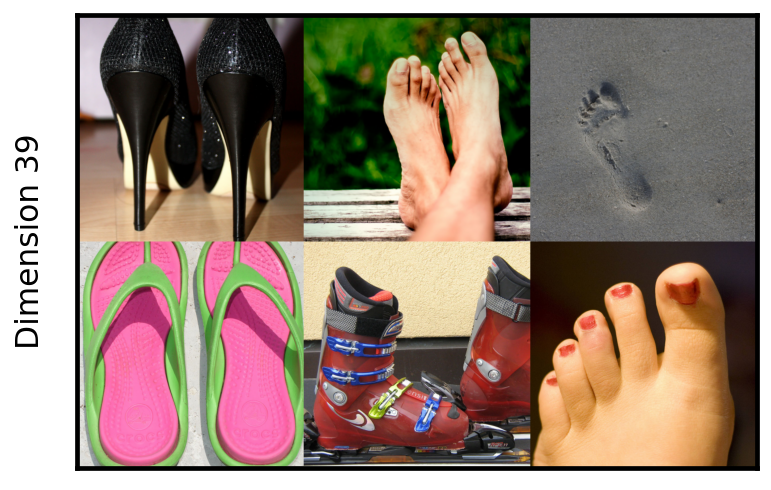}
\end{subfigure}%
\begin{subfigure}{.45\textwidth}
    \centering
    \includegraphics[width=1.0\textwidth]{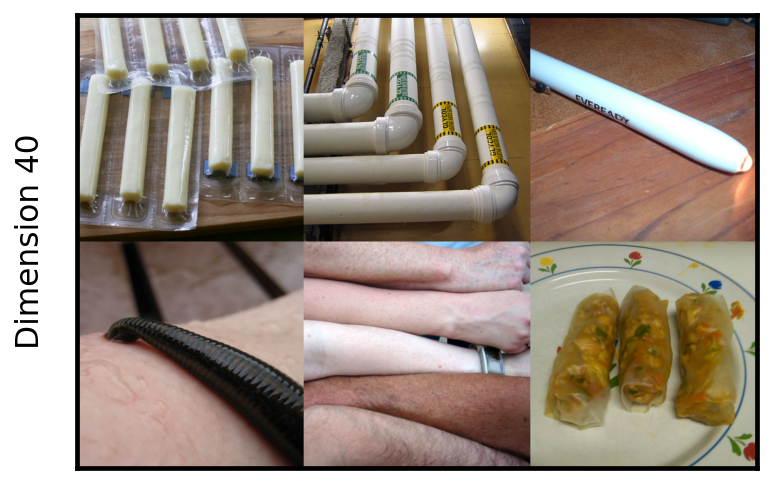}
\end{subfigure}
\caption{\textsc{Things} Dimensions 31-40.}
\end{figure*}

\begin{figure*}[t]
\centering
\begin{subfigure}{.5\textwidth}
    \centering
    \includegraphics[width=1.0\textwidth]{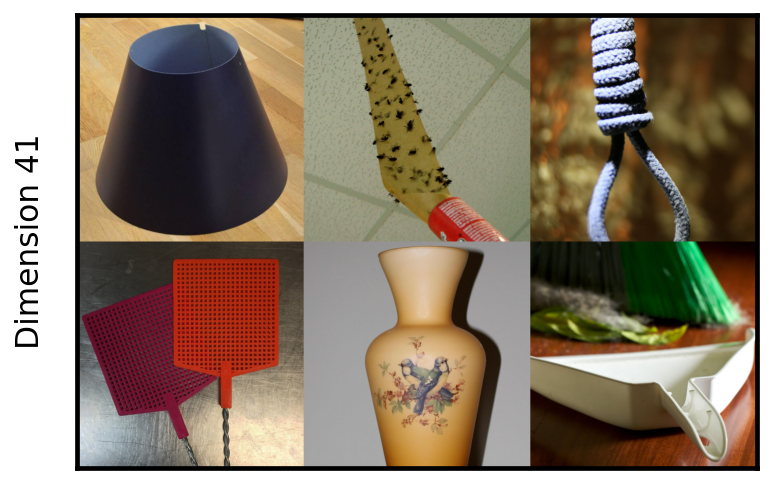}
\end{subfigure}%
\begin{subfigure}{.5\textwidth}
    \centering
    \includegraphics[width=1.0\textwidth]{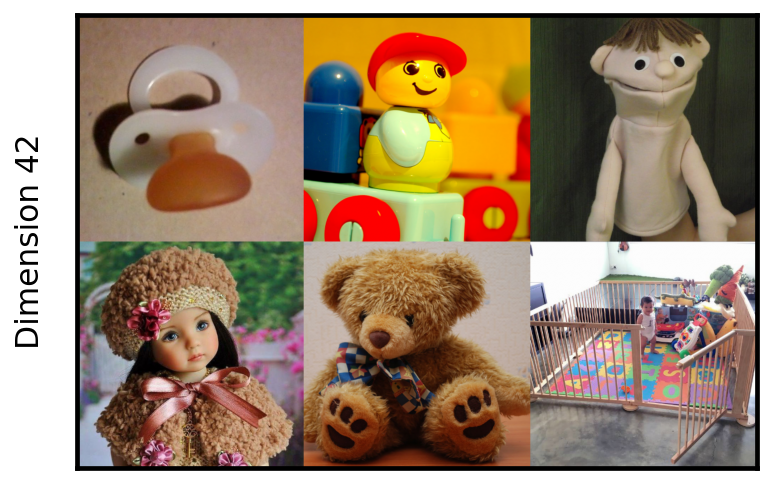}
\end{subfigure}
\begin{subfigure}{.5\textwidth}
    \centering
    \includegraphics[width=1.0\textwidth]{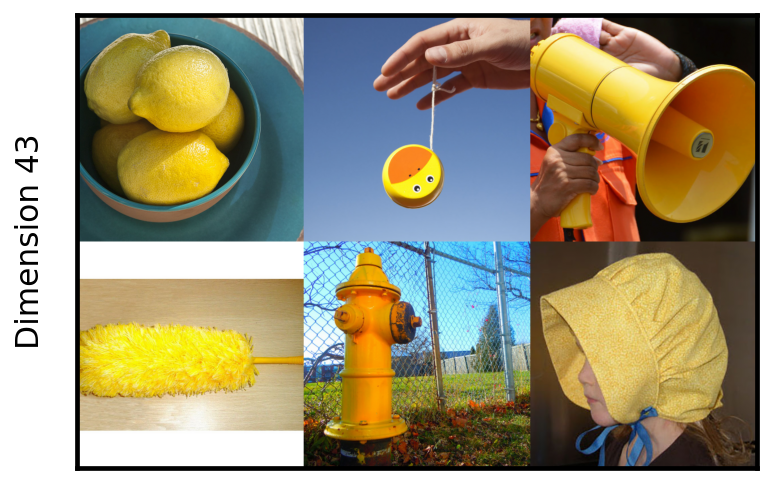}
\end{subfigure}%
\begin{subfigure}{.5\textwidth}
    \centering
    \includegraphics[width=1.0\textwidth]{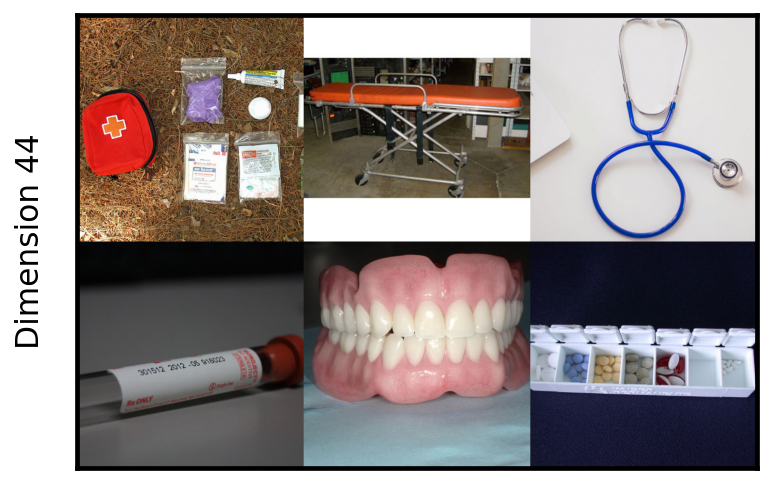}
\end{subfigure}
\begin{subfigure}{.5\textwidth}
    \centering
    \includegraphics[width=1.0\textwidth]{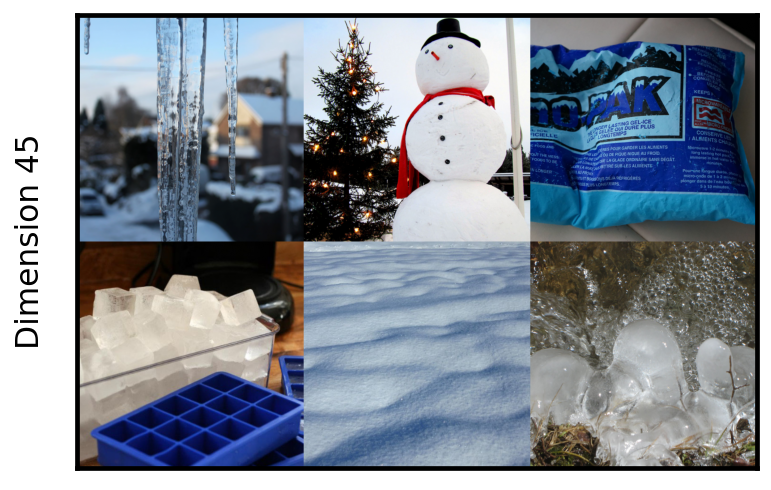}
\end{subfigure}
\caption{\textsc{Things} Dimensions 41-45.}
\end{figure*}
\newpage


\ifbool{true}{}{
\begin{figure*}[t]
\centering
\begin{subfigure}{.5\textwidth}
    \centering
    \includegraphics[width=1.0\textwidth]{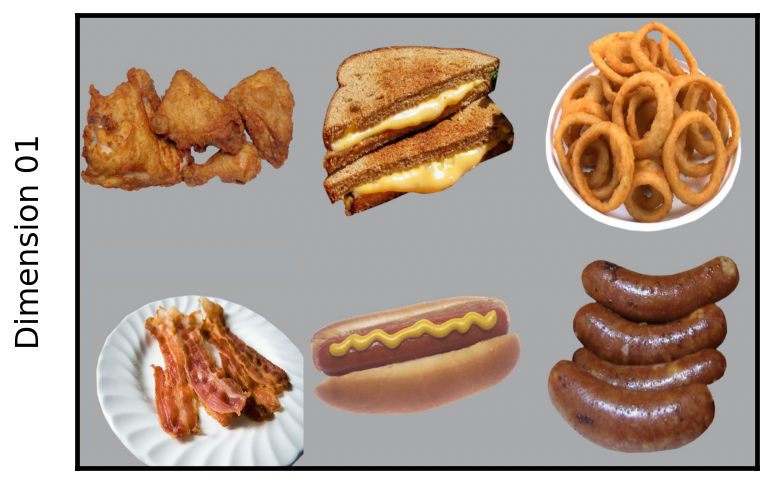}
\end{subfigure}%
\begin{subfigure}{.5\textwidth}
    \centering
    \includegraphics[width=1.0\textwidth]{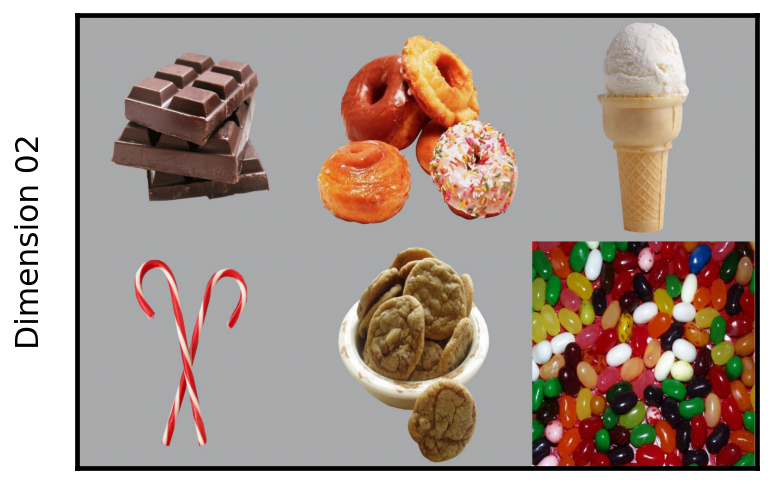}
\end{subfigure}
\begin{subfigure}{.5\textwidth}
    \centering
    \includegraphics[width=1.0\textwidth]{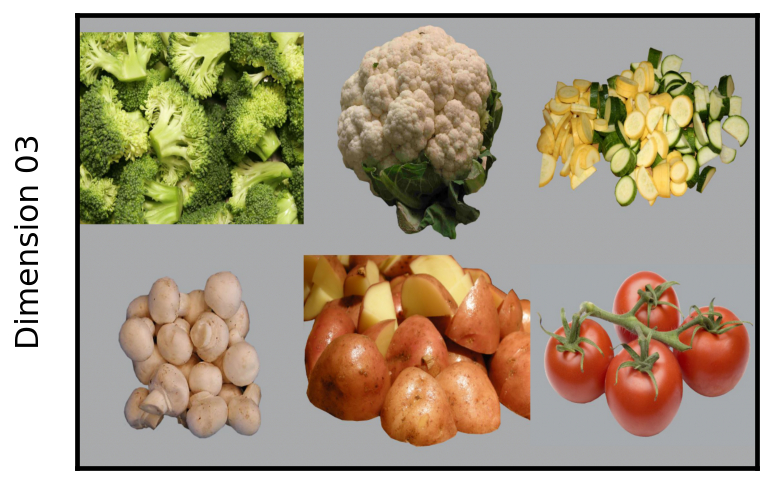}
\end{subfigure}%
\begin{subfigure}{.5\textwidth}
    \centering
    \includegraphics[width=1.0\textwidth]{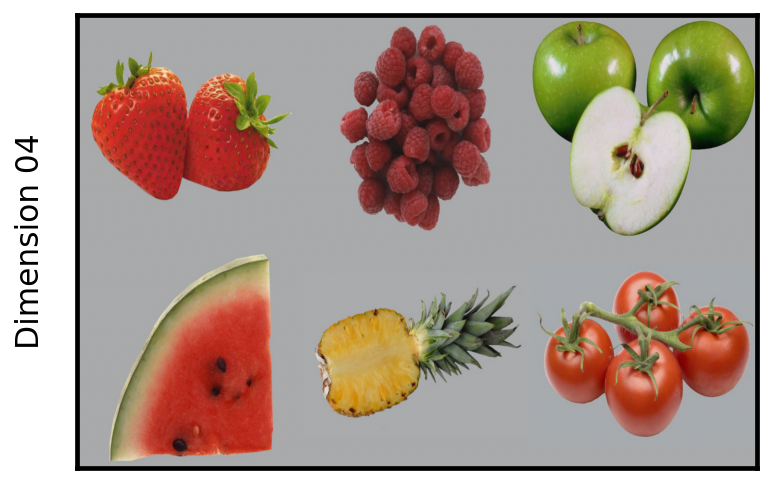}
\end{subfigure}
\begin{subfigure}{.5\textwidth}
    \centering
    \includegraphics[width=1.0\textwidth]{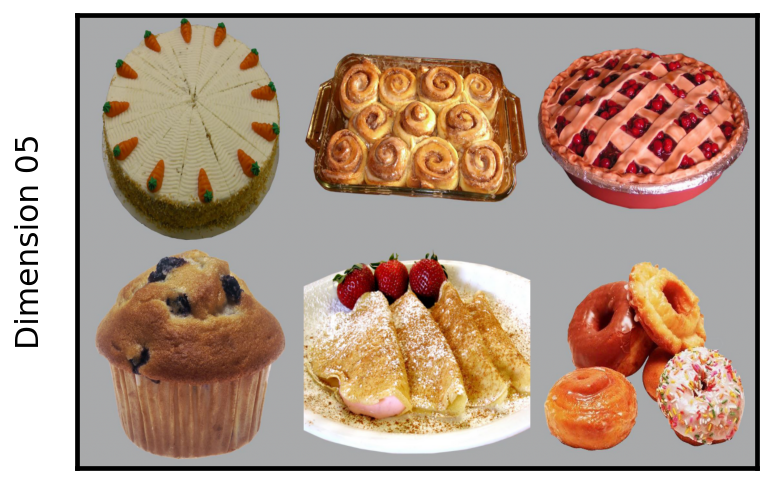}
\end{subfigure}
\caption{\textsc{Food} Dimensions 1-5.}
\end{figure*}
\newpage
}


\ifbool{true}{}{
\begin{figure*}[t]
\centering
\begin{subfigure}{.5\textwidth}
    \centering
    \includegraphics[width=1.0\textwidth]{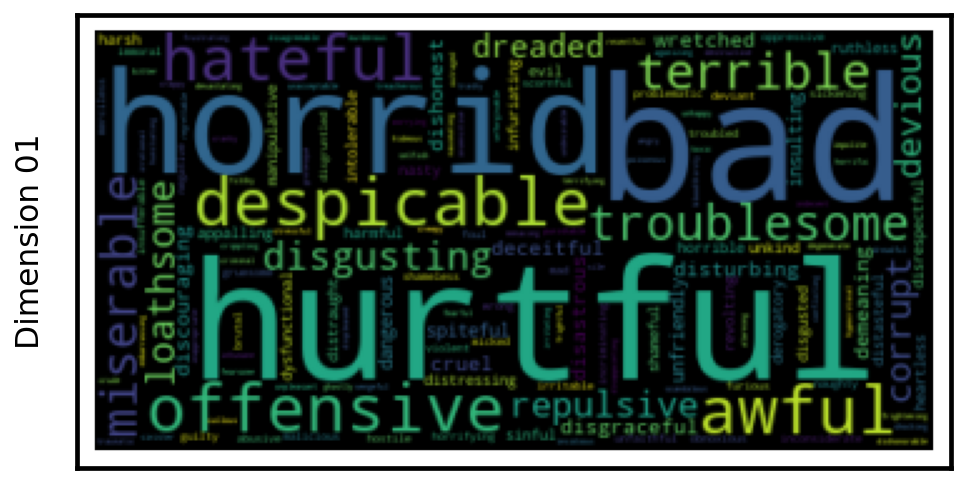}
\end{subfigure}%
\begin{subfigure}{.5\textwidth}
    \centering
    \includegraphics[width=1.0\textwidth]{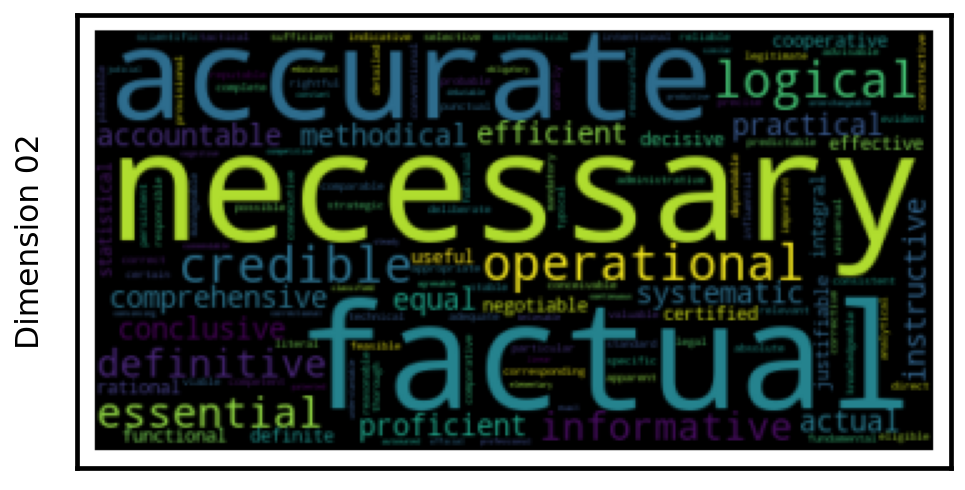}
\end{subfigure}
\begin{subfigure}{.5\textwidth}
    \centering
    \includegraphics[width=1.0\textwidth]{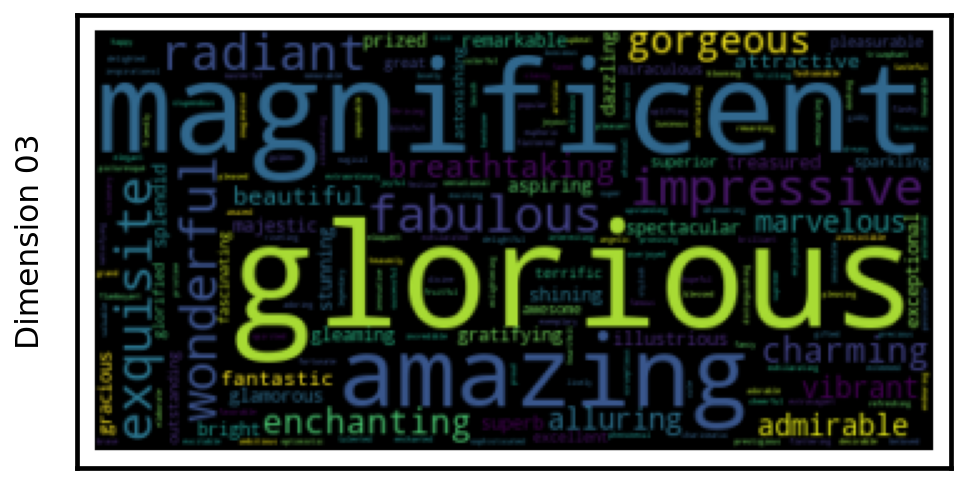}
\end{subfigure}%
\begin{subfigure}{.5\textwidth}
    \centering
    \includegraphics[width=1.0\textwidth]{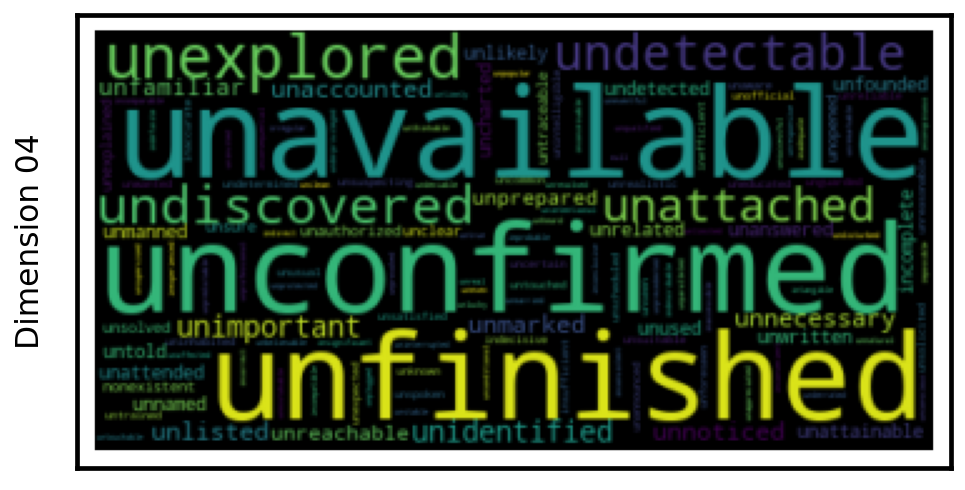}
\end{subfigure}
\begin{subfigure}{.5\textwidth}
    \centering
    \includegraphics[width=1.0\textwidth]{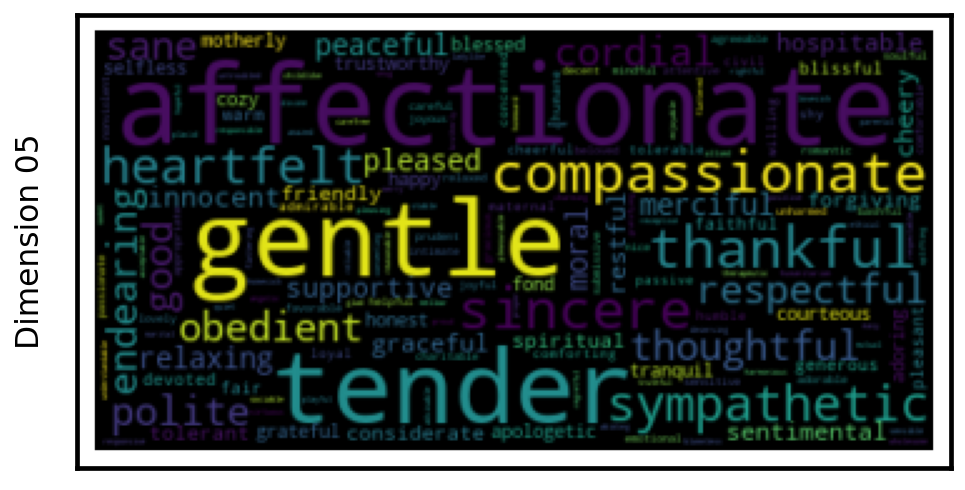}
\end{subfigure}%
\begin{subfigure}{.5\textwidth}
    \centering
    \includegraphics[width=1.0\textwidth]{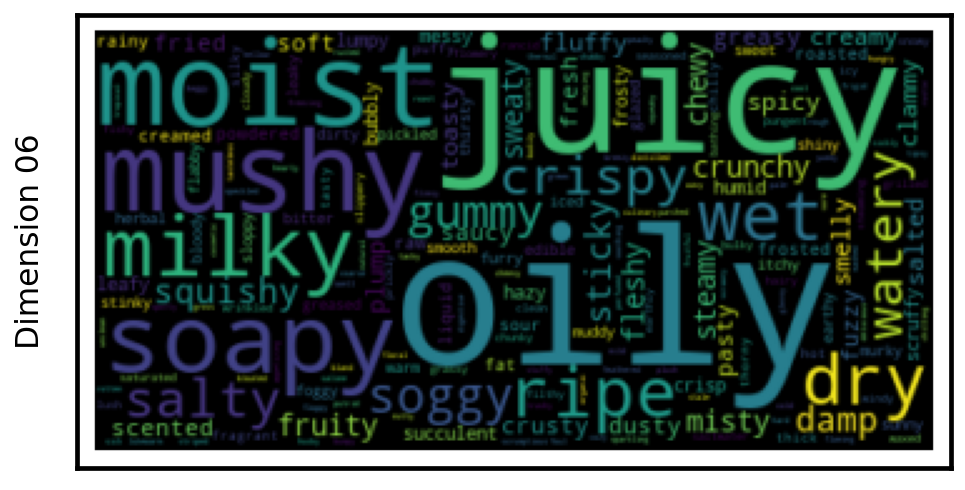}
\end{subfigure}
\begin{subfigure}{.5\textwidth}
    \centering
    \includegraphics[width=1.0\textwidth]{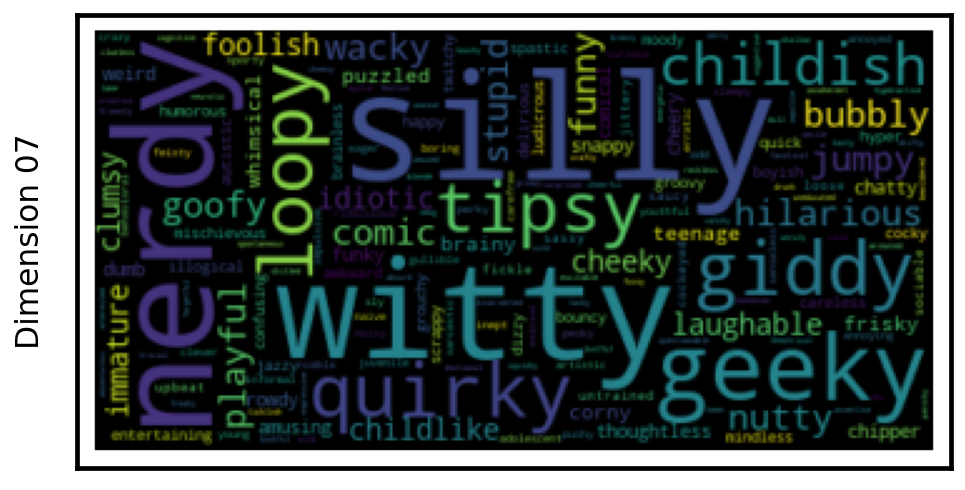}
\end{subfigure}%
\begin{subfigure}{.5\textwidth}
    \centering
    \includegraphics[width=1.0\textwidth]{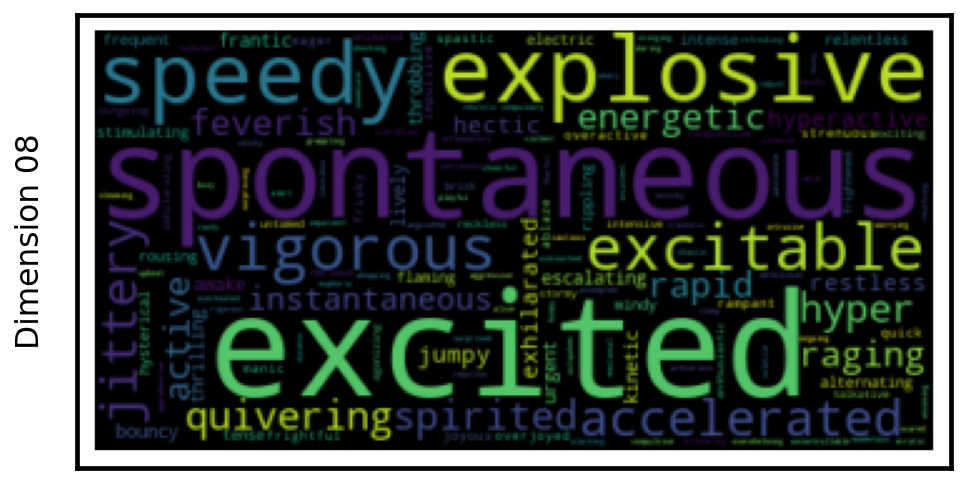}
\end{subfigure}
\begin{subfigure}{.5\textwidth}
    \centering
    \includegraphics[width=1.0\textwidth]{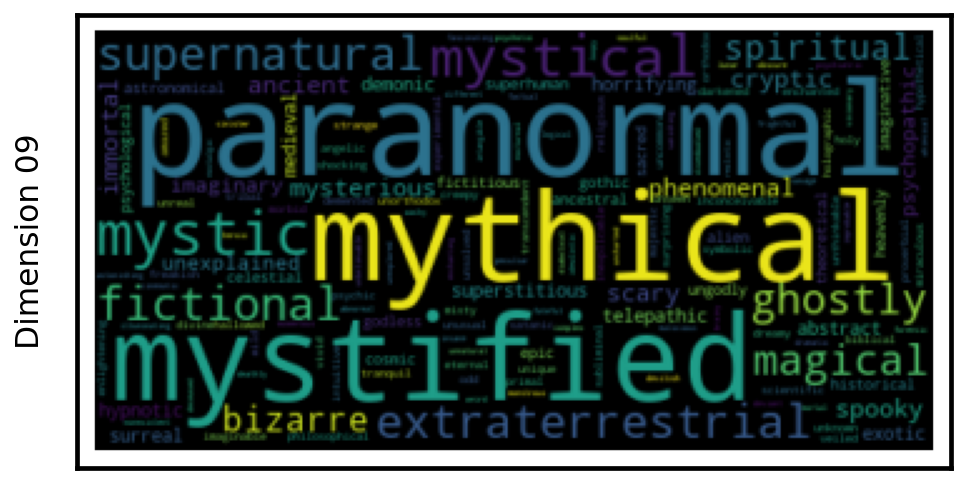}
\end{subfigure}%
\begin{subfigure}{.5\textwidth}
    \centering
    \includegraphics[width=1.0\textwidth]{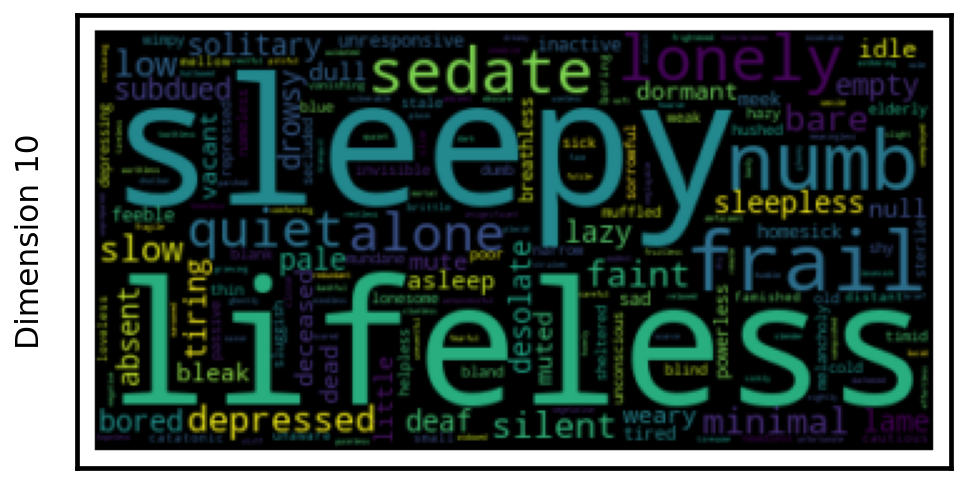}
\end{subfigure}
\caption{\textsc{Adjectives} Dimensions 1-10.}
\end{figure*}

\begin{figure*}[t]
\centering
\begin{subfigure}{.5\textwidth}
    \centering
    \includegraphics[width=1.0\textwidth]{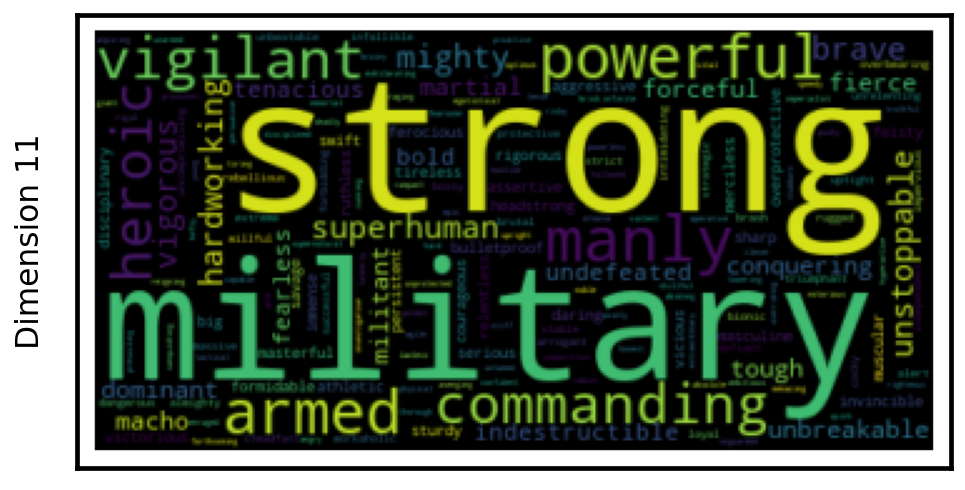}
\end{subfigure}%
\begin{subfigure}{.5\textwidth}
    \centering
    \includegraphics[width=1.0\textwidth]{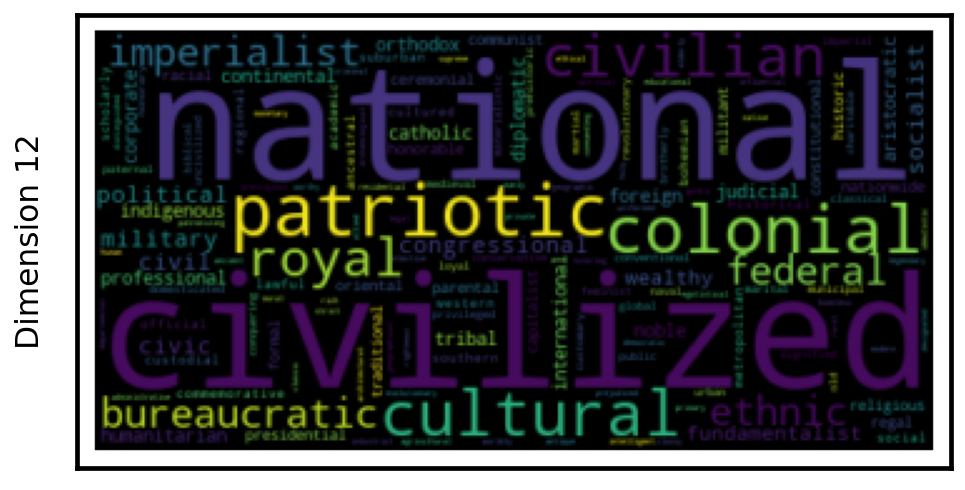}
\end{subfigure}
\begin{subfigure}{.5\textwidth}
    \centering
    \includegraphics[width=1.0\textwidth]{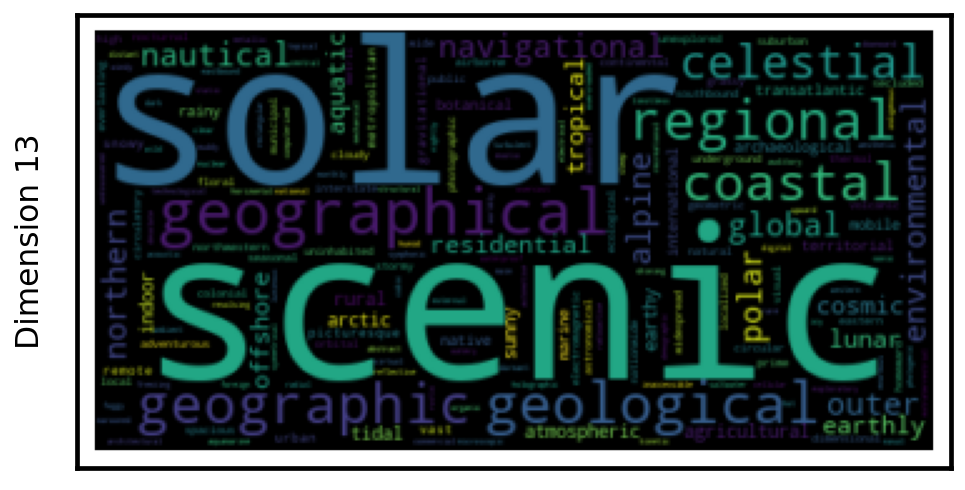}
\end{subfigure}%
\begin{subfigure}{.5\textwidth}
    \centering
    \includegraphics[width=1.0\textwidth]{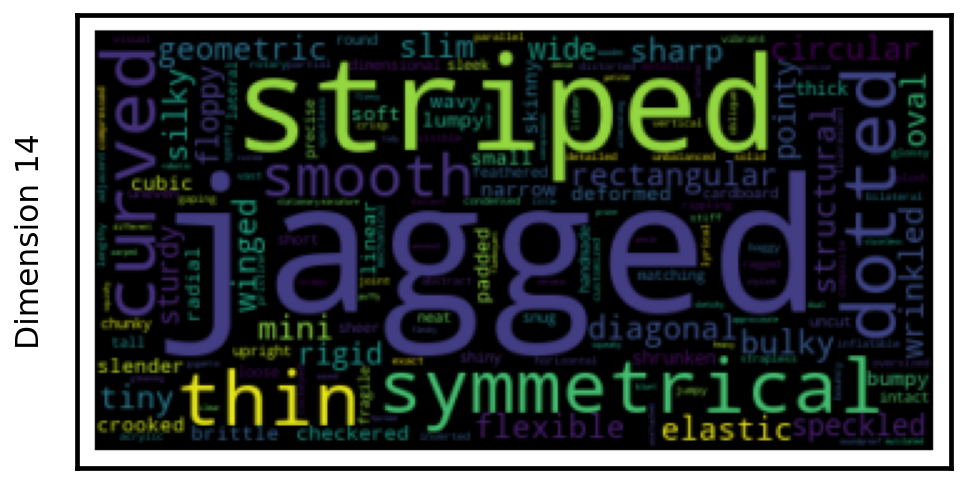}
\end{subfigure}
\begin{subfigure}{.5\textwidth}
    \centering
    \includegraphics[width=1.0\textwidth]{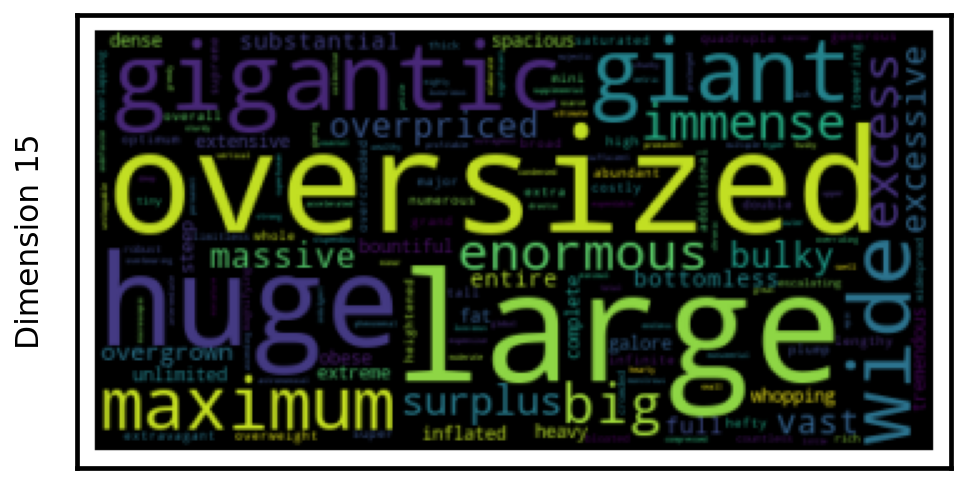}
\end{subfigure}%
\begin{subfigure}{.5\textwidth}
    \centering
    \includegraphics[width=1.0\textwidth]{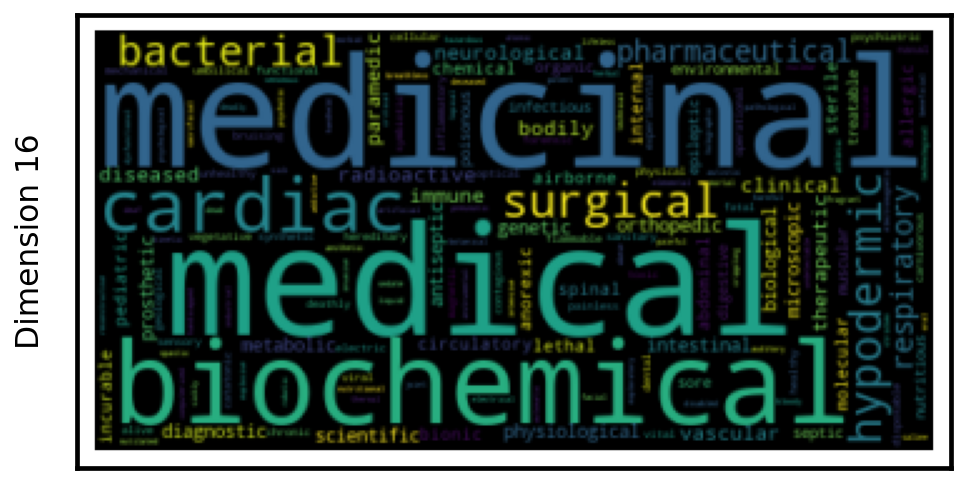}
\end{subfigure}
\begin{subfigure}{.5\textwidth}
    \centering
    \includegraphics[width=1.0\textwidth]{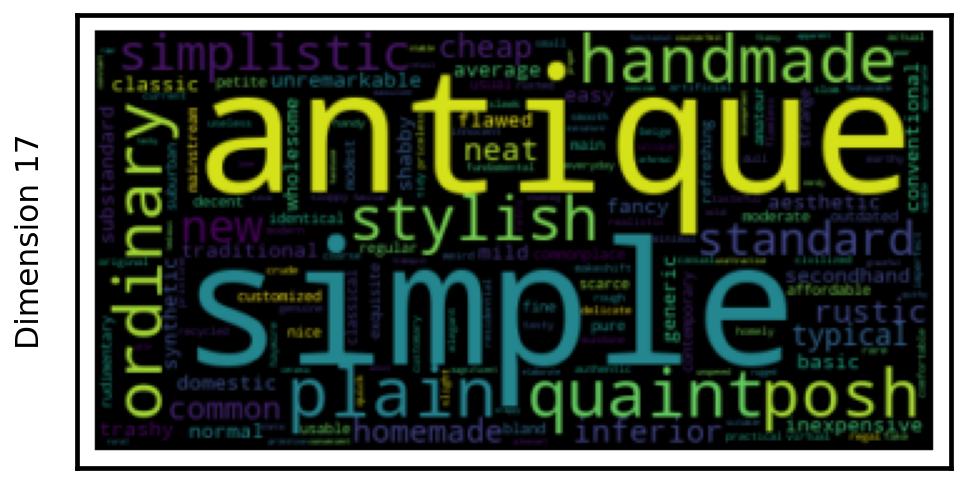}
\end{subfigure}%
\begin{subfigure}{.5\textwidth}
    \centering
    \includegraphics[width=1.0\textwidth]{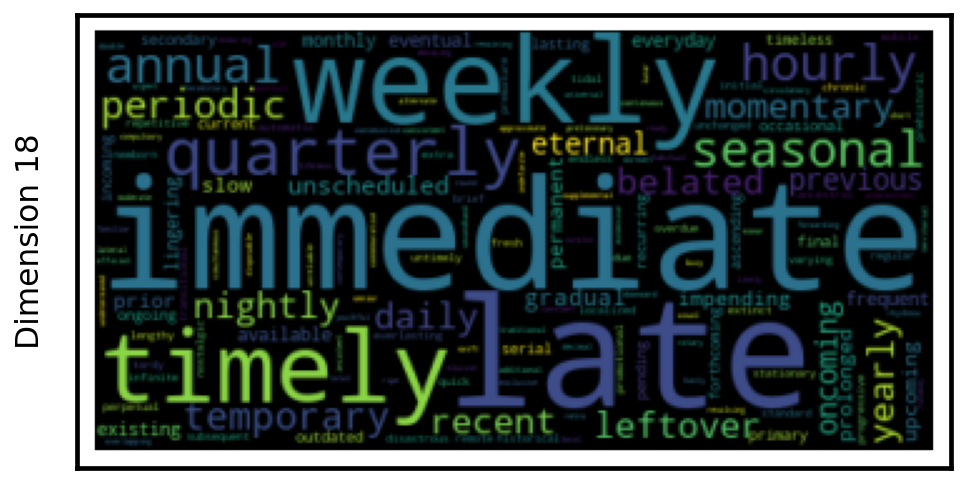}
\end{subfigure}
\begin{subfigure}{.5\textwidth}
    \centering
    \includegraphics[width=1.0\textwidth]{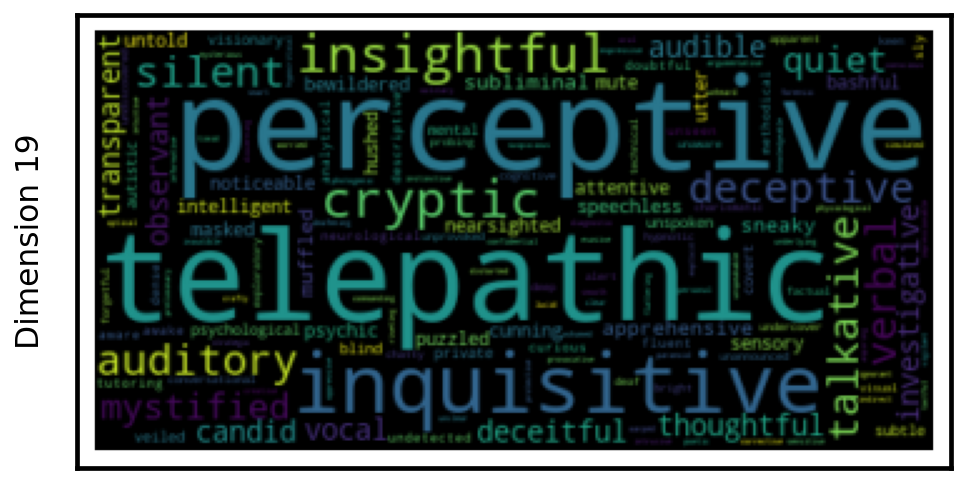}
\end{subfigure}%
\begin{subfigure}{.5\textwidth}
    \centering
    \includegraphics[width=1.0\textwidth]{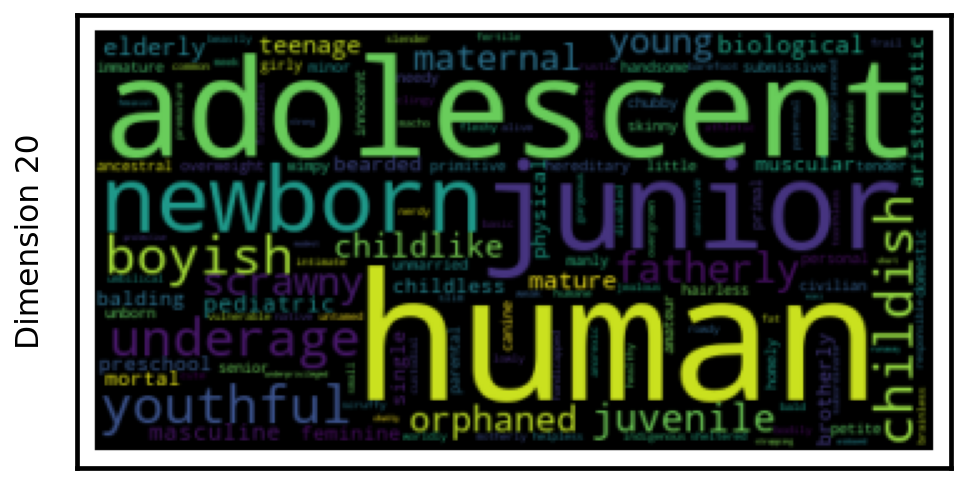}
\end{subfigure}
\caption{\textsc{Adjectives} Dimensions 11-20.}
\end{figure*}
}

\end{document}